\documentclass[numbers,webpdf,imaiai]{ima-authoring-template}%

\usepackage[utf8]{inputenc} % allow utf-8 input
\usepackage[T1]{fontenc}    % use 8-bit T1 fonts
\usepackage{hyperref}       % hyperlinks
\usepackage{url}            % simple URL typesetting
\usepackage{booktabs}       % professional-quality tables
\usepackage{amsfonts}       % blackboard math symbols
\usepackage{nicefrac}       % compact symbols for 1/2, etc.
\usepackage{microtype}      % microtypography
\usepackage{xcolor}         % colors
\usepackage{subcaption}
\usepackage{ulem}
\usepackage{amsmath}
\usepackage{amssymb}
\usepackage{mathtools}
\usepackage{amsthm}
\usepackage{comment}
\usepackage{graphicx}

\usepackage{wrapfig}

\newcommand{\AdaBatchGrad}{{\texttt{AdaBatchGrad}}}
\newcommand{\AdaGrad}{AdaGrad}

\usepackage[textsize=tiny]{todonotes}

\newcommand{\subf}[2]{%
  {\small\begin{tabular}[t]{@{}c@{}}
  #1\\#2
  \end{tabular}}%
}

\newcommand{\lp}{\left[}
\newcommand{\rp}{\right]}
\newcommand{\lb}{\left\lbrace}
\newcommand{\rb}{\right\rbrace}
\newcommand{\la}{\left\langle}
\newcommand{\ra}{\right\rangle}

\newcommand{\R}{\mathbb{R}}
\newcommand{\e}{\varepsilon}

\newcommand{\E}{\mathbb{E}}

\newcommand{\w}{{\bf w}}
\newcommand{\x}{{\bf w}}
\newcommand{\n}{{d}}

\theoremstyle{thmstyletwo}%
\newtheorem{theorem}{Theorem}%  meant for continuous numbers

\newtheorem{proposition}[theorem]{Proposition}%

\newtheorem{corollary}{Corollary}%
\newtheorem{lemma}{Lemma}%

\newtheorem{definition}{Definition}

\numberwithin{equation}{section}

\begin{document}

\firstpage{1}

\title[\AdaBatchGrad]{\AdaBatchGrad: Combining Adaptive Batch Size and Adaptive Step Size}

\author{Petr Ostroukhov \ORCID{0009-0005-5544-7461}*
\address{\orgname{Mohamed bin Zayed University of Artificial Intelligence}, \orgaddress{\state{Abu-Dhabi}, \country{United Arab Emirates}}}
\address{\orgname{Moscow Institute of Physics and Technology}, \orgaddress{\state{Dolgoprudny}, \country{Russia}}}
\address{\orgname{Institute for Information Transmission Problems RAS}, \orgaddress{\state{Moscow}, \country{Russia}}}}
\author{Aigerim Zhumabayeva \ORCID{0000-0002-5685-0739}
\address{\orgname{Mohamed bin Zayed University of Artificial Intelligence}, \orgaddress{\state{Abu-Dhabi}, \country{United Arab Emirates}}}}
\author{Chulu Xiang \ORCID{0009-0002-2962-9585}
\address{\orgname{Mohamed bin Zayed University of Artificial Intelligence}, \orgaddress{\state{Abu-Dhabi}, \country{United Arab Emirates}}}}
\author{Alexander Gasnikov \ORCID{0000-0002-7386-039X}
\address{\orgname{Innopolis University}, \orgaddress{\state{Innopolis}, \country{Russia}}}
\address{\orgname{Moscow Institute of Physics and Technology}, \orgaddress{\state{Dolgoprudny}, \country{Russia}}}
\address{\orgname{Skoltech}, \orgaddress{\state{Moscow}, \country{Russia}}}}
\author{Martin Tak\'a\v{c} \ORCID{0000-0001-7455-2025}
\address{\orgname{Mohamed bin Zayed University of Artificial Intelligence}, \orgaddress{\state{Abu-Dhabi}, \country{United Arab Emirates}}}}
\author{Dmitry Kamzolov \ORCID{0000-0001-8488-9692}
\address{\orgname{Mohamed bin Zayed University of Artificial Intelligence}, \orgaddress{\state{Abu-Dhabi}, \country{United Arab Emirates}}}}
\authormark{Petr Ostroukhov et al.}

\corresp[*]{Corresponding author: \href{email:email-id.com}{postroukhov12@gmail.com}}

% \received{Date}{0}{Year}
% \revised{Date}{0}{Year}
% \accepted{Date}{0}{Year}

%\editor{Associate Editor: Name}

\abstract{
This paper presents a novel adaptation of the Stochastic Gradient Descent (SGD), termed \AdaBatchGrad. 
This modification seamlessly integrates an adaptive step size with an adjustable batch size. 
An increase in batch size and a decrease in step size are well-known techniques to ``tighten'' the area of convergence of SGD and decrease its variance.
A range of studies by R. Byrd and J. Nocedal introduced various testing techniques to assess the quality of mini-batch gradient approximations and choose the appropriate batch sizes at every step.
Methods that utilized exact tests were observed to converge within $O(LR^2/\varepsilon)$ iterations. 
Conversely, inexact test implementations sometimes resulted in non-convergence and erratic performance. 
To address these challenges, \AdaBatchGrad\ incorporates both adaptive batch and step sizes, enhancing the method's robustness and stability. For exact tests, our approach converges in $O(LR^2/\varepsilon)$ iterations, analogous to standard gradient descent. For inexact tests, it achieves convergence in 
$O(\max\lbrace LR^2/\varepsilon, \sigma^2 R^2/\varepsilon^2 \rbrace )$ iterations. This makes \AdaBatchGrad\ markedly more robust and computationally efficient relative to prevailing methods.
To substantiate the efficacy of our method, we experimentally show, how the introduction of adaptive step size and adaptive batch size gradually improves the performance of regular SGD.
The results imply that \AdaBatchGrad\ surpasses alternative methods, especially when applied to inexact tests.
}


% \boxedtext{
% \begin{itemize}
% \item Key boxed text here.
% \item Key boxed text here.
% \item Key boxed text here.
% \end{itemize}}

\maketitle

\section{Introduction}
    In this work, we focus on the next stochastic optimization problem 
    \begin{equation}\label{eq:convex_problem}
        \min_{\w \in \R^d} \left\{ f(\w) \equiv \E_{\xi} [f(\w, \xi)] \right\},
    \end{equation}
    where $f: \R^d \to \R$ is a (non-)convex, $L$-smooth function:
    \begin{equation}\label{eq:smoothness}
        \|\nabla f(\w) - \nabla f(\w')\| \le L \|\w - \w'\|, \quad \forall \x, \w' \in \R^\n.
    \end{equation}
     Problem \eqref{eq:convex_problem} is one of the most popular problems in machine learning and statistics, also sometimes called \textit{risk minimization}. There are two popular approaches in optimization how to solve this problem: stochastic approximation (SA) and sample average approximation (SAA). 
     
     In SAA regime, we approximate the expectation $f(w)$ by the average of functions generated by a fixed number of samples $\xi_i$ for $i \in [1, \ldots, N]$: 
     \begin{equation}\label{eq:erm}
        \min_{\w \in \R^d} \left\{ \hat{f}(\w) \equiv \frac{1}{N}\sum\limits_{i=1}^{N} f(\w, \xi_i)] \right\}.
    \end{equation}
    Problem \eqref{eq:erm} is also referred to as \textit{empirical risk minimization}. Problem \eqref{eq:erm} is a deterministic sum-type optimization problem and could be solved by various optimization methods including variance-reduction techniques. However, even the exact solution of problem \eqref{eq:erm} would give only an approximate solution to the original risk minimization problem \eqref{eq:convex_problem}, where the quality of approximation depends on the number of samples $K$.
    
    In our paper, we focus on the stochastic approximation (SA) approach \cite{robbins1951stochastic,nemirovski1983problem, polyak1990new, polyak1992acceleration, nemirovski2009robust,lan2012optimal}. It is based on the access to stochastic oracle $\w \mapsto \{f(\w, \xi), \nabla f(\x, \xi) \}$ that returns a stochastic approximation of the function value $f(\x, \xi)$ and its gradient $\nabla f(\x, \xi)$ for the given point $\w$ . The $\lb \xi_i\rb_{i\geq 1}$ is a sequence of i.i.d. random variables, which are independent of the point $\w$. In machine learning, the randomness $\xi_i$ is generated from the data points $(x_i,y_i)$ leading to $f(\w; x_i, y_i)$. 
    
    We assume, that the gradient approximation in fixed point $\w$ is unbiased:
    \begin{equation}\label{eq:unbiasedness}
        \E [\nabla f(\w, \xi) ] = \nabla f(\w), \quad \forall \w \in \R^\n.
    \end{equation}
    Also, we use a rather standard assumption on bounded variance $\sigma$:
    \begin{equation}\label{eq:boundness of variance}
        \E \left[ \|\nabla f(\w, \xi) - \nabla f(\w) \|^2 \right] \le \sigma^2, \qquad \forall \w \in \R^\n.
    \end{equation}
    To solve the problem \eqref{eq:convex_problem} in SA regime, one can use Stochastic Gradient Descent (SGD)  
        \begin{equation}\label{eq:sgd_basic}
            \w_{t + 1} = \w_t - \eta_t \nabla f(\w_t, \xi_t),
        \end{equation} 
    where $\eta_t$ is a learning rate. This method with optimal step size has the next convergence rate for convex, $L$-smooth function \eqref{eq:smoothness}:
    \begin{equation}
        \E \lp f(\w_{T}) - f^{\ast}\rp \leq \mathcal{O}(1) \lp \frac{L}{T} + \frac{\sigma}{\sqrt{T}} \rp, 
    \end{equation}
where $\mathcal{O}(1)$ includes absolute and distance bounds. So, to reach $\e$-accuracy, such that $\E \lp f(\w_{T}) - f^{\ast}\rp \leq \e$, we need to calculate $K = \mathcal{O}(1)\max \lb \tfrac{L}{\e};\tfrac{\sigma^2}{\e^2}\rb$ stochastic gradients and the same amount of SGD steps $T = \mathcal{O}(1)\max \lb \tfrac{L}{\e};\tfrac{\sigma^2}{\e^2}\rb$.

To decrease the number of steps, we operate with mini-batch approximations of gradient, when we query the stochastic gradient multiple times and then use its mean as new gradient approximation:
\begin{equation}\label{eq:mini-batch approximation}
        \nabla f_{S_t}(\x_t) \equiv \frac{1}{|S_t|} \sum\limits_{i = 1}^{|S_t|} \nabla f_i(\x_t),
\end{equation}
where $S_t = \lb \xi_1, \ldots, \xi_{|S_t|} \rb$ denote a set of generated samples $\xi_i$ on step $t$. For simplicity, we denote $f_i(\w_t) \equiv f(\w_t,\xi_{i_t}) \equiv f(\w_t; x_{i_t}, y_{i_t})$.
Batched Stochastic Gradient Descent (SGD) algorithm
        \begin{equation}\label{eq:sgd equation}
            \x_{t + 1} = \x_t - \eta_t \nabla f_{S_t}(\x_t)
        \end{equation} 
with the fixed batch size $|S_t|=m$ and optimal step size converges with the next rate
    \begin{equation}
        \E \lp f(\w_{T}) - f^{\ast}\rp \leq \mathcal{O}(1) \lp \frac{L}{T} + \frac{\sigma}{\sqrt{m T}} \rp.
    \end{equation}
So, the total number of computed stochastic gradients equals to $K=mT=\max \lb \tfrac{m L}{\e}; \tfrac{\sigma^2}{\e^2} \rb$. However, the total number number of steps could be less and equals to $T=\max \lb \tfrac{L}{\e}; \tfrac{\sigma^2}{m \e^2} \rb$. In the case of parallel computations, it is better to compute multiple gradients for a small number of iterations than to compute $1$ gradient for a big number of iterations. Therefore, we are interested in finding a method that minimizes the number of steps $T$ required to reach $\varepsilon$-accuracy. The optimal batch size in that case is $$|S| = m = \frac{\sigma^2}{L \varepsilon}.$$
    
    From a practical perspective, SGD has two parameters, as shown in \eqref{eq:mini-batch approximation} and \eqref{eq:sgd equation}: the step size $\eta_t$ 
    and the batch size $|S_t|$. We do not know the exact $L$ and $\sigma$ values to find the optimal step size and optimal batch size. Hence, we have to tune these two parameters to decrease the variance of the stochastic gradient and guide us toward a better solution.
    
    To better motivate the need of such tuning from theoretical perspective, consider the convergence of batched SGD for some  not-optimal fixed step size $\eta$ and batch size $|S_t|=m$ for smooth  convex objective:
$$
    \E[f(\x_T)-f^{\ast}]
    \leq
    O(1) \lp \frac{1}{T\eta} + \frac{\sigma^2 \eta}{m}\rp,
    $$
    
    We can see that the second term in the right-hand side does not depend on $T$, and it represents some area of convergence, that we can not pass through with fixed parameters $\eta$ and $|S|$.
    If we decrease $\eta$, on the one hand, we decrease this convergence area, on the other hand, we also decrease the speed of convergence to this area, since $\tfrac{1}{T\eta}$ becomes bigger.
    On the other hand, if we increase $m$, the convergence area becomes smaller with the same convergence rate, but one iteration of the algorithm becomes more computationnaly expensive since we need to calculate more gradients.
    Thus, it makes sense to change these parameters as algorithm gets closer to the solution. A simple approach to decrease step size is to use a scheduler, that decreases the step size over time. However, it even more increase the number of parameters to tune. Thus, it hard in practice to find the optimal combination of $\eta_t$ and $|S_t|$. This lead us to the idea of adaptive step size and adaptive batch size. 

    \begin{center}
    \textit{Could one create a stochastic optimization method that is adaptive in both step size and batch size?}
    \end{center}
    
    One of the most popular adaptive techniques is \AdaGrad  \ \cite{streeter2010less,duchi2011adaptive}, where the step size choice is based on the history of oracle responses.
    In this paper, we focus on a modification of \AdaGrad \  from the paper \cite{li2019convergence} with the next step size policy
    \begin{equation}\label{eq:global step size}
        \eta_t = \frac{\alpha}{\left( \beta + \sum_{i = 1}^{t - 1} \|\nabla f_{S_i}(\x_i)\|^2 \right)^{\frac{1}{2} + \tau}},
    \end{equation}
    where $\alpha, \beta > 0$, $ 0 \le \tau < \frac{1}{2}$. This formula has two differences from the regular \AdaGrad\ step from \cite{duchi2011adaptive,rakhlin2013optimization, orabona2018scale}. 
    Two important elements of this modification are worth noting. First, an additional parameter $\tau$ gives us more flexibility in selecting the step size and provides us with almost sure convergence of the SGD.  
    Second, the current step's stochastic gradient isn't considered in the definition of the step size, because it introduces bias to the update direction. 
    For a more detailed explanation, see \cite[Section 4]{li2019convergence}.

    For the adaptive batch size strategy, we aim to have such batch size $|S_t|$ that the batched gradient $\nabla f_{S_t}(\w_t)$ would be close enough to the true gradient $\nabla f(\w_t)$ to make the efficient SGD step.
    In \cite{byrd2012sample,hashemi2014adaptive,cartis2018global,bollapragada2018adaptive} authors use specific tests, that check if the approximation obtained with the chosen batch size is close to the true gradient.
    ``Closeness to true gradient'' of the approximated gradient may be expressed either through its direction or through its length.
    In \cite{byrd2012sample} the authors propose so-called \textit{norm test} (Definition \ref{def:norm test}), that controls both of these parameters.
    In the paper \cite{bollapragada2018adaptive}, the authors propose \textit{inner product test} (Definition \ref{def:inner product test}) and \textit{orthogonality test} (Definition \ref{def:orthogonality test}), that only control the direction of the approximated gradient.
    Thus, these tests give more freedom to the procedure of batch size adjustment.
    For more details see \cite{bollapragada2018adaptive}.

    \begin{definition}\label{def:norm test}
        Let $\w \in \R^d$ be some fixed point, $S=\{\xi_1, ..., \xi_{|S|}\}$ be some batch of samples $\xi_i$, $\nabla f_{S}(\x)$ be defined in \eqref{eq:mini-batch approximation}.
        Then for some $\omega > 0$, the following inequality is called \textit{norm test}
        \begin{equation}\label{eq:norm test}
            \E \left[\|\nabla f_{S}(\x) - \nabla f(\x) \|^2 \right] \le \omega^2 \|\nabla f(\x)\|^2.
        \end{equation}
    \end{definition}
    
    \begin{definition}\label{def:inner product test}
        Let $\w \in \R^d$ be some fixed point, $S=\{\xi_1, ..., \xi_{|S|}\}$ be some batch of samples $\xi_i$, $\theta > 0$, then the following inequality is called \textit{inner product test}
        \begin{equation*}\label{eq:inner product test true}
            \E \left[ \left(\nabla f_{S}(\x)^{\top} \nabla f(\x) - \|\nabla f(\w) \|^2 \right)^2 \right] \le \theta^2 \|\nabla f(\x)\|^4, 
        \end{equation*}
        which is equivalent (see \cite{bollapragada2018adaptive}) to
        \begin{equation}\label{eq:inner product test}
            \frac{\E \left[ \nabla f_i(\w)^\top \nabla f(\w) - \| \nabla f(\w) \|^2 \right]}{|S|} \le \theta^2 \|\nabla f(\w)\|^4.
        \end{equation}
    \end{definition}

    \begin{definition}\label{def:orthogonality test}
        Let $\w \in \R^d$ be some fixed point, $S=\{\xi_1, ..., \xi_{|S|}\}$ be some batch of samples $\xi_i$, $\nu > 0$ then the following inequality is called \textit{orthogonality test}
        \begin{equation*}
            \E \left[ \nabla f_S(\x) - \frac{\nabla f_S(\x)^{\top} \nabla f(\x)}{\|\nabla f(\x)\|^2} \nabla f(\x) \right] 
            \le \nu^2 \|\nabla f(\x)\|^2,
        \end{equation*}
        which is equivalent (see \cite{bollapragada2018adaptive}) to
        \begin{equation}\label{eq:orthogonality test}
           \frac{\E \left[ \nabla f_i(\x) - \frac{\nabla f_i(\x)^{\top} \nabla f(\x)}{\|\nabla f(\x)\|^2} \nabla f(\x) \right]}{|S|} \le \nu^2 \|\nabla f(\w)\|^2.
        \end{equation}
    \end{definition}

    \subsection{Motivation}\label{sec:motivation}
    Let us suppose that on some step $t$ of procedure \eqref{eq:sgd equation}, we get the batch size $|S_t|$ from \eqref{eq:inner product test} and \eqref{eq:orthogonality test} for point $\w_t$, then
    \begin{equation}\label{eq:new batch size equation}
        |S_t| = \max \left(
            \frac{\E [\nabla f_{i}(\x_t)^{\top} \nabla f(\x_t)]}{\theta^2 \|\nabla f(\x_t)\|^4},
            \frac{\E\left[ \nabla f_i(\x_t) - \frac{\nabla f_i(\x_t)^{\top} \nabla f(\x_t)}{\|\nabla f(\x_t)\|^2} \nabla f(\x_t) \right]}{\nu^2 \|\nabla f(\x_t)\|^2}
        \right).
    \end{equation}
    The problem here is that to perform these tests you need to know the true gradient, which sometimes is impossible to compute, especially in setup \eqref{eq:convex_problem}.
    To tackle this issue, it was suggested to use approximated versions of these tests \eqref{eq:inner product test}, \eqref{eq:orthogonality test}, where the mini-batch approximation is used instead of the true gradient:
    \begin{align}
        \frac{1}{|S_t| - 1} \frac{\sum_{i \in S_t} 
            \left(
                \nabla f_i(\x_t)^{\top} \nabla f_{S_t}(\x_t) - \|\nabla f_{S_t}(\x_t)\|^2 
            \right)^2}
            {|S_t|}
        &\le \theta^2 \|\nabla f_{S_t}(\x_t)\|^4, \label{eq:approximated inner product test} \\
        \frac{1}{|S_t| - 1} 
            \frac{\sum_{i \in S_t}
                \left\|
                    \nabla f_i(\x_t) - \frac{\nabla f_i(\x_t)^{\top} \nabla f_{S_t}(\x_t)}{\|\nabla f_{S_t}(\x_t)\|^2} \nabla f_{S_t}(\x_t)
                \right\|^2}
            {|S_t|}
        &\le \nu^2 \|\nabla f_{S_t}(\x_t)\|^2. \label{eq:approximated orthogonality test}
    \end{align}      
    As a result, the approximation of \eqref{eq:new batch size equation} looks as follows:
    \begin{equation}\label{eq:approximated new batch size equation}
        |\hat S_t| = \max \left( 
            \frac{ \sum_{i \in S_t} 
                \left(
                    \nabla f_i(\x_t)^{\top} \nabla f_{S_t}(\x_t) - \|\nabla f_{S_t}(\x_t)\|^2 
                \right)^2}
            {(|S_t| - 1)\theta^2 \|\nabla f_{S_t}(\x_t)\|^4},
            \frac{\sum_{i \in S_t}
                \left\|
                    \nabla f_i(\x_t) - \frac{\nabla f_i(\x_t)^{\top} \nabla f_{S_t}(\x_t)}{\|\nabla f_{S_t}(\x_t)\|^2} \nabla f_{S_t}(\x_t)
                \right\|^2}
            {(|S_t| - 1)\nu^2 \|\nabla f_{S_t}(\x_t)\|^2}
        \right).
    \end{equation}
    However, the theoretical proofs only exist for the original exact versions of these tests \eqref{eq:inner product test}, \eqref{eq:orthogonality test} with true gradient.
    This introduces some \textit{test inconsistency}, which means that the method can get into situation, where the approximated test is passed (that we can detect), but fails the exact one (that we cannot detect).
    Then the batch size does not increase, because the method ``thinks'' that we have a good approximation of the true gradient. 
    Thus, the algorithm can get a bad approximation of the gradient, and, as a result, the algorithm performs a step in the direction, that is different from the direction of anti-gradient.

    The situation becomes even worse if the method has a big step size. 
    In \cite{bollapragada2018adaptive}, line-search was used to adjust the step size of their algorithm.
    Unfortunately, in stochastic case, this procedure may cause the issue of the big step size in the wrong direction.
    The main problem here is that line-search tries on each step to increase the step size towards the direction of the minimum of the function it gets as input. 
    But in the stochastic case, we operate with stochastic approximations of objective functions and their gradients. 
    Thus, in case line-search gets as input a bad stochastic approximation of function and its gradient, the method could perform a huge step in the wrong direction and diverge.

    \subsubsection{Test inconsistency example}\label{sec:example}

        \begin{figure}
            \centering
            \includegraphics[width=\textwidth]{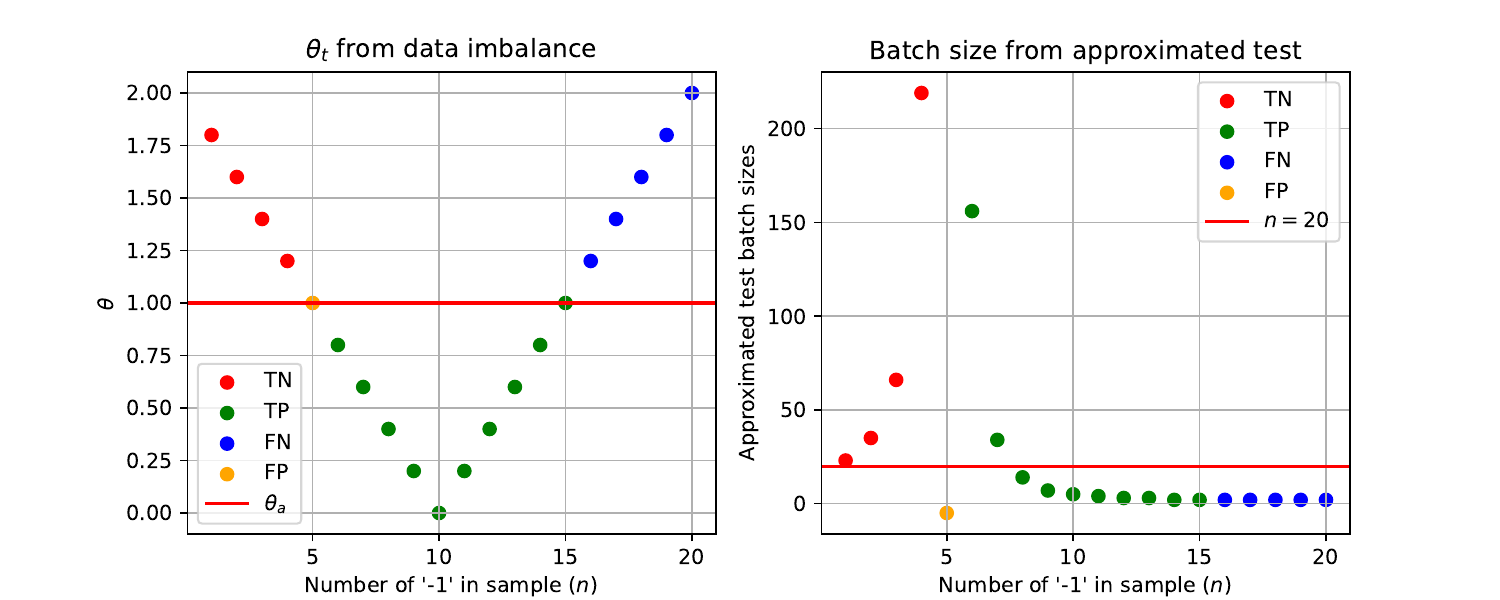}
            \caption{Inner product test inconsistency}
            \label{fig:inner product test test inconsistency}
        \end{figure}
        
        Let us consider the following example. We define  
        \begin{gather}
            f_i(\w) = \frac{1}{2} (\w - \xi_i)^2,\ \w \in \R,\ \xi \in \{-1, 1\}. \label{eq:test inconsistency problem}
        \end{gather}
        The number of samples $\xi_i = -1$ in our sample $S$ is denoted by $n$. Then, we can explicitly calculate
        \begin{gather*}
            \nabla f(\w) = \w,\\
            \nabla f_i(\w) = \w - \xi_i, \\
            \nabla f_S(\w) = \w - 1 + \frac{2n}{|S|}.
        \end{gather*}
    
        Let $\w = 0.5$, $|S| = 20$. 
        Denote $\theta_a$ -- parameter, that we use in approximated inner-product test and choose $\theta_a = 1$.
        In Figure \ref{fig:inner product test test inconsistency}, we show the mentioned test inconsistency for a sample of size $|S|$.
        The $x$ axis corresponds to $n$ on both plots.
    
        We start with the left plot.
        This plot is related only to the exact inner product test \eqref{eq:inner product test}.
        Here we show, which minimal $\theta$ we need to choose for the exact test to work for particular $n$.
        It means that for every $\theta$ higher than a particular point exact test will pass for a particular $n$.
        With the red line on this plot, we depict our $\theta_a$.
        As you can see, we pass the test for $n \in [5, 15]$.
        We use different colors for points to show the cases, when an exact test with parameter $\theta_a$ gives true-positive (TP), true-negative (TN), false-positive (FP), or false-negative (FN) results for particular $n$.
        The meaning of these results is as follows: for chosen $n$ and corresponding $S$
        \begin{itemize}
            \item \textbf{TN}: $\nabla f_S(\w)^{\top} \nabla f(\w) \le 0$ and $\theta_a$ with chosen batch fails exact test \eqref{eq:inner product test};
            \item \textbf{TP}: $\nabla f_S(\w)^{\top} \nabla f(\w) > 0$ and $\theta_a$ with chosen batch passes exact test \eqref{eq:inner product test};
            \item \textbf{FN}: $\nabla f_S(\w)^{\top} \nabla f(\w) > 0$ and $\theta_a$ with chosen batch fails exact test \eqref{eq:inner product test};
            \item \textbf{FP}: $\nabla f_S(\w)^{\top} \nabla f(\w) \le 0$ and $\theta_a$ with chosen batch passes exact test \eqref{eq:inner product test}.
        \end{itemize}
    
        Next, let us move to the second picture of Figure \ref{fig:inner product test test inconsistency}.
        This plot is related to the approximated inner-product test \eqref{eq:approximated inner product test}.
        Every point here shows the batch size, which is given to us by the approximated test with $\theta_a$ for corresponding $n$.
        The color of point for every $n$ is taken from the left picture.
        For the orange point approximated test gives batch size equal to infinity, so we plot it with a value $-5$.
    
        Eventually, this plot shows us, that behaviors of exact and approximated inner-product tests are very different.
        For example, if we look at the green points on the right picture of Figure \ref{fig:inner product test test inconsistency}, we see that the approximated test tells us that we need to increase batch size, although it is TP and this batch passes the exact test. To summarize, Figure \ref{fig:inner product test test inconsistency} shows that the approximated tests could return very different results from the exact in both false-positive and false-negative ways. Hence, the method should be robust to this FP and FN test failures.

    \subsection{Contribution}
        Our main contributions can be summarized in the following:
        \begin{itemize}
            \item We propose \AdaBatchGrad\ method, which is adaptive both in step size and in batch size.
            \item We claim, that such modification of SGD allows to escape \textit{test inconsistency} situation, that we mentioned earlier.
            \item We show, that if our method on each step passes the exact norm test \eqref{eq:norm test} (therefore, passes exact inner-product test \eqref{eq:inner product test} and orthogonality test \eqref{eq:orthogonality test}), we can get convergence rate $O(1/\e)$ regardless of noise intensity.
            \item We experimentally show, that by introducing adaptivity in step size and adaptivity in batch size we can gradually improve the overall performance of SGD.
            Moreover, we show, that adaptivity in batch size and adaptivity in step size balance each other.
            This means, that by slightly increasing batch size and slightly decreasing step size, we can achieve similar results as if we significantly increase batch size or significantly decrease step size.
        \end{itemize}

    \subsection{Structure}
        The structure of the paper is as follows.
        In Section \ref{sec:related work} we provide a literature review.
        In particular, in Section \ref{sec:adaptive step size} we mention papers, connected with adaptive step sizes in SGD, and in Section \ref{sec:adaptive batch size} -- with adaptive batch sizes.
        In Section \ref{sec:results} we present our main results: Section \ref{sec:background} provides some background information, that we need before introducing \AdaBatchGrad; Section \ref{sec:adabatchgrad} introduces \AdaBatchGrad along with its convergence rates; Section \ref{sec:adabatchgrad with norm test} presents our theoretical convergence analysis of \AdaBatchGrad\ in case when it satisfies norm test \eqref{eq:norm test} on each iteration.
        In Section \ref{sec:numericals} we provide our experimental results. In particular, we describe, how SGD benefits from changing step size, batch size, and their adaptivity.

\section{Related work}\label{sec:related work}
    Stochastic gradient descent \cite{robbins1951stochastic,nemirovski1983problem,polyak1990new,polyak1992acceleration, shalev2007pegasos,nemirovski2009robust,lan2012optimal,hardt2016train, gorbunov2020unified, gorbunov2020stochastic,ilandarideva2023accelerated} has become a standard method for lots of modern machine learning problems.
    In the paper \cite{gower2019sgd}, the authors provide a significant analysis of SGD, with different sampling strategies. 
    The authors provide explicit formulas for optimal batch sizes for different sampling strategies and show when the step size choice strategy should switch from constant to decrease.
    Additionally, the authors show the dependence between optimal step size and optimal batch size, and how it influences the overall convergence.
    \subsection{Adaptive step size}\label{sec:adaptive step size}
        In the smooth deterministic case, a standard way to choose step size is $\gamma=1/L$ (see \cite{nesterov2018lectures}).
        But in practice, we usually do not know smoothness constant $L$.
        Thus, one way to adjust step size is by Armijo-type line search, where on each iteration the algorithm checks, that its upper quadratic approximation is correct.
        The extension of this technique to stochastic cases was proposed in \cite{dvinskikh2019adaptive}.

        Another approach is based on the idea of taking $\gamma=1/\|\nabla f(x)\|$, which came from non-smooth optimization \cite{polyak1987introduction}.
        As a result, in \cite{streeter2010less,duchi2011adaptive} authors independently propose global and coordinate-wise step size strategies for stochastic optimization, where step size depends on the inverted sum of previous gradients.
        
        However, in these papers, the authors assume that the objective domain should be convex and bounded, which can be quite limiting, as claimed in \cite{li2019convergence}.
        Later, \cite{li2019convergence} proposed a modified version of \AdaGrad \  both for global and for coordinate-wise step size adjustment.
        In this article, the authors provide almost-surely convergence of SGD with such step size strategy both for convex and non-convex setups.
        However, the authors assume knowledge of smoothness constant to choose step size parameters correctly.
        Contrastingly, the work \cite{ward2020adagrad} does not have this limitation, although the authors assume that the objective's full gradient is uniformly bounded.
        In \cite{chen2019convergence}, the paper studied the sufficient conditions for the convergence of Adam-type \cite{kingma2014adam} algorithms in non-convex settings.
        There is a series of works on parameter-free algorithms \cite{carmon2022making, defazio2023learning}, 
        where authors propose adaptive step size adjustment strategies without any hyperparameters. 
        However, since these authors do not make any assumptions on noise variance, the complexity of their algorithms is worse than  $O(1/ \e^2)$.
        Additionally, there are some works on stochastic Polyak step size \cite{loizou2021stochastic,gower2021stochastic,li2022sp2,jiang2023adaptive, abdukhakimov2023stochastic,abdukhakimov2023sania}.

    \subsection{Adaptive batch size}\label{sec:adaptive batch size} 
        In \cite{smith2018don}, the authors provide empirical evidence for the fact that by increasing batch size we can get the same performance gains as from decreasing step size. 
        However, it's crucial not to overdo this.
        The most obvious reason is that an increase in batch size leads to an increase in the computational complexity of a single operation, although this problem can be tackled via parallelization.
        Additionally, in \cite{lecun2002efficient,keskar2016large,dinh2017sharp,shallue2019measuring}, the authors study the problem of a possible increase of generalization error of deep learning model when using large batch sizes.
        The authors of \cite{keskar2016large} conclude, that models, trained with larger batch sizes, converge to sharp minima, which causes poor generalization.
        However, \cite{dinh2017sharp,shallue2019measuring} argue with this claim.
        
        Additionally, in \cite{shallue2019measuring,zhang2019algorithmic}, the authors show that the positive impact on the optimization process of increase of batch size diminishes after some critical batch size.
        
        In \cite{gower2019sgd}, the authors provide a theoretical connection between step size and batch size.
        Later, in \cite{alfarra2020adaptive} authors use the equation for the optimal batch size from \cite{gower2019sgd} on every step of the SGD.
        However, since the optimal value assumes knowledge of $\|\nabla f_i(\w^{\ast})\|$, the authors use approximations of these values, which makes this batch size suboptimal.
        Another way of making the procedure of batch size adjustment adaptive is to use specific tests at each iteration to check if the current batch size is sufficient.
        Among the first works
        to explore this approach was \cite{byrd2012sample}. It introduced the \textit{norm test} \eqref{eq:norm test}, which ensures the direction and norm of a stochastic gradient approximation don't diverge too much from the full gradient.
        Later, in \cite{bollapragada2018adaptive}, the authors improved this idea and proposed two new tests. The first, the \textit{inner-product test} \eqref{eq:inner product test}, focuses on ensuring the direction of the gradient approximation is the descent direction with high probability. The second, the \textit{orthogonality test} \eqref{eq:orthogonality test}, checks that the variance of sampled gradients, orthogonal to the full gradient, doesn't increase too much.
         This idea was later expanded to constrained and composite optimization in \cite{xie2020constrained}.

         In the paper \cite{dvinskikh2019adaptive}, one can find stochastic SGD and accelerated stochastic SGD, that are adaptive both in step size and in batch size simultaneously.
         The idea behind the approach of \cite{dvinskikh2019adaptive} is the backtracking procedure of Yu.~Nesterov from \cite{nesterov2015universal} generalized on stochastic setup.
         This procedure allows authors to make both step size and batch size adaptive.
         The theoretical and practical performance of this approach is quite limited due to the line-search procedure, that gives multiplicative logarithmic factor to overall complexity.
         Moreover, authors assume the knowledge of the upper bound of $\sigma^2$ in \eqref{eq:boundness of variance}.

\section{Main Results}\label{sec:results}
    \subsection{Background information}\label{sec:background}
    
    \begin{algorithm}[t]
        \caption{Backtracking Line Search [Algorithm 2 in \cite{bollapragada2018adaptive}]}\label{alg:line search}
        \begin{algorithmic}[1]
            \Require $S_t, L_{t-1}, \alpha > 1$.
            \State $a_t = (Var_{i \in S_t} [\nabla f_i(\w_t)])/(|S_t| \|\nabla f_{S_t}(\w_t)\|^2) + 1$, where $Var_{i \in S_t}$ is a sampled variance over $S_t$.
            \State $\zeta_{t} = \max(1, 2/a_t)$
            \State $L_t = L_{t - 1} / \zeta_t$
            \State $f_{new} = f_{S_t}\left(\w_t - \frac{1}{L_t} \nabla f_{S_t}(\w_t)\right)$
            \While{$f_{new} > f_{S_t}(\w_t) - \frac{1}{2 L_t} \|\nabla f_{S_t}(\w_t)\|^2$}
                \State $L_t = \alpha L_{t - 1}$
                \State $f_{new} = f_{S_t}\left( \w_t - \frac{1}{L_t} \nabla f_{S_t}(\w_t) \right)$
            \EndWhile
            \State \Return $L_t$
        \end{algorithmic}
    \end{algorithm}

    \begin{algorithm}[t]
        \caption{Adaptive Sampling Method [Algorithm 3 in \cite{bollapragada2018adaptive}]}\label{alg:adaptive sampling}
        \begin{algorithmic}[1]
            \Require $\w_1$, $S_1$, $\theta > 0$, $\nu > 0$
            \State $t = 1$
            \While{convergence criterion is not satisfied}
                \State Compute $L_t$ with Algorithm \ref{alg:line search}
                \State $\eta_t = 1 / L_t$
                \State Compute $\nabla f_{S_t}(\w_t)$
                \State $\w_{t + 1} = \w_t - \eta_t \nabla f_{S_t}(\w_t)$
                \State $t = t + 1$
                \State $|S_t| = |S_{t-1}|$, choose new sample $S_t$
                \If{not \eqref{eq:approximated inner product test} or not \eqref{eq:approximated orthogonality test}}
                    \State Compute $|S_t|$ using \eqref{eq:approximated new batch size equation} and choose a new sample $S_t$.
                \EndIf
            \EndWhile
            \State \Return $\w_t$
        \end{algorithmic}
    \end{algorithm}

    Our proposed method is based on Adaptive Sampling Method from \cite{bollapragada2018adaptive}.
    We provide its basic version in Algorithm \ref{alg:adaptive sampling} without running average heuristic.
    Additionally, we do not use the running average heuristic, when batch size remains constant for certain number of iterations.
    Next, we provide theoretical convergence results from \cite{bollapragada2018adaptive}.
    \begin{theorem}[Theorem 3.3 in \cite{bollapragada2018adaptive}]\label{th:bollaprogada convex}
        For problem \eqref{eq:convex_problem} with convex, differentiable, smooth \eqref{eq:smoothness} function $f(\w)$,
        let $\{x_k\}$ be generated by \eqref{eq:sgd equation}, where $|S_k|$ is chosen such that exact inner-product test \eqref{eq:inner product test} and exact orthogonality test \eqref{eq:orthogonality test} are satisfied at each iteration for any given $\theta > 0$ and $\nu > 0$.
        Then, if the step size satisfies 
        \[
            \eta_t = \alpha < \frac{1}{(1 + \theta^2 + \nu^2)L},
        \]
        we have
        \begin{equation}\label{eq:bollaprogada convex}
            \min_{1 \le t \le T} \E [f(\x_t)] - f^{\ast} \le \frac{1}{2 \alpha c T}.
        \end{equation}
    \end{theorem}
    \begin{theorem}[Theorem 3.4 in \cite{bollapragada2018adaptive}]\label{th:bollaprogada non-convex}
        For problem \eqref{eq:convex_problem} with non-convex, differentiable, smooth \eqref{eq:smoothness} function $f(\w)$,
        let $\{x_k\}$ be generated by \eqref{eq:sgd equation}, where $|S_k|$ is chosen such that exact inner-product test \eqref{eq:inner product test} and exact orthogonality test \eqref{eq:orthogonality test} are satisfied at each iteration for any given $\theta > 0$ and $\nu > 0$.
        Then, if the step size satisfies 
        \[
            \eta_t = \alpha \le \frac{1}{(1 + \theta^2 + \nu^2)L},
        \]
        then
        \[
            \lim_{t \to \infty} \E \left[ \| \nabla f(\x_t) \right] = 0.
        \]
        Moreover,
        \begin{equation}\label{eq:bollaprogada non-convex}
            \min_{1 \le t \le T} \E \left[\| \nabla f(\x_t) \|^2 \right] \le \frac{2}{\alpha T}(f(\w_0) - f_{\min}),
        \end{equation}
        where $f_{\min}$ is a lower bound on $f$ in $\R^d$.
    \end{theorem}
    Obviously, the method, that is proposed in these theorems, is not implementable.
    Firstly, as we mentioned earlier, it is impossible to check exact tests in practice.
    Secondly, to properly determine step size, we need Lipschitz constant, that is usually unknown in practice.
    In Algorithm \ref{alg:adaptive sampling}, proposed by authors, they use approximated versions of mentioned tests from \eqref{eq:approximated inner product test} and \eqref{eq:approximated orthogonality test} and line search (Algorithm \ref{alg:line search}) to choose step size.
    It is not hard to understand, that these heuristics do not guarantee us, that Algorithm \ref{alg:adaptive sampling} will fit in the assumptions, made by these theorems.

    Next, we provide convergence results for SGD with AdaGrad stepsize \eqref{eq:global step size} both for convex and non-convex cases.
    \begin{theorem}[Theorem 3 in \cite{li2019convergence}]\label{th:li-orabona convex}
        For problem \eqref{eq:convex_problem} with convex, differentiable, smooth \eqref{eq:smoothness} function $f(\w)$, we have access to the stochastic unbiased gradients \eqref{eq:unbiasedness}, and noise variance is upper bounded \eqref{eq:boundness of variance}.
        Let the step size be as in \eqref{eq:global step size}, where $\alpha, \beta > 0, \tau \in [0, \frac{1}{2}), 2\alpha L < \beta^{\frac{1}{2} + \tau}$.
        Denote
        \[
            \gamma = 
            \begin{cases}
                O\left( \frac{1 + \alpha^2 \ln T}{\alpha \left(1 - \frac{2\alpha}{\sqrt \beta}) \right)} \right), & \tau = 0, \\
                O\left( \frac{1 + \alpha^2 \left(\frac{1}{\tau} + \sigma^2 \ln T \right)}{\alpha \left(1 - \frac{2 \alpha}{\beta^{1/2 + \tau}} \right)} \right), & \tau > 0.
            \end{cases}, \quad
            \hat \w_T = \frac{1}{T} \sum_{t=1}^T \w_t.
        \]
        Then, SGD \eqref{eq:sgd equation} with step size \eqref{eq:global step size} converges as
        \begin{equation}\label{eq:li-orabona convex}
            \E \left[ (f(\bar \w_T) - f(\w^{\ast}))^{1/2 - \tau} \right] \le 
            \frac{1}{T^{1/2-\tau}} \max 
            \left(
                2^{\frac{1}{1/2 - \tau}} L^{1/2 + \tau} \gamma;
                (\beta + T \sigma^2)^{1/4 - \tau^2} \gamma^{1/2 - \tau}
            \right).
        \end{equation}
    \end{theorem}

    \begin{theorem}[Theorem 4 in \cite{li2019convergence}]\label{th:li-orabona non-convex}
        For problem \eqref{eq:convex_problem} with non-convex, differentiable, smooth \eqref{eq:smoothness} function $f(\w)$, we have access to the stochastic unbiased gradients \eqref{eq:unbiasedness}, and noise variance is upper bounded \eqref{eq:boundness of variance}.
        Let the step size be as in \eqref{eq:global step size}, where $\alpha, \beta > 0, \tau \in [0, \frac{1}{2}), 2\alpha L < \beta^{\frac{1}{2} + \tau}$.
        Denote
        \[
            \gamma = 
            \begin{cases}
                O\left( \frac{1 + \alpha^2 \ln T}{\alpha \left(1 - \frac{2\alpha}{\sqrt \beta}) \right)} \right), & \tau = 0, \\
                O\left( \frac{1 + \alpha^2 \left(\frac{1}{\tau} + \sigma^2 \ln T \right)}{\alpha \left(1 - \frac{2 \alpha}{\beta^{1/2 + \tau}} \right)} \right), & \tau > 0.
            \end{cases}, \quad
            \hat \w_T = \frac{1}{T} \sum_{t=1}^T \w_t.
        \]
        Then, SGD \eqref{eq:sgd equation} with step size \eqref{eq:global step size} converges as
        \begin{equation}\label{eq:li-orabona non-convex}
            \E \left[ \min_{1 \le t \le T} \|\nabla f(x_t)\|^{1 - 2\tau} \right] \le 
            \frac{1}{T^{1/2-\tau}} \max 
            \left(
                2^{\frac{1/2 + \tau}{1/2 - \tau}} \gamma;
                2^{1/2 + \tau}(\beta + 2T \sigma^2)^{1/4 - \tau^2} \gamma^{1/2 - \tau}
            \right).
        \end{equation}
    \end{theorem}

    In the paper \cite{li2019convergence}, the authors operate with any stochastic gradients, which are unbiased \eqref{eq:unbiasedness} and have bounded variance \eqref{eq:boundness of variance}.
    Thus, the convergence rates from Theorems \ref{th:li-orabona convex} and \ref{th:li-orabona non-convex} also hold for mini-batch approximations of gradients.

   \subsection{AdaBatchGrad: algorithm, robust to test inconsistency}\label{sec:adabatchgrad}
    
    \begin{algorithm}[t]
        \caption{AdaBatchGrad}\label{alg:adabatchgrad}
        \begin{algorithmic}[1]
            \Require $\w_1$, $S_1$, $\alpha > 0$, $\beta > 0$, $\tau \in [0, 1/2]$, $\theta > 0$, $\nu > 0$
            \State $t = 1$
            \While{convergence criterion is not satisfied}
                \State 
                \[
                    \eta_t = \frac{\alpha}{\left( \beta + \sum_{i = 1}^{t - 1} \|\nabla f_{S_i}(\x_i)\|^2 \right)^{\frac{1}{2} + \tau}},
                \]
                \State Compute $\nabla f_{S_t}(\w_t)$
                \State $\w_{t + 1} = \w_t - \eta_t \nabla f_{S_t}(\w_t)$
                \State $t = t + 1$
                \State $|S_t| = |S_{t-1}|$, choose new sample $S_t$
                \If{not \eqref{eq:approximated inner product test} or not \eqref{eq:approximated orthogonality test}}
                    \State Compute $|S_t|$ using \eqref{eq:approximated new batch size equation} and choose a new sample $S_t$.
                \EndIf
            \EndWhile
            \State \Return $\w_t$
        \end{algorithmic}
    \end{algorithm}

    Our main result is provided in Algorithm \ref{alg:adabatchgrad}.
    The main difference between this method and Algorithm \ref{alg:adaptive sampling} is in the strategy of adjustment of step size.
    Our Algorithm \ref{alg:adabatchgrad} uses AdaGrad strategy \eqref{eq:global step size} from \cite{li2019convergence}, and Algorithm \ref{alg:adaptive sampling} uses line search from Algorithm \ref{alg:line search}.
    This, at first glance, insignificant change makes our Algorithm robust to \textit{test inconsistency} situation, that we described in Section \ref{sec:motivation}.
    In case when batch size $|S_t|$ passes the approximated test but fails the exact one (again, we can not check it), the method does not increase batch size, although it should.
    Thus, the method gets a bad approximation of the gradient at the current point.
    If the method uses line search to choose step size, it will perform a big step in the wrong direction, which is caused by a bad approximation of anti-gradient.
    This may cause a divergence of the method.
    However, if the method uses AdaGrad strategy \eqref{eq:global step size} to choose step size, it will still converge with rates from Theorem \ref{th:li-orabona convex} or \ref{th:li-orabona non-convex} if gradient approximations are unbiased \eqref{eq:unbiasedness} and have bounded variance \eqref{eq:boundness of variance}.
    Thus, by ``merging'' these two approaches we can get the following result for the convex case.
    \begin{proposition}
        For problem \eqref{eq:convex_problem} with convex, differentiable, smooth \eqref{eq:smoothness} function $f(\w)$, where we have access to unbiased stochastic gradients \eqref{eq:unbiasedness}, and their variance is upper-bounded \eqref{eq:boundness of variance},
        Algorithm \ref{alg:adabatchgrad} converges as \eqref{eq:li-orabona convex}:
        \begin{equation*}
            \E \left[ (f(\bar \w_T) - f(\w^{\ast}))^{1/2 - \tau} \right] \le 
            \frac{1}{T^{1/2-\tau}} \max 
            \left(
                2^{\frac{1}{1/2 - \tau}} L^{1/2 + \tau} \gamma;
                (\beta + T \sigma^2)^{1/4 - \tau^2} \gamma^{1/2 - \tau}
            \right),
        \end{equation*}
        where $\gamma$ is defined in Theorem \ref{th:li-orabona convex}.
        Moreover, if on each step of Algorithm \ref{alg:adabatchgrad} batch size $|S_k|$ is chosen such that exact inner-product test \eqref{eq:inner product test} and exact orthogonality test \eqref{eq:orthogonality test} are satisfied, and $\eta_t < 1 / ((1 + \theta^2 + \nu^2)L)$, then Algorithm \ref{alg:adabatchgrad} converges as \eqref{eq:bollaprogada convex}:
        \begin{equation*}
            \min_{1 \le t \le T} \E [f(\x_t)] - f^{\ast} \le \frac{1}{2 \alpha c T}.
        \end{equation*}
    \end{proposition}
    And we can make a similar claim for the non-convex case.
    \begin{proposition}
        For problem \eqref{eq:convex_problem} with non-convex, differentiable, smooth \eqref{eq:smoothness} function $f(\w)$, where we have access to unbiased stochastic gradients \eqref{eq:unbiasedness}, and their variance is upper-bounded \eqref{eq:boundness of variance},
        Algorithm \ref{alg:adabatchgrad} converges as \eqref{eq:li-orabona non-convex}:
        \begin{equation*}
            \E \left[ \min_{1 \le t \le T} \|\nabla f(x_t)\|^{1 - 2\tau} \right] \le 
            \frac{1}{T^{1/2-\tau}} \max 
            \left(
                2^{\frac{1/2 + \tau}{1/2 - \tau}} \gamma;
                2^{1/2 + \tau}(\beta + 2T \sigma^2)^{1/4 - \tau^2} \gamma^{1/2 - \tau}
            \right).
        \end{equation*}
        where $\gamma$ is defined in Theorem \ref{th:li-orabona non-convex}.
        Moreover, if on each step of Algorithm \ref{alg:adabatchgrad} batch size $|S_k|$ is chosen such that exact inner-product test \eqref{eq:inner product test} and exact orthogonality test \eqref{eq:orthogonality test} are satisfied, and $\eta_t \le 1 / ((1 + \theta^2 + \nu^2)L)$, then Algorithm \ref{alg:adabatchgrad} converges as \eqref{eq:bollaprogada non-convex}:
        \begin{equation*}
            \min_{1 \le t \le T} \E \left[\| \nabla f(\x_t) \|^2 \right] \le \frac{2}{\alpha T}(f(\w_0) - f_{\min}),
        \end{equation*}
    \end{proposition}

    To finalize, we get a method, that is robust to \textit{test inconsistency}. It means that the method either converges with AdaGrad rate in the case where the test fails or converges with Gradient descent rate in the case where the test is accurate.
    
    We show the practical benefits of the method both with adaptive step size and with adaptive batch size in Section \ref{sec:numericals}.

    \subsection{AdaBatchGrad with Norm Test}\label{sec:adabatchgrad with norm test}

        In this Section, we show, that if additionally to the assumption on boundness of variance \eqref{eq:boundness of variance}, the batch size $|S_k|$ is chosen in such a way that it satisfies exact norm test \eqref{eq:norm test} on each step of Algorithm \ref{alg:adabatchgrad}, we can get convergence rate $O(1/\e)$ in convex and $O(1/\e^2)$ in non-convex case.
        In contrast to Theorems \ref{th:bollaprogada convex} and \ref{th:bollaprogada non-convex}, we do not need to know Lipschitz constant to get this result.

        Firstly, we speak about convex case.
        The following theorem is based on Theorem \ref{th:li-orabona convex}.
        \begin{theorem}\label{thm:convex case}
            For problem \eqref{eq:convex_problem} with convex, differentiable, smooth \eqref{eq:smoothness} function $f(x)$, where we have access to unbiased stochastic gradients, let mini-batch approximations \eqref{eq:mini-batch approximation} satisfy norm-test \eqref{eq:norm test} on every iteration. 
            Let the step size be as in \eqref{eq:global step size},
            where $\alpha, \beta > 0, 0 \le \tau < \frac{1}{2}$, and $4L\alpha(1+\omega^2) < \beta^{\frac{1}{2} + \tau}$.
            Denote 
            \[
                \gamma \equiv \frac{1}{\alpha \left( 1 - \frac{4L \alpha \left(1 + \omega^2 \right)}{\beta^{1/2 + \tau}} \right)} \|\x_1 - \x^{\ast}\|^2,\ \bar \x_T \equiv \frac{1}{T} \sum_{t = 1}^{T} \x_t.
            \]
            Then, SGD \eqref{eq:sgd equation} with step size \eqref{eq:global step size} 
            converges in expectation up to approximation error $\e = \E \left[f(\bar \x_T) - f(\x^{\ast})\right]$ in
            \begin{equation}
                \frac{1}{\e} \max \left\{ 2 \gamma \beta^{1/2 + \tau}; \left( 2 \gamma \left( 4 L ( 1 + \omega^2 ) \right)^{1/2 + \tau} \right)^\frac{1}{1/2 - \tau} \right\}
            \end{equation}
            iterations.
        \end{theorem}

        \begin{proof}
                Denote $\delta_t \equiv f(\x_t) - f(\x^{\ast})$, and $\Delta \equiv \sum_{t=1}^{T} \delta_t$.
                From convexity of $f$ we have
                \begin{gather}
                    f(\bar \x_T) - f(\x^{\ast}) \le \frac{1}{T} \sum_{t=1}^{T} f(\x_t) - f(\x^{\ast}) = 
                    \frac{1}{T} \sum_{t=1}^{T} \delta_t = \frac{1}{T} \Delta. \label{eq:upper bound of expectation of approximation error}
                \end{gather}
                Thus, we can search for the upper bound of $\E \left[ \Delta\right]$.

                From the definition of $\eta_t$ we can get
                \begin{align}
                    \E \left[ \Delta \right] = 
                    \E \left[\eta_T \Delta \frac{1}{\eta_T} \right] \le 
                    \E \left[ \eta_T \Delta \right] \E \left[\frac{1}{\eta_T} \right] \le 
                    \E \left[ \sum_{t = 1}^T \eta_t \delta_t \right] \E \left[\frac{1}{\eta_T} \right].
                    \label{eq:thm3:main inequality}
                \end{align}
                where in the last inequality we used the fact that $\eta_t$ monotonically decreases.

            Firstly, we estimate the left part of the r.h.s of the \eqref{eq:thm3:main inequality}. From the definition of the SGD step \eqref{eq:sgd equation} $\x_{t + 1} = \x_t - \eta_t \nabla f_{S_t}(\x_t)$ we get
            \begin{align*}
                \|\x_{t + 1} - \x^{\ast}\|^2 &= \|\x_{t + 1} - \x_t + \x_t - \x^{\ast}\|^2 \\
                &= \la \x_{t + 1} - \x_t + \x_t - \x^{\ast},\x_{t + 1} - \x_t + \x_t - \x^{\ast} \ra \\
                &= \eta_t^2 \|\nabla f_{S_t}(\x_t)\|^2 - 2 \eta_t \la \nabla f_{S_t}(\x_t), \x_t - \x^{\ast} \ra + \|\x_t - \x^{\ast}\|^2.
            \end{align*}
            Denote $\E_t [\cdot] \equiv \E [\cdot | S_1, \ldots, S_{t - 1}]$.
            Then, taking the conditional expectation of $\la \nabla f_{S_t}(\x_t), \x_t - \x^{\ast} \ra$ 
            with respect to $S_1, \ldots, S_{t - 1}$, we have that
            \begin{equation*}
                \E_{t} \left[ \la \nabla f_{S_t}(\x_t), \x_t - \x^{\ast} \ra \right] = 
                    \la \nabla f(\x_t), \x_t - \x^{\ast} \ra \ge f(\x_t) - f(\x^{\ast}) = \delta_t,
            \end{equation*}
            where in the inequality we used the fact that $f$ is convex. 
            Thus,
            \begin{align*}
                \E \left[ \sum_{t=1}^T \eta_t \delta_t \right]
                &\le \E \left[ \sum_{t=1}^T \eta_t \E_t \left[ \frac{\|\w_t - \w^{\ast}\|^2 - \|\w_{t + 1} - \w^{\ast}\|^2 + \eta_t^2 \|\nabla f_{S_t}(\w_t) \|^2}{2 \eta_t} \right] \right]\\
                &= \frac{1}{2} \E \left[ \sum_{t=1}^T \|\w_t - \w^{\ast}\|^2 - \|\w_{t + 1} - \w^{\ast}\|^2 + \eta_t^2 \|\nabla f_{S_t}(\w_t) \|^2 \right] \\
                &=\frac{1}{2} \E \left[ \|\w_1 - \w^{\ast}\|^2 - \|\w_{T + 1} - \w^{\ast}\|^2 + \sum_{t=1}^T \eta_t^2 \|\nabla f_{S_t} (\w_t)\|^2 \right] \\
                &\le \|\w_1 - \w^{\ast}\|^2 + \E \left[ \sum_{t=1}^T \eta_t^2 \|\nabla f_{S_t}(\w_t) \|^2 \right] \\
                &\overset{\eqref{eq:lemma 8 result}}{\le} \|\w_1 - \w^{\ast}\|^2 + 2(1 + \omega^2) \E \left[ \sum_{t=1}^T \eta_t^2 \|\nabla f(\w_t) \|^2 \right] \\
                &\overset{\eqref{eq:lemma 4 result}}{\le} \|\w_1 - \w^{\ast}\|^2 + 4L\eta_1(1 + \omega^2) \E \left[ \sum_{t=1}^T \eta_t \delta_t \right].
            \end{align*}
            
            Thus, we get from the definition of $\eta_1$ \eqref{eq:global step size}
            \begin{equation}\label{eq:thm3:first part estimation}
                \E \left[\sum_{t = 1}^{T} \eta_t \delta_t\right] \le \frac{1}{1 - \frac{4L \alpha \left(1 + \omega^2 \right)}{\beta^{1/2 + \tau}}} \|\x_1 - \x^{\ast}\|^2.
            \end{equation}
 
            Secondly, we estimate the right part of \eqref{eq:thm3:main inequality}
            
            \begin{align}
                \E \left[ \frac{1}{\eta_T} \right] 
                &= \E \left[
                    \frac{\left(
                        \beta + \sum_{t = 1}^{T - 1} \| \nabla f_{S_t}(\w_t)\|^2
                    \right)^{1/2 + \tau}}
                    {\alpha}
                \right]. \notag
            \end{align}
            Since $\forall x >0,\ 0 \le \tau < \frac{1}{2},\ x^{\frac{1}{2} + \tau}$ is concave, by inverse Jensen inequality \eqref{eq:Jensen}, the fact that $(x + y)^2 \le 2x^2 + 2y^2$ and triangle inequality we get
            \begin{align}
                \E \left[ \frac{1}{\eta_T} \right] 
                &\le \frac{1}{\alpha} \left(
                    \beta + 2\sum_{t=1}^{T-1} \left(
                        \E \left[ \|\nabla f(\w_t) - \nabla f_{S_t}(\w_t) \|^2 \right] + \E \left[ \|\nabla f(\w_t)\|^2 \right]
                    \right) 
                \right)^{1/2 + \tau} \notag
            \end{align}
            Since $\E[\cdot] = \E[\E_t[\cdot]]$, we get
            \begin{align}
                \E \left[ \frac{1}{\eta_T} \right] 
                &\le \frac{1}{\alpha} \left(
                    \beta + 2\sum_{t=1}^{T-1} \left(
                        \E \left[ \|\nabla f(\w_t) - \nabla f_{S_t}(\w_t) \|^2 \right] + \E \left[ \|\nabla f(\w_t)\|^2 \right]
                    \right) 
                \right)^{1/2 + \tau}  \notag \\
                &=\frac{1}{\alpha} \left(
                    \beta + 2\sum_{t=1}^{T-1} \left(
                        \E \left[ \E_t \left[ \|\nabla f(\w_t) - \nabla f_{S_t}(\w_t) \|^2 \right] \right] + \E \left[ \|\nabla f(\w_t)\|^2 \right]
                    \right) 
                \right)^{1/2 + \tau} \notag \\
                &\overset{\eqref{eq:norm test}}{\le} \frac{1}{\alpha} \left(
                    \beta + 2(1 + \omega^2) \sum_{t=1}^{T-1} \E \left[ \|\nabla f(\w_t)\|^2 \right]
                \right)^{1/2 + \tau} \label{eq:E[1/eta_T] upper bound through sum of squares of gradient norms}\\
                &\overset{\eqref{eq:lemma 4 result}}{\le} \frac{1}{\alpha} \left(
                    \beta + 4(1 + \omega^2) L \E \left[ \Delta \right]
                \right)^{1/2 + \tau}.\label{eq:thm3:second part estimation}
            \end{align}
            
            Eventually, we get
            \begin{align}
                \E [\Delta] 
                &\overset{\eqref{eq:thm3:main inequality}}{\le} \E \left[ \sum_{t = 1}^T \eta_t \delta_t \right] \E \left[\frac{1}{\eta_T} \right] \notag \\
                &\overset{\eqref{eq:thm3:first part estimation},\eqref{eq:thm3:second part estimation}}{\le} \frac{1}{\alpha \left( 1 - \frac{4L \alpha \left(1 + \omega^2 \right)}{\beta^{1/2 + \tau}} \right)} \|\x_1 - \x^{\ast}\|^2 \left(
                    \beta + 4L(1 + \omega^2) \E \left[ \Delta \right]
                \right)^{1/2 + \tau}\label{eq:estimation of main inequality}.
            \end{align}

            Define
            \begin{equation}\label{eq:gamma definition}
                \gamma \equiv \frac{1}{\alpha \left( 1 - \frac{4L \alpha \left(1 + \omega^2 \right)}{\beta^{1/2 + \tau}} \right)} \|\x_1 - \x^{\ast}\|^2.
            \end{equation}
            Then, we obtain
            \begin{equation}\label{eq:upper bound expectation of square root of delta with epsilon}
                \E \left[\Delta\right]
                \overset{\eqref{eq:gamma definition}, \eqref{eq:estimation of main inequality}}{\le} \gamma \left( \beta + 4 L \left( 1 + \omega^2 \right) \E \left[ \Delta\right] \right)^{1/2 + \tau} \overset{\eqref{eq:(x+y)^p <= x^p + y^p}}{\le} \gamma \left( \beta^{1/2 + \tau} + (4L(1 + \omega^2))^{1/2 + \tau} (\E [\Delta])^{1/2 + \tau} \right).
            \end{equation}

            If we denote $A \equiv \beta^{1/2 + \tau}$, $B \equiv \left( 4 L ( 1 + \omega^2 ) \right)^{1/2 + \tau}$, $C \equiv \gamma$, 
            $x \equiv \E \left[ \Delta \right]$, we can apply Lemma \ref{lem:technical lemma 5} and get
            \[
                \E \left[ \frac{1}{T} \Delta \right] \le \frac{1}{T} \max \left\{ 2 \gamma \beta^{1/2 + \tau}; 
                \left( 2 \gamma \left( 4 L ( 1 + \omega^2 ) \right)^{1/2 + \tau} \right)^\frac{1}{1/2 - \tau} \right\}.
            \]
            Then, using  \eqref{eq:upper bound of expectation of approximation error}, we get
            \[
                \E \left[ f(\bar \x_T) - f(\x^{\ast}) \right] 
                \le \E \left[ \frac{1}{T} \Delta \right] \le \frac{1}{T} \max \left\{ 2 \gamma \beta^{1/2 + \tau}; 
                \left( 2 \gamma \left( 4 L ( 1 + \omega^2 ) \right)^{1/2 + \tau} \right)^\frac{1}{1/2 - \tau} \right\}.
            \]
        \end{proof}

        Let us fix optimal parameters to show the best convergence rate.
        \begin{corollary}
            Denote $R \equiv \|\x_1-\x^{\ast}\|$. For $\alpha = R$, $\tau=0$, $\beta = (8 \alpha L (1+\omega^2))^\frac{1}{1/2 + \tau}$, we get $\gamma = R$ and SGD \eqref{eq:sgd equation} with step size \eqref{eq:global step size} in convex case converges 
            as
             \begin{equation}\label{eq:easier convergence for convex case}
                T
                \le O\left(\frac{LR^2(1 + \omega^2)}{\e} 
                    \right).
            \end{equation}
        \end{corollary}

        Then, we speak about non-convex case.
        The following theorem is based on Theorem \ref{th:li-orabona non-convex}.
        \begin{theorem}\label{thm:non-convex}
            For problem \eqref{eq:convex_problem} with non-convex, differentiable, smooth \eqref{eq:smoothness} function $f(\w)$, where we have access to unbiased stochastic gradients\eqref{eq:unbiasedness}, let mini-batch approximations \eqref{eq:mini-batch approximation} satisfy norm-test \eqref{eq:norm test} on every iteration. 
            Let the step size be as in \eqref{eq:global step size},
            where $\alpha, \beta > 0, 0 \le \tau < \frac{1}{2}$, and $\alpha L (1 + \omega^2) < \beta^{\frac{1}{2} + \tau}$.  
            Denote
            \[
                \gamma \equiv \frac{f(\x_1) - f^{\ast}}{\alpha (1 - L (1 + \omega^2)\eta_1) },
            \]
                     
            Then, SGD \eqref{eq:sgd equation} with step size \eqref{eq:global step size} converges to approximation error $\E \left[ \min_{1\le t \le T} \| \nabla f(\w_t) \| \right] = \e$ in
            \[
                T
                \le \frac{1}{\e^2} \max
                \left\{
                    2 \gamma \beta^{1/2 + \tau};
                    \left(
                        2 \gamma \left(2 (1 + \omega^2) \right)^{1/2 + \tau}
                    \right)^{1/2 - \tau}
                \right\}
            \]
            iterations.
        \end{theorem}

        \begin{proof}
            Denote $\delta_t \equiv \|\nabla f(\x_t)\|^2$, and $\Delta \equiv \sum_{t=1}^{T} \delta_t$. 
       
            We know that 
            \begin{align}
                \E \left[ \Delta \right] &= \E \left[ \left( \sum_{t = 1}^T \|\nabla f(\x_t) \|^2 \right)\right] \notag \\
                &\ge \E \left[ \left( T \min_{1 \le t \le T} \|\nabla f(\x_t) \|^2 \right) \right] \notag \\
                &= T \E \left[ \left( \min_{1 \le t \le T} \|\nabla f(\x_t) \|^2 \right) \right]. \label{eq:thm4:minimal gradient norm upper bound}
            \end{align}
            So, now we can work with $\E \left[ \Delta \right]$.

            From \eqref{eq:Holder} we have
            \begin{align}
                \E \left[ \Delta \right] &= \E \left[ \left( \sum_{t = 1}^T \eta_t \frac{1}{\eta_t} \|\nabla f(\x_t)\|^2 \right)^{1/2-\tau} \right] \notag \\
                &\overset{\eqref{eq:global step size}}{\le} \E \left[ \left( \sum_{t=1}^{T} \eta_t \|\nabla f(\x_t) \|^2 \right) \left( \frac{1}{\eta_T} \right) \right] \notag \\
                &\overset{\eqref{eq:Holder}}{\le} \left( \E \left[ \sum_{t=1}^{T} \eta_t \|\nabla f(\x_t) \|^2 \right) \right) \left( \E \left[ \left( \frac{1}{\eta_T} \right) \right] \right) \label{eq:thm4:main inequality}.
            \end{align}

            Firstly, consider the left term from the r.h.s. of \eqref{eq:thm4:main inequality}. From Lemmas \ref{lem:lemma 3} and \ref{lem:modification of lemma 8} we can see that
            \begin{align*}
                \E \left[ \sum_{t=1}^T \eta_t \|\nabla f(\x_t)\|^2 \right] 
                    &\overset{\eqref{eq:lemma 3 result}}{\le} f(\x_1) - f^{\ast} + \frac{L}{2} \E \left[ \sum_{t = 1}^T \eta_t^2 \|\nabla f_{S_t}(\x_t)\|^2  \right] \\
                    &\overset{\eqref{eq:lemma 8 result}}{\le} f(\x_1) - f^{\ast} + L (1 + \omega^2) \E \left[ \sum_{t = 1}^T \eta_t^2 \|\nabla f(\x_t)\|^2 \right] \\
                    &\le f(\x_1) - f^{\ast} + L (1 + \omega^2) \eta_1 \E \left[ \sum_{t = 1}^T \eta_t \|\nabla f(\x_t)\|^2 \right].
            \end{align*}
            Thus,
            
            \begin{equation}\label{eq:thm4:first part estimation}
                \E \left[ \sum_{t=1}^T \eta_t \|\nabla f(\x_t)\|^2 \right]  \le \frac{f(\x_1) - f^{\ast}}{1 - L \eta_1 (1 + \omega^2)}.
            \end{equation}

            Secondly, consider the right term from the r.h.s. of \eqref{eq:thm4:main inequality}. 
            From \eqref{eq:E[1/eta_T] upper bound through sum of squares of gradient norms} we know that
            \begin{align}
                \E\left[ \frac{1}{\eta_T} \right]
                &\overset{\eqref{eq:E[1/eta_T] upper bound through sum of squares of gradient norms}}{\le} 
                \frac{1}{\alpha}
                \left(
                      \beta + 2(1+\omega^2)\sum_{t=1}^{T-1} \E \left[\|\nabla f(\x_t) \|^2 \right]  
                \right)^{1/2 + \tau} \notag \\
                &\overset{\eqref{eq:(x+y)^p <= x^p + y^p}}{\le}
                \frac{1}{\alpha}
                \left(
                      \beta^{1/2 + \tau} + (2(1+\omega^2))^{1/2 + \tau} \left(\E \left[ \Delta \right]\right)^{1/2 + \tau}
                \right).\label{eq:thm4:second part estimation}
            \end{align}

            Now we can get the estimation of $\E \left[ \Delta \right]$:
            \begin{align}
                \E \left[ \Delta \right]
                    &\overset{\eqref{eq:thm4:main inequality}}{\le} \left( \E \left[ \sum_{t=1}^{T} \frac{1}{\eta_t} \|\nabla f(\x_t) \|^2 \right] \right) 
                    \left( \E \left[ \left( \frac{1}{\eta_T} \right) \right] \right) \notag \\
                    &\overset{\eqref{eq:thm4:first part estimation}, \eqref{eq:thm4:second part estimation}}{\le} 
                    \frac{f(\x_1) - f^{\ast}}{\alpha (1 - L (1 + \omega^2)\eta_1) } 
                    \left( 
                        \beta^{1/2 + \tau} + (2(1 + \omega^2))^{1/2 + \tau} \left(\E \left[ \Delta \right] \right)^{1/2 + \tau}
                    \right) \label{eq:thm4:estimation of main inequality}.
            \end{align}

            Define
            \begin{equation}\label{eq:thm4:gamma definition}
                \gamma \equiv \frac{f(\x_1) - f^{\ast}}{\alpha (1 - L (1 + \omega^2)\eta_1) },
            \end{equation}            
            then
            \begin{equation}\label{eq:thm4:estimation for main inequality with gamma for e>0}
                \E \left[ \Delta \right] \le 
                \gamma
                \left( 
                    \beta^{1/2 + \tau} + (2(1 + \omega^2))^{1/2 + \tau} \left(\E \left[ \Delta \right] \right)^{1/2 + \tau}
                \right).
            \end{equation}
            
            And from Lemma \ref{lem:technical lemma 5}, if we denote $x \equiv \E \left[ \Delta \right],\ C \equiv \gamma,\ A \equiv \beta^{1/2 + \tau},\ B \equiv (2(1 + \omega^2))^{1/2 + \tau}$ we get that
            \begin{equation*}
                \E \left[ \Delta \right] 
                \le 
                \max \left\{ 
                    2\gamma \beta^{1/2 + \tau};
                    \left(
                        2 \gamma 
                        \left( 
                            2(1 + \omega^2)
                        \right)^{1/2 + \tau}
                    \right)^\frac{1}{1/2 - \tau}
                \right\}.
            \end{equation*}            %
            Thus, from Jensen inequality \eqref{eq:Jensen} and \eqref{eq:thm4:minimal gradient norm upper bound} we get 
            \[
                \left(
                    \E 
                    \left[
                        \min_{1 \le t \le T} \|\nabla f(\w_t)\|
                    \right]
                \right)^2
                \le \E
                \left[
                    \min_{1 \le t \le T} \|\nabla f(\w_t)\|^2
                \right]
                \le \frac{1}{T} \max
                \left\{
                    2 \gamma \beta^{1/2 + \tau};
                    \left(
                        2 \gamma \left(2 (1 + \omega^2) \right)^{1/2 + \tau}
                    \right)^{1/2 - \tau}
                \right\}.
            \]
            By taking $\E \left[ \min_{1 \le t \le T} \|\nabla f(\w_t)\| \right] = \e$, we get the needed result.
        \end{proof}

        Let us again fix optimal parameters to show the convergence rate.

        \begin{corollary}
            Denote $\Delta_1 \equiv f(\x_1) - f^{\ast}$. For $\alpha = \sqrt \frac{\Delta_1}{L},\ \tau=0,\ \beta = 2\alpha^2 L^2 (1 + \omega^2)^2 = 2\Delta_1 (L (1 + \omega^2))^2$ SGD \eqref{eq:sgd equation} with step size \eqref{eq:global step size} in non-convex case converges to approximation error $\E \left[ \min_{1\le t \le T} \| \nabla f(\w_t) \| \right] = \e$ in
            \begin{equation}\label{eq:easier convergence in non-convex case}
                T = O \left( \frac{L \Delta_1 (1 + \omega^2)}{\e^2} \right)
            \end{equation}
            iterations.
        \end{corollary}

       Let us stress out, that instead of the assumption of boundness of variance \eqref{eq:boundness of variance}, we use the condition \eqref{eq:norm test} holds on each step of the method.
       And \eqref{eq:norm test} is a condition, that can be used to adjust the batch size of the method.

\section{Numerical experiments}\label{sec:numericals}

    Since our Algorithm \ref{alg:adabatchgrad} has both adaptive batch size and adaptive step size, we show that adjusting both of these parameters gives us much more freedom in enhancing the performance of SGD, than tweaking only one of these factors.

    To verify the efficacy of our algorithm, we conducted separate experiments on four common problems. The following are our problem settings.
    \subsection{Problem Settings}
    \paragraph{\textbf{Synthetic}\;}
    First of all, we decided to conduct experiments on the synthetic data.
        This way, we can control the noise of the data by ourselves, allowing us to demonstrate our algorithm's advantages more clearly.
        Thus, we decided to take the linear regression problem: find $\w \in \R^\n$ such that
        \begin{equation}\label{eq:lin-reg problem}
            a_i^{\top} \w = b^{\ast}_i + \xi_i,
        \end{equation}
        where $a_i \in \R^{\n}$ is some random generated vectors, $\xi_i = \mathcal{N}(0, \sigma^2)$ is a random noise, and $b_i^{\ast} = a_i^{\top} \w^{\ast}$ is a true regression output, where $\x^{\ast} \in \R^\n$ is generated from normal distribution. However, we get access only to noisy $b_i = b_i^{\ast} + \xi_i$ ,
        In our experiments, we took $N = 1000$, $\n = 20$ and $\sigma = 4$.
        Since we know the optimal solution $\x^{\ast}$, we can easily determine the approximation error of our algorithm.
        If we rewrite \eqref{eq:lin-reg problem} as an optimization problem, we get a linear regression optimization problem:
        \[
            \min_{\w \in \R^d} \frac{1}{2N}\sum_{i = 1}^N (a_i^{\top} \w - b_i)^2.
        \]
\\
 Then, we conducted similar experiments on classification tasks for datasets a9a and w8a from LIBSVM \cite{CC01a} both on convex and non-convex problems.
    \paragraph{\textbf{Convex}\;}
    For convex problem, we consider \textit{logistic regression}:
        \[
            \min_{w \in \R^d} \frac{1}{N} \sum_{i = 1}^N \log( 1 + \exp(-y_i x_i^T \w)),
        \]
        where $x_i \in \R^d,\ y_i \in \{-1, 1\}$. To measure the performance of methods on logistic regression, we use $f - f^{\ast}$.
    \paragraph{\textbf{Non-convex}\;}
    For non-convex case, we consider \textit{non-linear least squares loss (NLLSQ)}:
        \[
            \min_{w \in \R^d} \frac{1}{N} \sum_{i = 1}^N (y_i - 1 / (1 + \exp(-x_i^T \w))^2,
        \]
        where $x_i \in \R^d,\ y_i \in \{0, 1\}$. We use $\|\nabla f(x)\|$ to measure the performance of different methods.
    \paragraph{\textbf{Neural Network: CNN}\;}
    Finally, we conduct experiments on the neural network classification task.
        We take a small CNN from PyTorch tutorial \footnote{\url{https://pytorch.org/tutorials/beginner/blitz/cifar10_tutorial.html}} and train it on FashionMNIST dataset \cite{xiao2017fashionmnist} for 100 epochs.
        Since we do not know the value $f^{\ast}$ for this problem, we use accuracy.
        We provide additional plots for loss values and its gradient norm in Section \ref{sec:additional plots for fashionmnist}.

\subsection{General Trends Across Problem Settings}
We conducted six experiments on multiple datasets in each problem setting, which produced many experimental results. For easy reference, we have put all the experimental results in the table \ref{tab:experimental-setup} below. Since the results of the six experiments show the same trend for both convex and non-convex problem settings (for all datasets), we put the result plots together for better comparison.
\begin{table*}[h]
\vskip-10pt
	\centering
	\caption{Experiments Setup}  
	\label{tab:methodcompare}
	\begin{tabular}{|c|c|c|c|}
		\hline
		\multirow{2}*{Experiments} & \multicolumn{3}{c|}{problem settings}   \\ 
		\cline{2-4}
		& synthetic \& logistic & nllsq & CNN  \\
		\hline 
		SGD (3 step sizes) & Figure \ref{logreg:task1all-datasets} & Figure \ref{nllsq:task1all-datasets} & Figure \ref{nn:task1}   \\
		AdaGrad vs SGD (high $\beta$)  & Figure \ref{logreg:task2all-datasets} & Figure \ref{nllsq:task2all-datasets} & Figure \ref{nn:task2}   \\
		AdaGrad vs SGD (normal $\beta$)  & Figure \ref{logreg:task3all-datasets} & Figure \ref{nllsq:task3all-datasets} & Figure \ref{nn:task3}   \\
            SGD (3 batch sizes)  & Figure \ref{logreg:task4all-datasets} & Figure \ref{nllsq:task4all-datasets} & Figure \ref{nn:task4}   \\
            SGD vs SGD + tests  & Figure \ref{logreg:task5all-datasets} & Figure \ref{nllsq:task5all-datasets} & Figure \ref{nn:task5}   \\
            SGD vs SGD + tests vs AdaBatchGrad  & Figure \ref{logreg:task6all-datasets} & Figure \ref{nllsq:task6all-datasets} & Figure \ref{nn:task6}  \\
		\hline
	\end{tabular}
        \label{tab:experimental-setup}
        \vskip-10pt
\end{table*}
\vspace{-10pt}

    \subsection{Experiment Results}
    In this section, we show numerically, how regular SGD benefits from an increase of batch size and a decrease of step size.
    At first, we show the difference in performance of the same method, but with the different step size and batch size.
    Then we make each of these procedures adaptive separately.
    Finally, we make both step size and batch size adaptive and show how they balance each other.
    Besides the provided plots of this section, we attach additional plots for change of gradient norm, batch size, and step size in Supplementary material.
    \paragraph{\textbf{SGD with different step sizes}\;} 
        Regular SGD suffers from several problems.
        The first one comes from regular GD.
        On the one hand, if it has too big step size, it will start jump around the area of optimum, since level sets in the area of the minimum may become very narrow.
        On the other hand, if the step size is too low, it will converge to the optimum very slowly.
        The second problem with SGD is that it is not robust. 
        If we have very noisy data, the gradients are very noisy too, which results in high fluctuation of SGD even far from the optimum.
        One of the ways how to enhance the performance of SGD is by reducing its step size.
        Figure \ref{logreg:task1all-datasets},\ref{nllsq:task1all-datasets} shows us that with the step size decrease we reduce oscillation of SGD because it decreases the impact of bad gradient. 
        Additionally, it allows the algorithm to go deeper to the minimum. From Figure \ref{nn:task1}, we can again see the importance of step size when you train a neural network.
         \begin{figure}[H]
            \centering
            \subf{\includegraphics[width=0.3\textwidth]{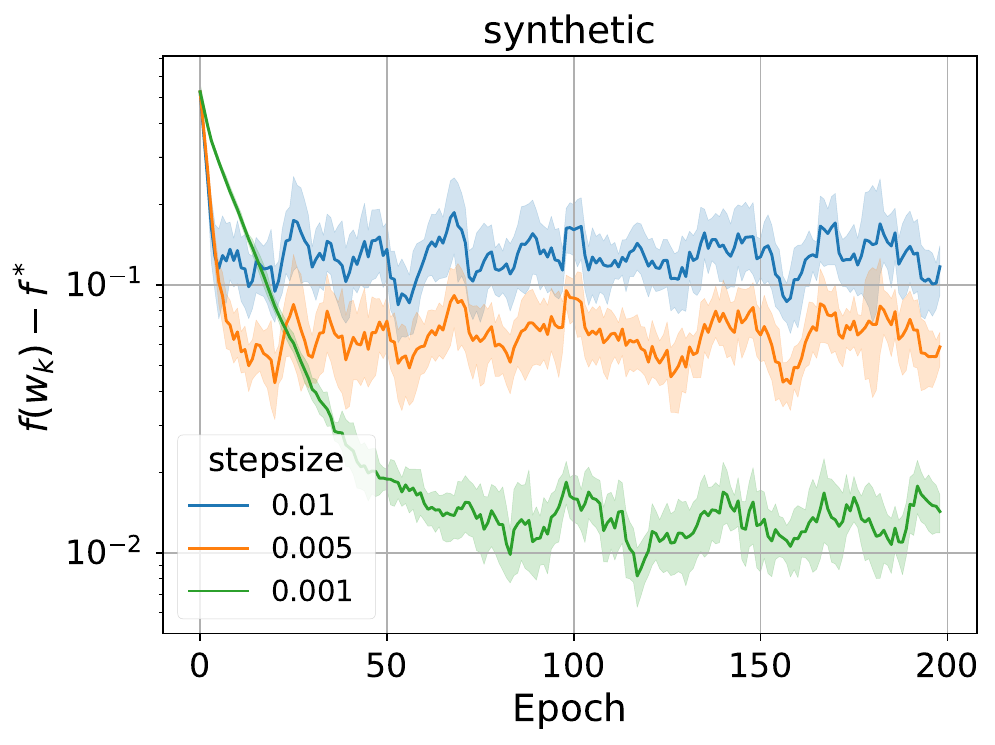}}{}
            \subf{\includegraphics[width=0.3\textwidth]{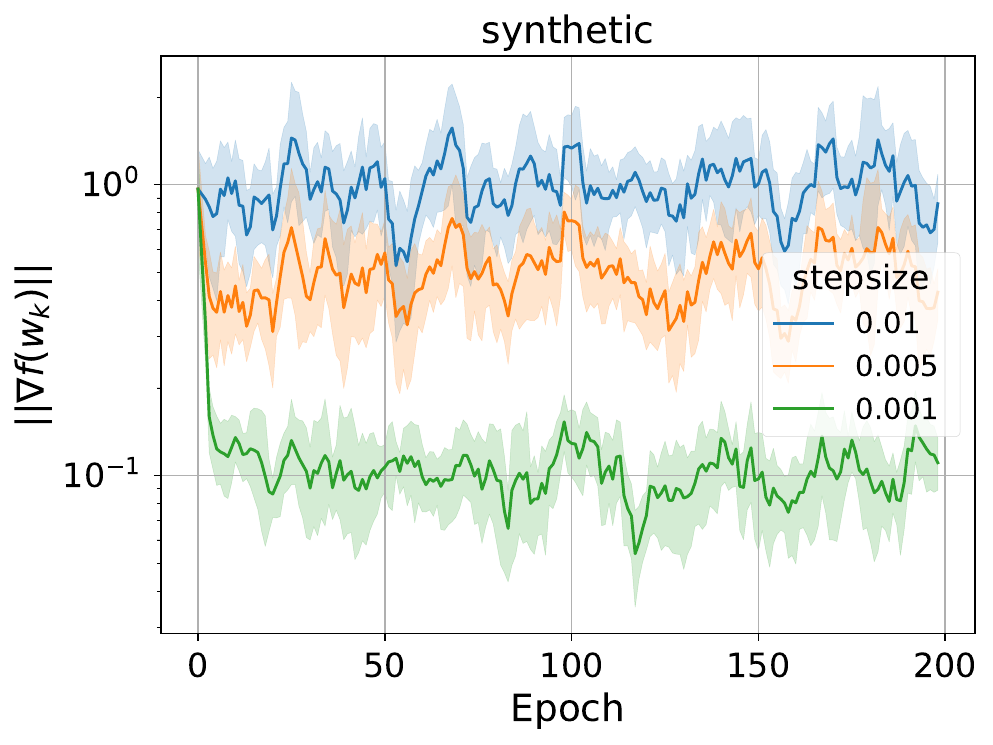}}{}\\
            \vskip-7pt
            \subf{\includegraphics[width=0.3\textwidth]{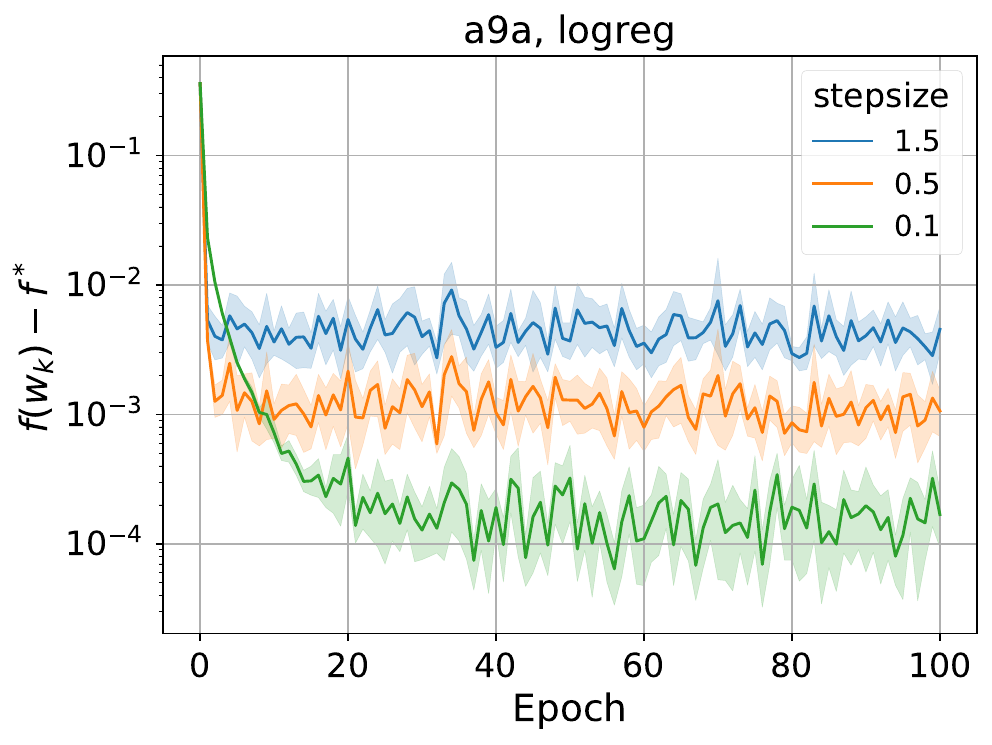}}{}
            \subf{\includegraphics[width=0.3\textwidth]{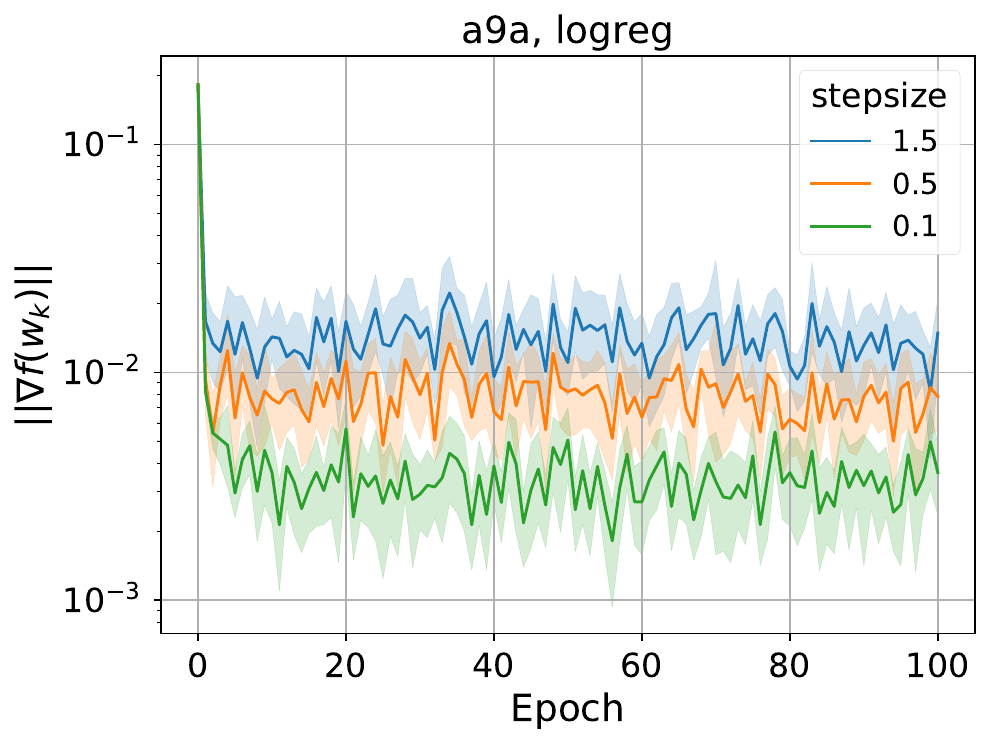}}{} \\
            \vskip-7pt
            \subf{\includegraphics[width=0.3\textwidth]{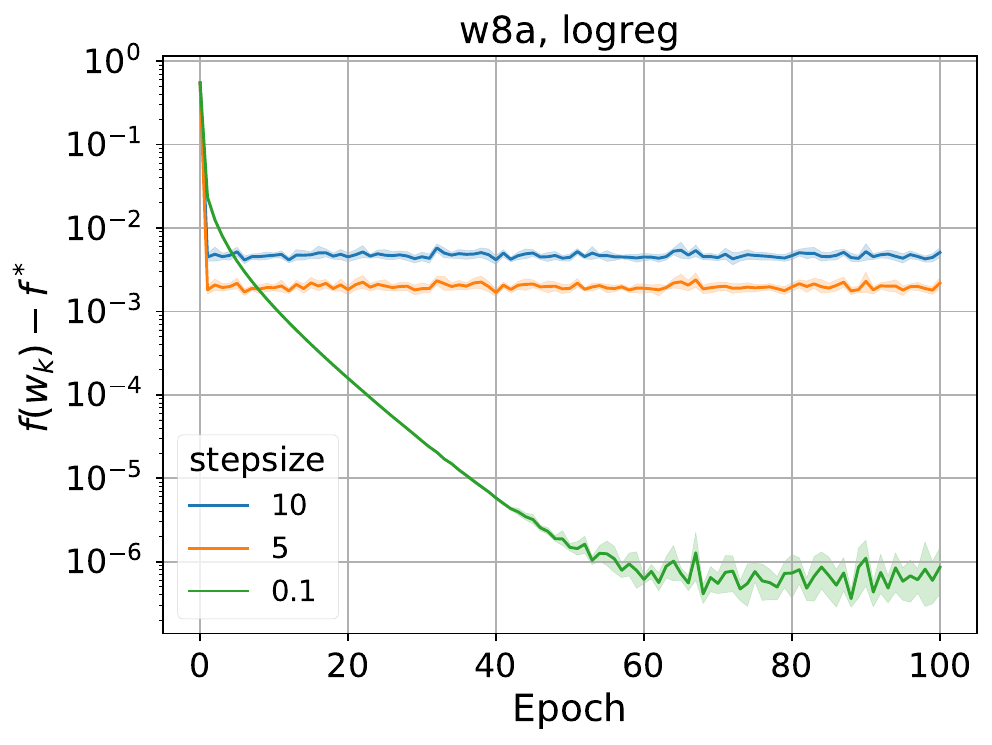}}{}              
            \subf{\includegraphics[width=0.3\textwidth]{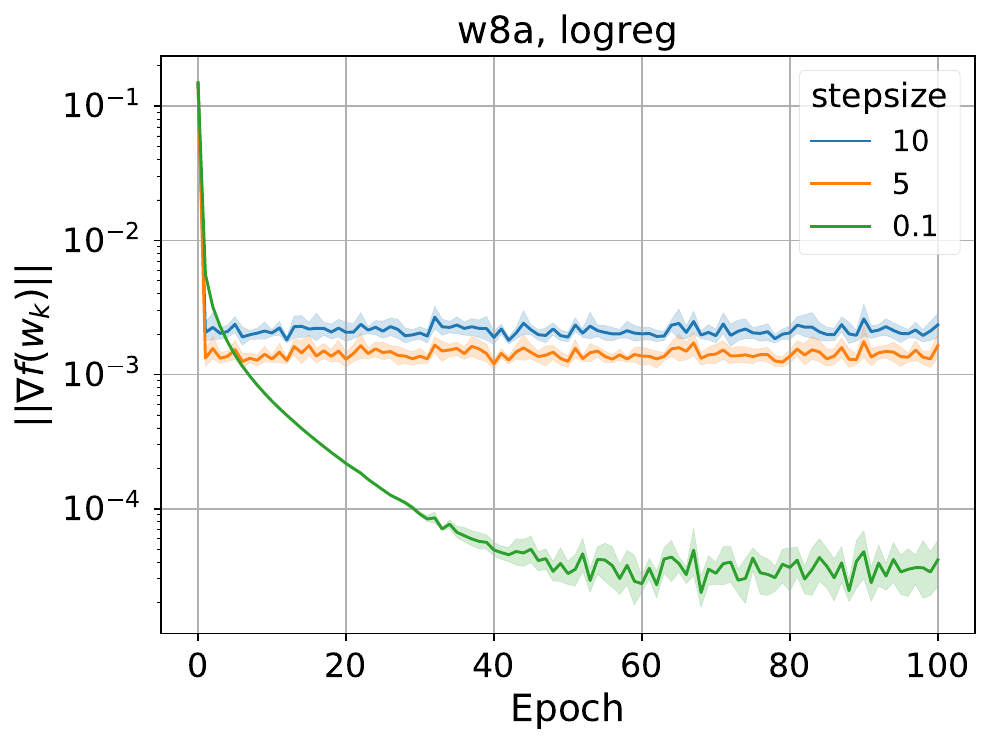}}{}
            \vskip-7pt
            \caption{\label{logreg:task1all-datasets}Results for SGD with different step sizes for synthetic (upper two plots), logistic regression on a9a (middle two plots) and on w8a (lower two plots) datasets. Smaller step size allows the algorithm to converge to a better confusion region.}
            \vskip-30pt
            \label{fig:sgd_3_lrs}
        \end{figure}
    \begin{figure}[H]
            \centering
            \subf{\includegraphics[width=0.32\textwidth]{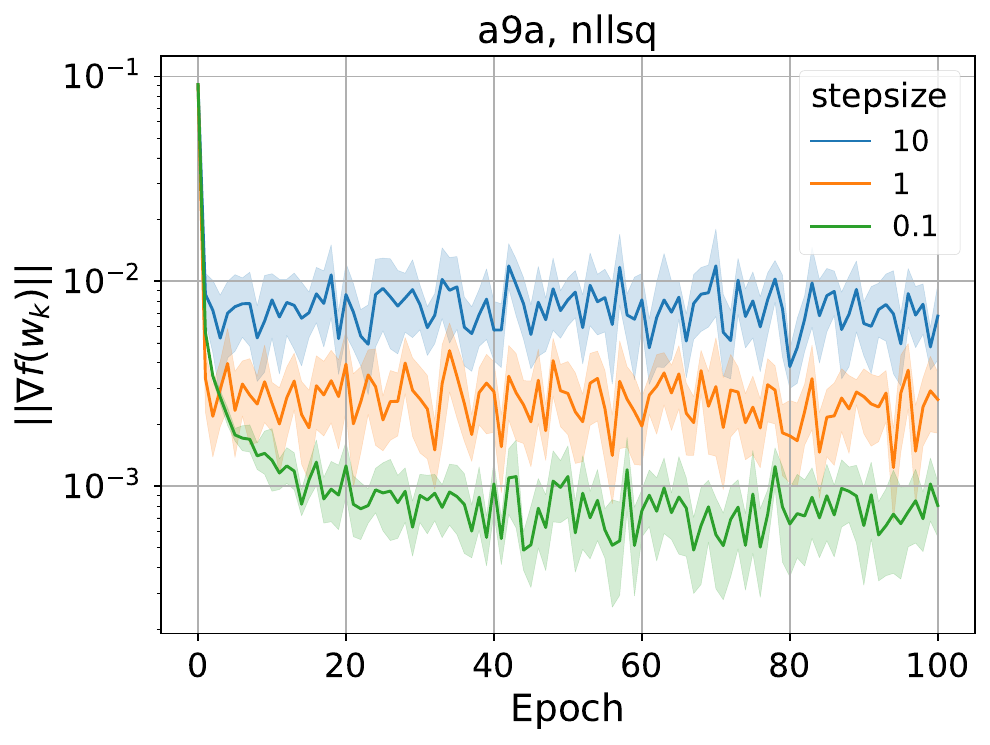}}{}
            \subf{\includegraphics[width=0.32\textwidth]{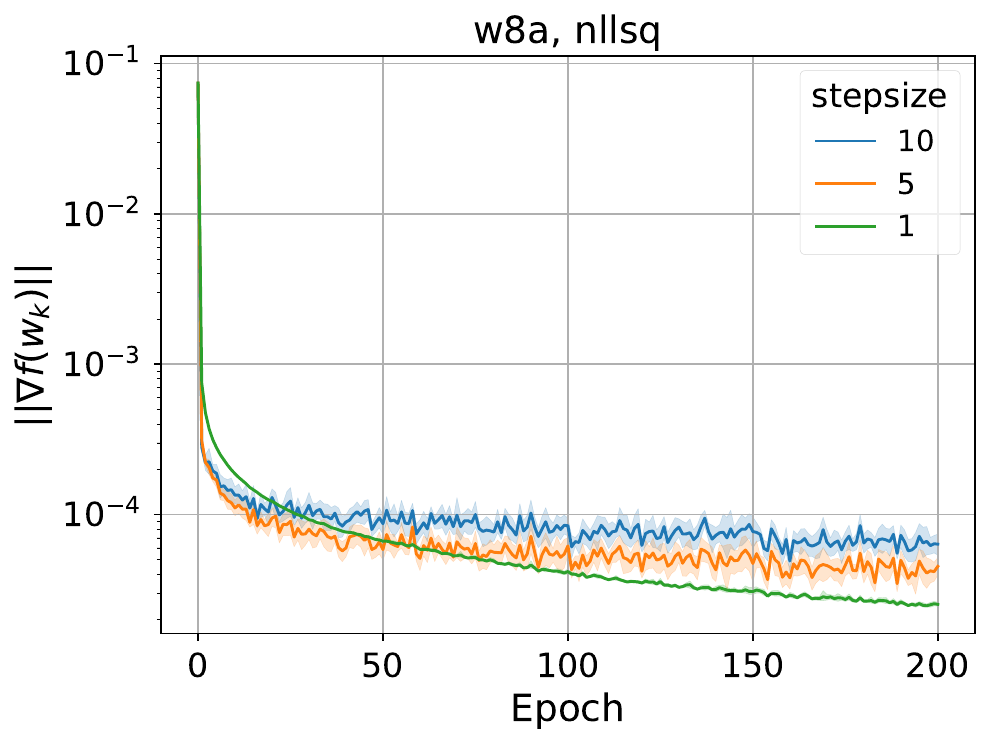}}{}
            \vskip-5pt
            \caption{\label{nllsq:task1all-datasets}Results for SGD with different step size for NLLSQ on a9a (left plot) and w8a (right plot). Smaller step size allows the algorithm to converge to a better confusion region.}
            \vskip-30pt
        \end{figure}
        \begin{figure}[H]
        \centering
            \subf{\includegraphics[width=0.34\textwidth]{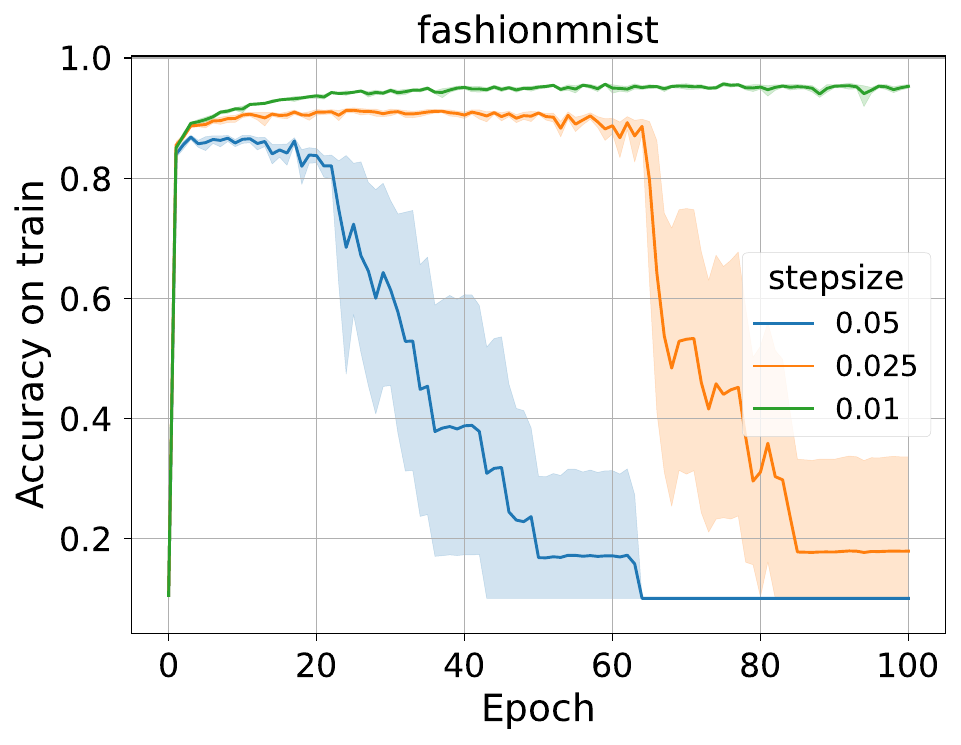}}{}
            \vskip-3pt
            \caption{\label{nn:task1}SGD with different step sizes for CNN on FashionMNIST. Smaller step size allows algorithm converge to better solution.}
            \vskip-5pt
        \end{figure}

    \paragraph{\textbf{AdaGrad vs SGD with high $\beta$}\;}
    But instead of adjusting the step size by ourselves, it is better to do it adaptively. 
        Here goes AdaGrad.
        As you can see from the AdaGrad step \eqref{eq:global step size}, it decreases in time, due to the sum in the denominator.
        However, with $\alpha$ and $\beta$, we can control the impact of this summation factor.
        To demonstrate the advantage of AdaGrad over regular SGD we fix the step size of SGD to 0.01, choose $\tau=0$ in \eqref{eq:global step size} and choose such $\alpha$ and $\beta$, that $\alpha / \sqrt{\beta} = 0.01$, which equals to step size of SGD.
        At first, we take big $\beta$.
        Since in this case $\beta$ in the denominator majorizes the sum of gradient norms, we get the similar performance of AdaGrad and SGD.
        It can be seen on Figures \ref{logreg:task2all-datasets}, \ref{nllsq:task2all-datasets}. For experiments on neural networks, as you can see in Figure \ref{nn:task2}, SGD and AdaGrad share similar behavior patterns.

\begin{figure}[H]
\vskip-40pt
            \centering
            \subf{\includegraphics[width=0.32\textwidth]{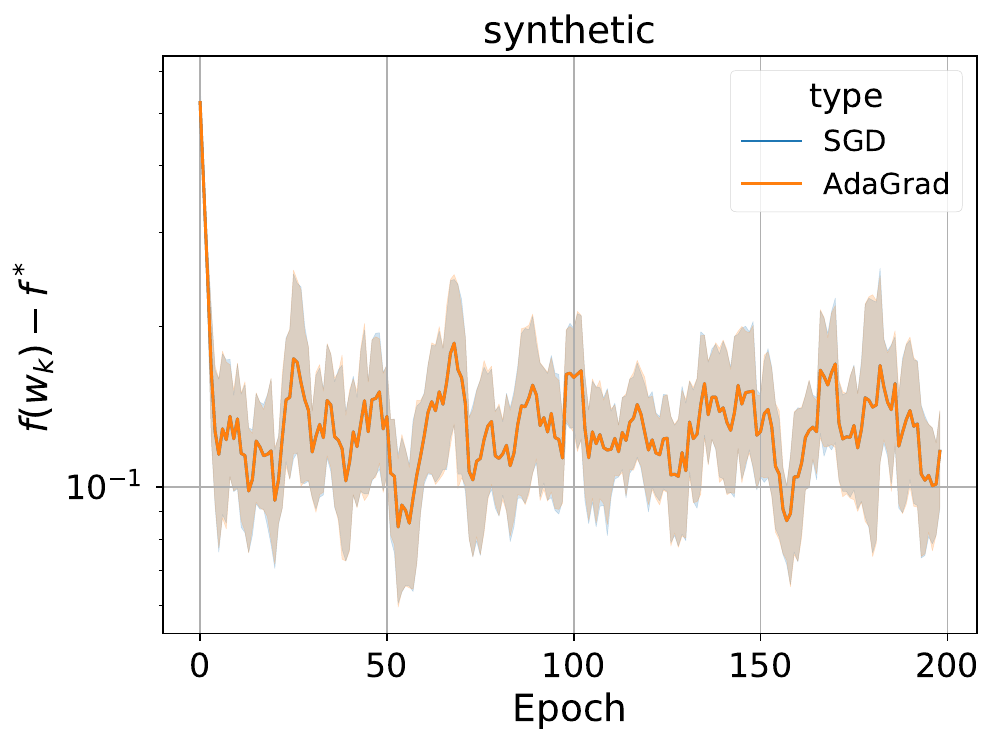}}{}  
            \subf{\includegraphics[width=0.32\textwidth]{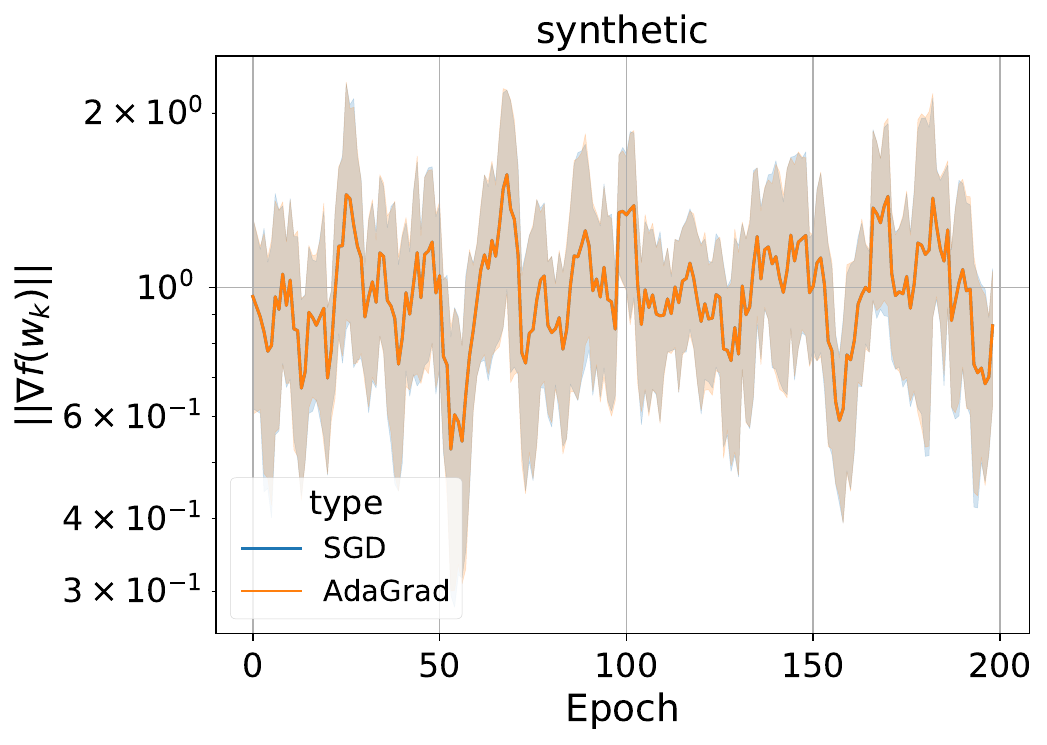}}{} 
            \subf{\includegraphics[width=0.32\textwidth]{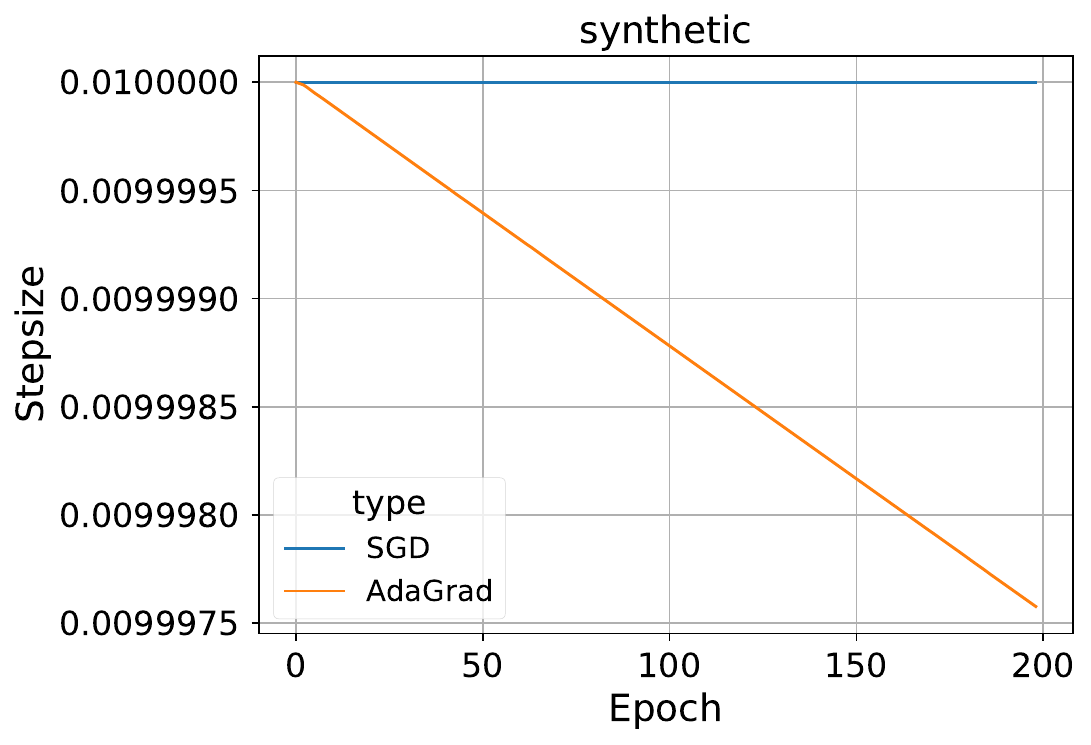}}{}
            \vskip-3pt
            \subf{\includegraphics[width=0.32\textwidth]{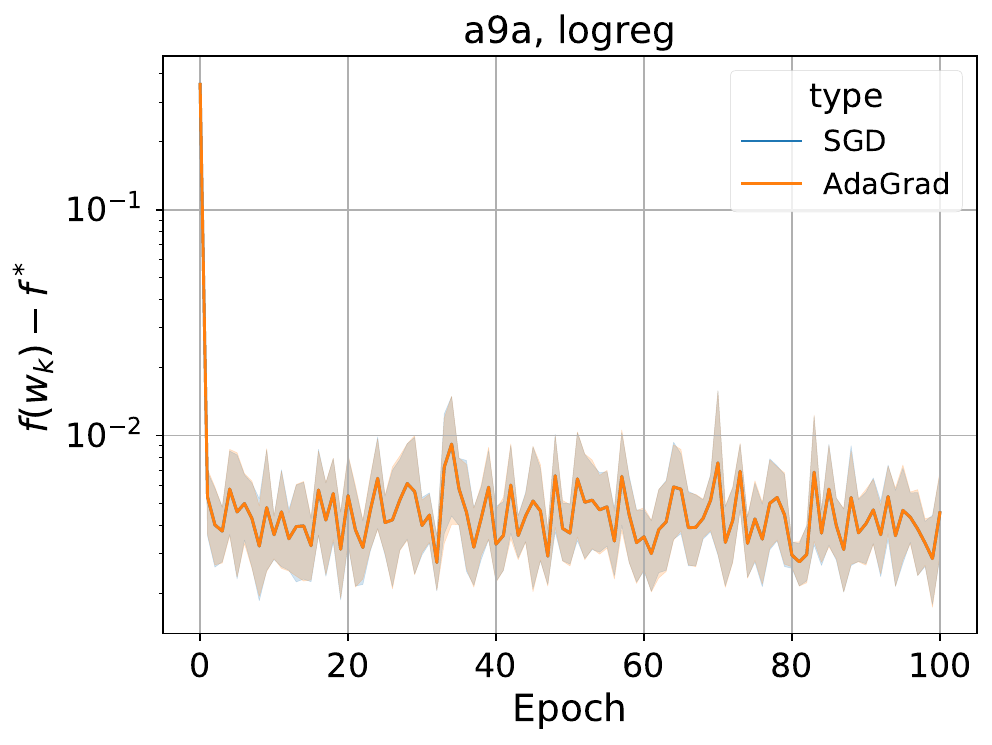}}{}  
            \subf{\includegraphics[width=0.32\textwidth]{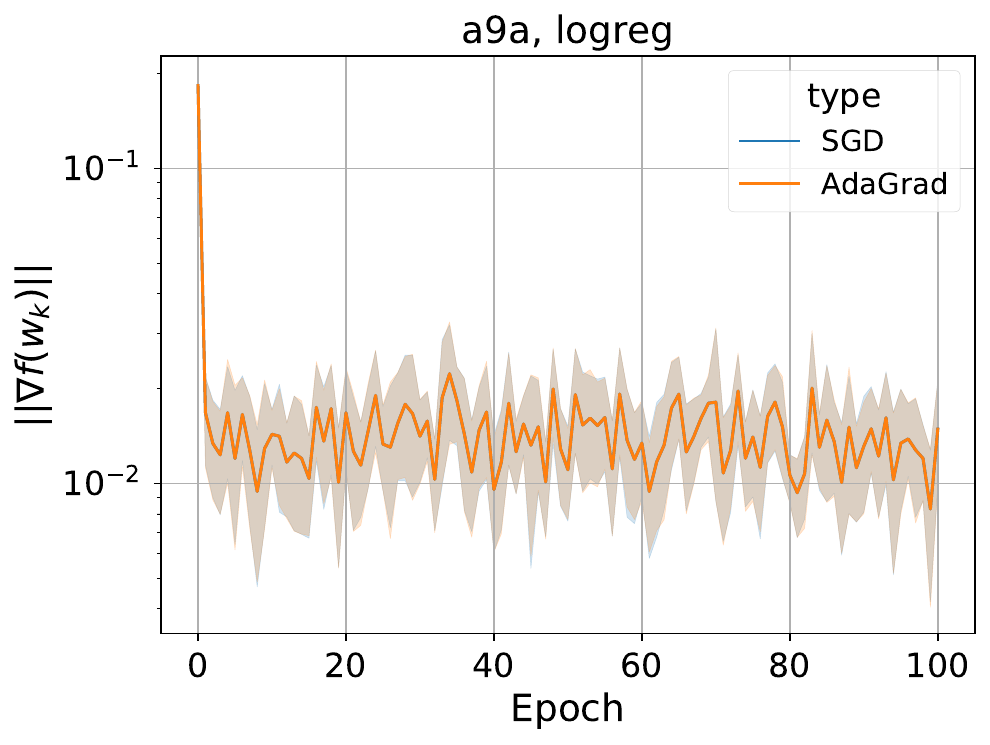}}{}  
            \subf{\includegraphics[width=0.32\textwidth]{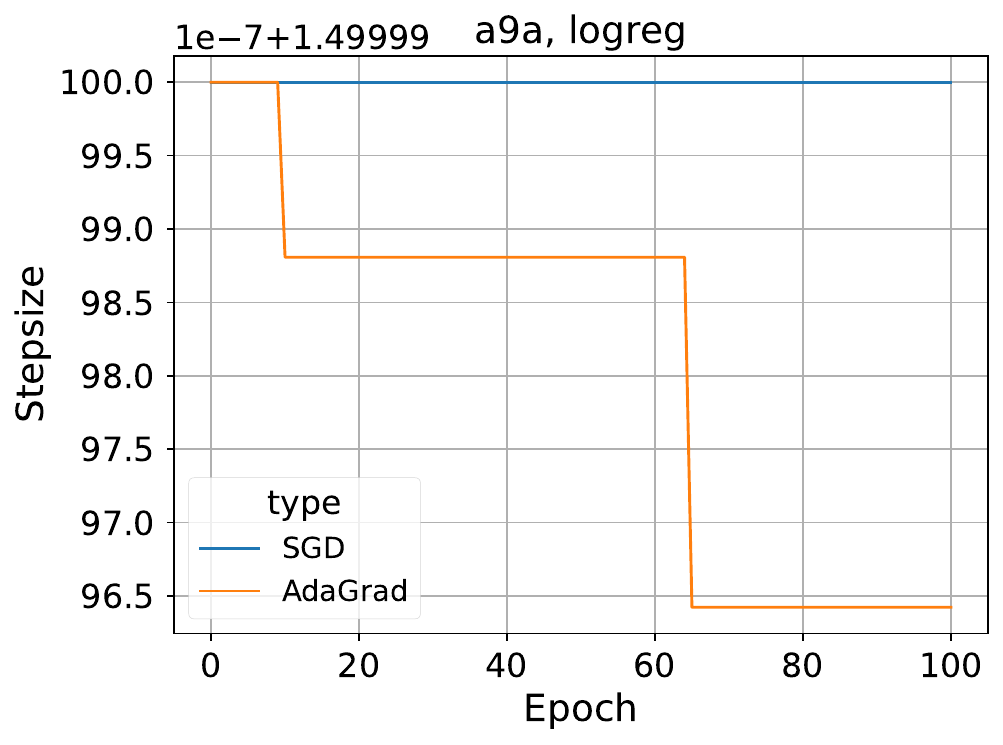}}{}  
            \vskip-3pt
            \subf{\includegraphics[width=0.32\textwidth]{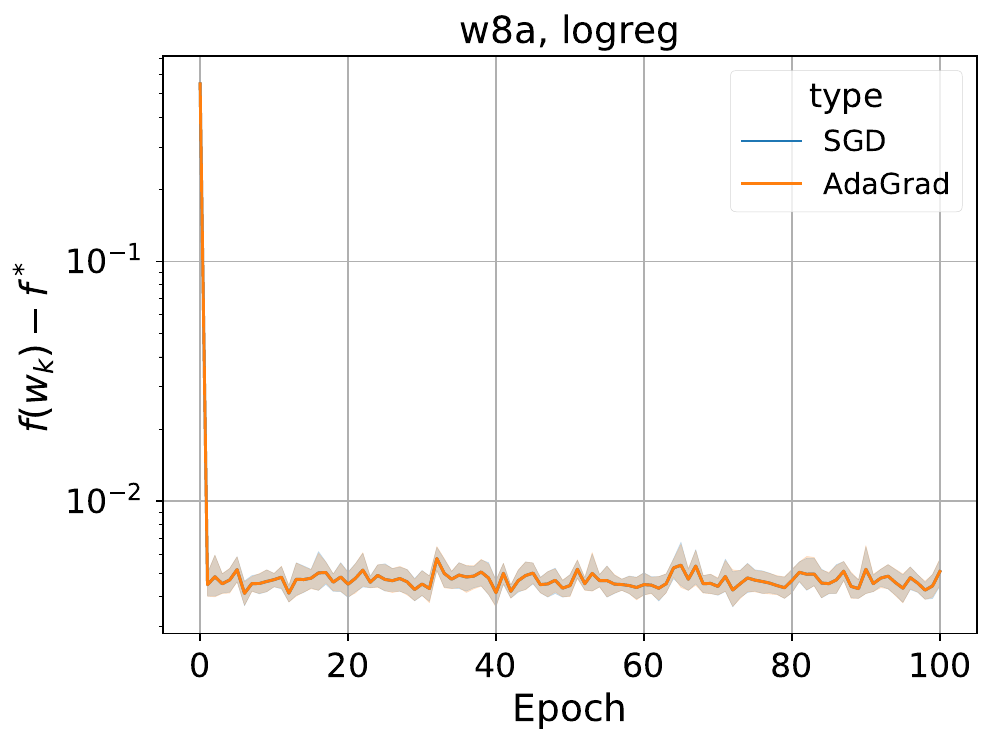}}{}  
            \subf{\includegraphics[width=0.32\textwidth]{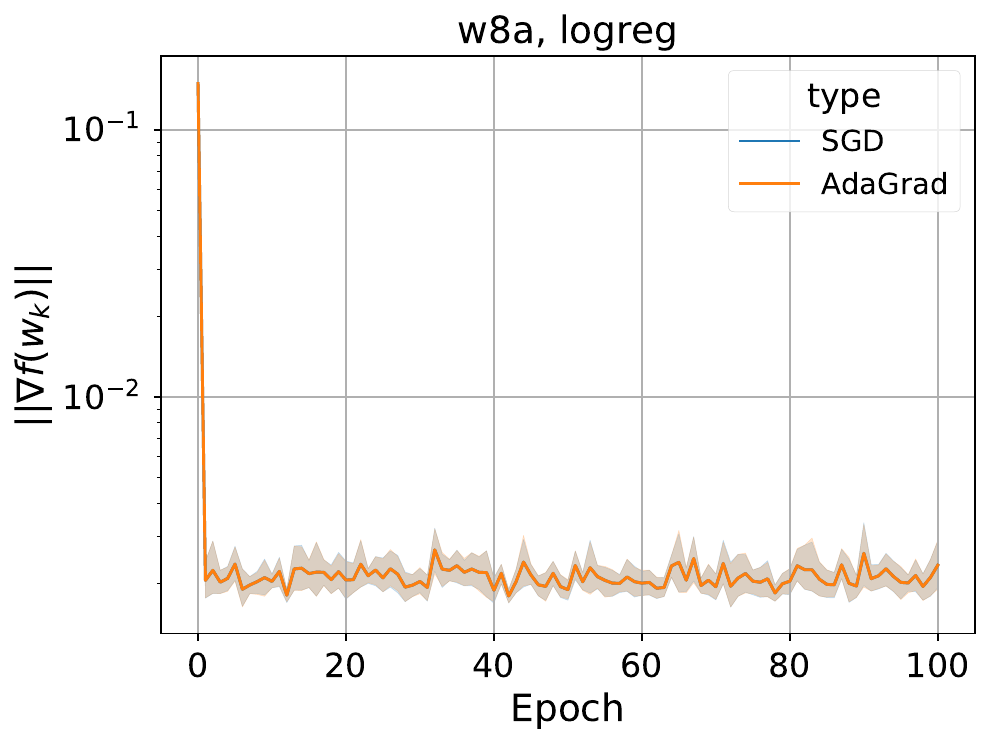}}{}  
            \subf{\includegraphics[width=0.32\textwidth]{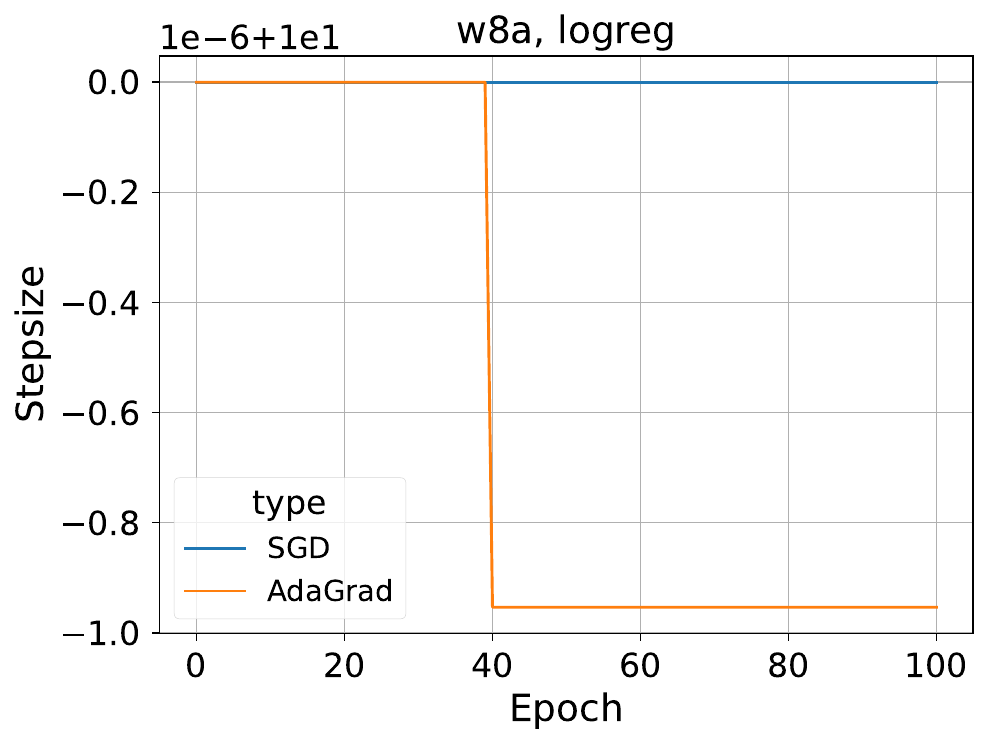}}{}
            \vskip-4pt
            \caption{\label{logreg:task2all-datasets}AdaGrad vs SGD with large $\beta$ for synthetic (upper plots), logistic regression on a9a (middle plots) and on w8a (lower plots). Here we choose $\beta$ so big that it majorizes the sum of gradient norms in the denominator of \eqref{eq:global step size}, and we have the same behavior as regular SGD. The step size is almost constant.}
            \vskip-30pt
            \end{figure}
            \begin{figure}[H]
                \centering
                \vskip-30pt
            \subf{\includegraphics[width=0.24\textwidth]{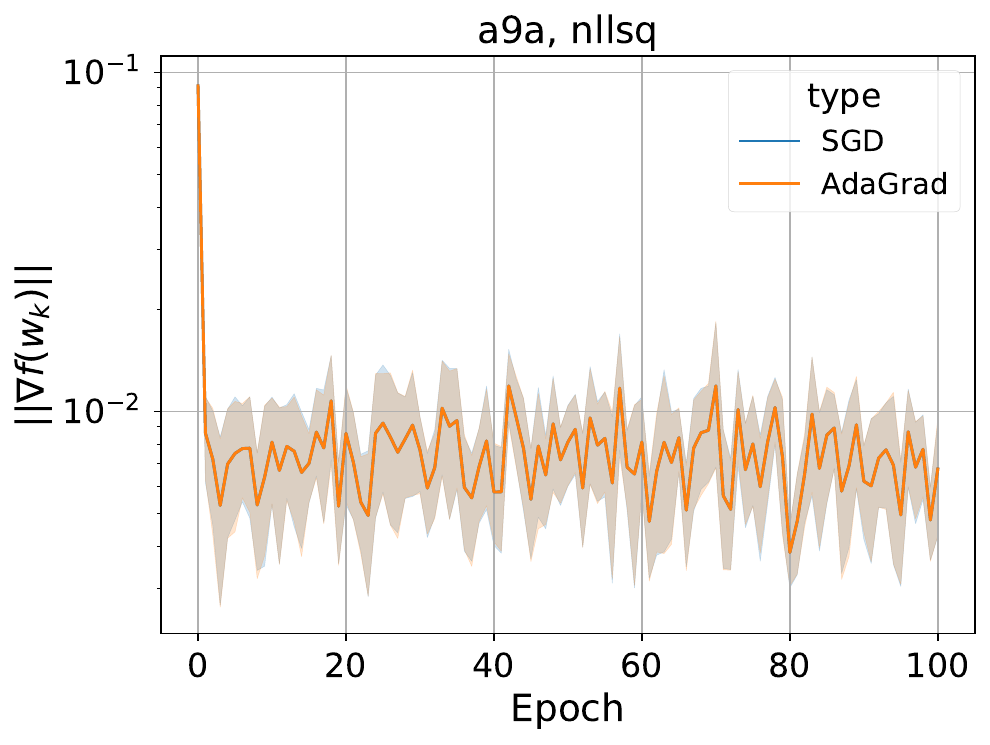}}{}  
            \subf{\includegraphics[width=0.24\textwidth]{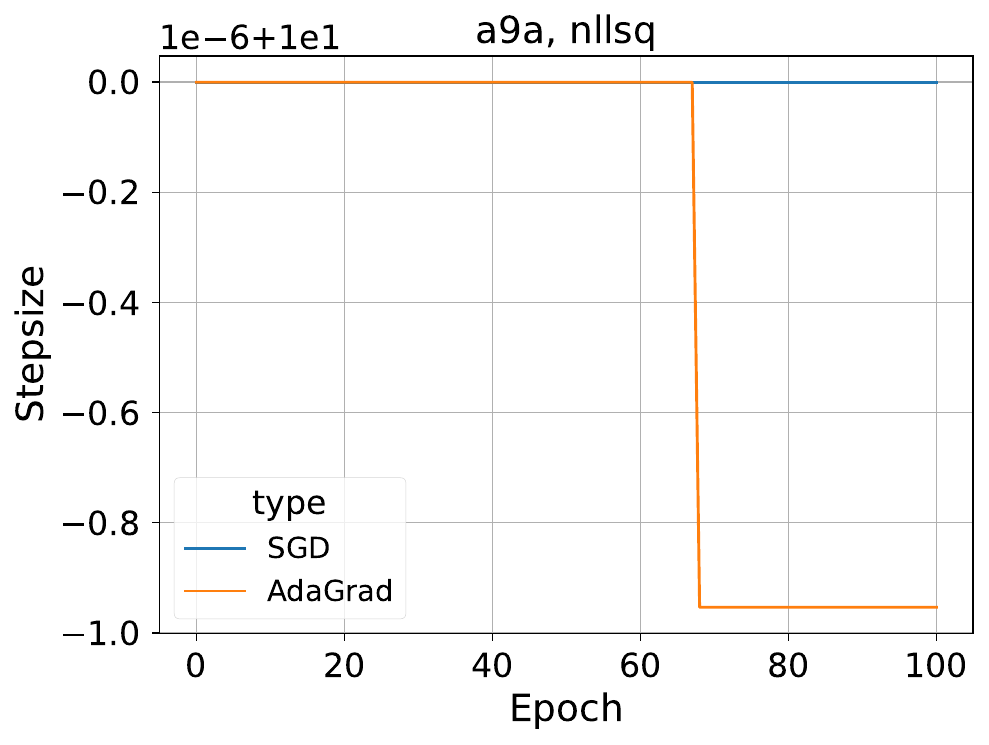}}{}   
            \subf{\includegraphics[width=0.24\textwidth]{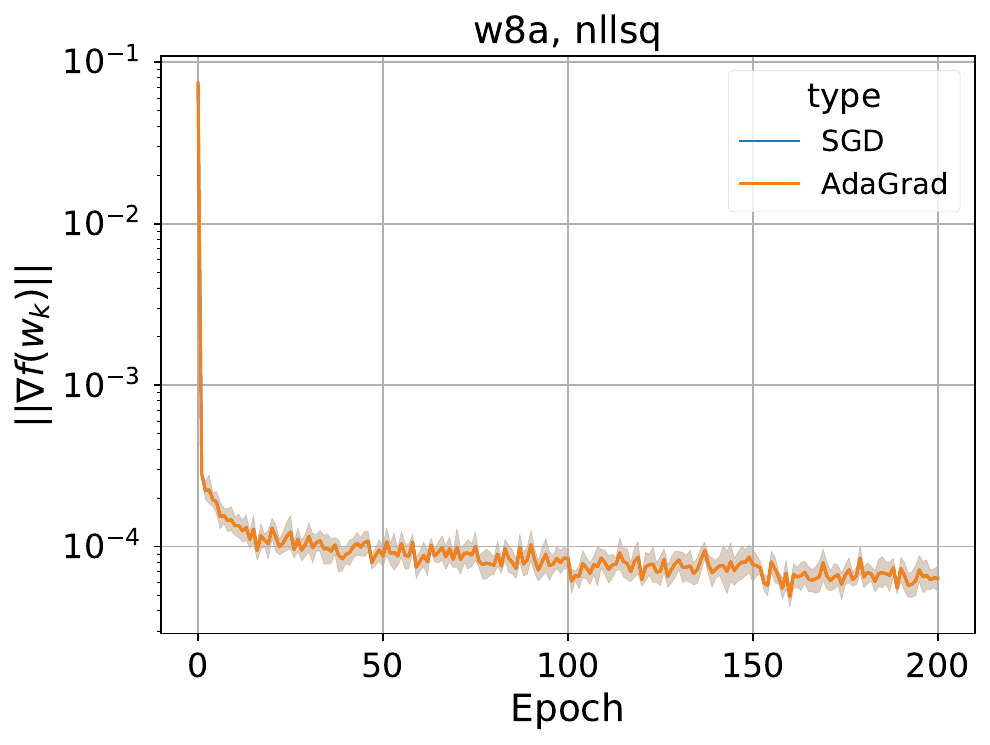}}{}  
            \subf{\includegraphics[width=0.24\textwidth]{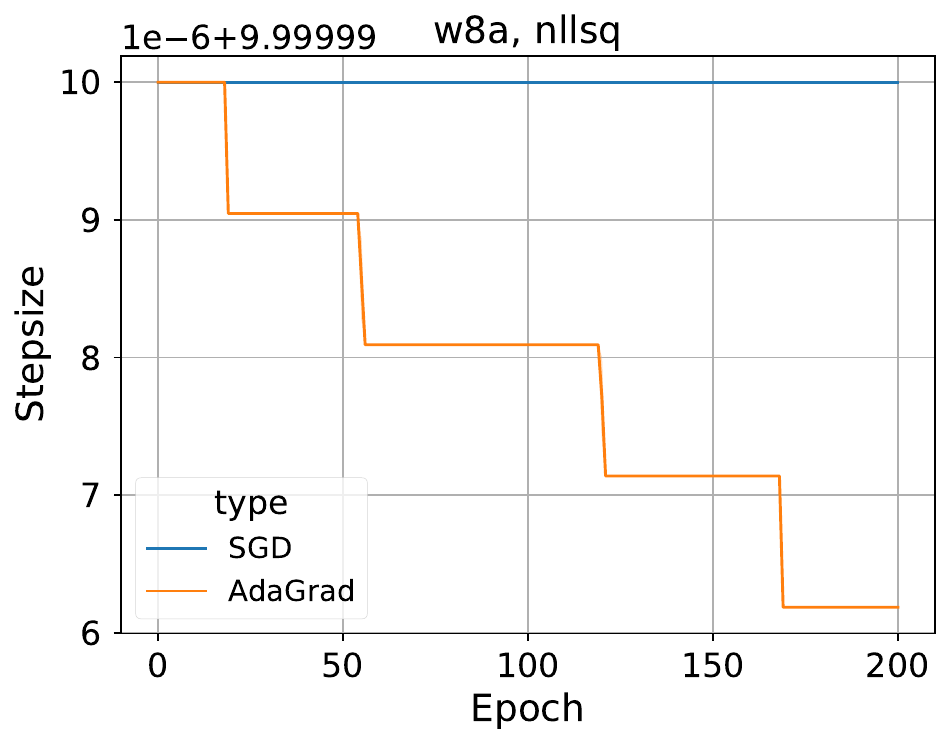}}{}    
            \vskip-3pt
            \caption{\label{nllsq:task2all-datasets}AdaGrad vs SGD with large $\beta$ for NLLSQ on a9a (left two plots) and on w8a (right two plots). Here we choose $\beta$ so big that it majorizes the sum of gradient norms in the denominator of \eqref{eq:global step size}, and we have the same behavior as regular SGD. The step size is almost constant.}
            \vskip-20pt
            \end{figure}
            
            \begin{figure}[H]
                \centering
            \subf{\includegraphics[width=0.32\textwidth]{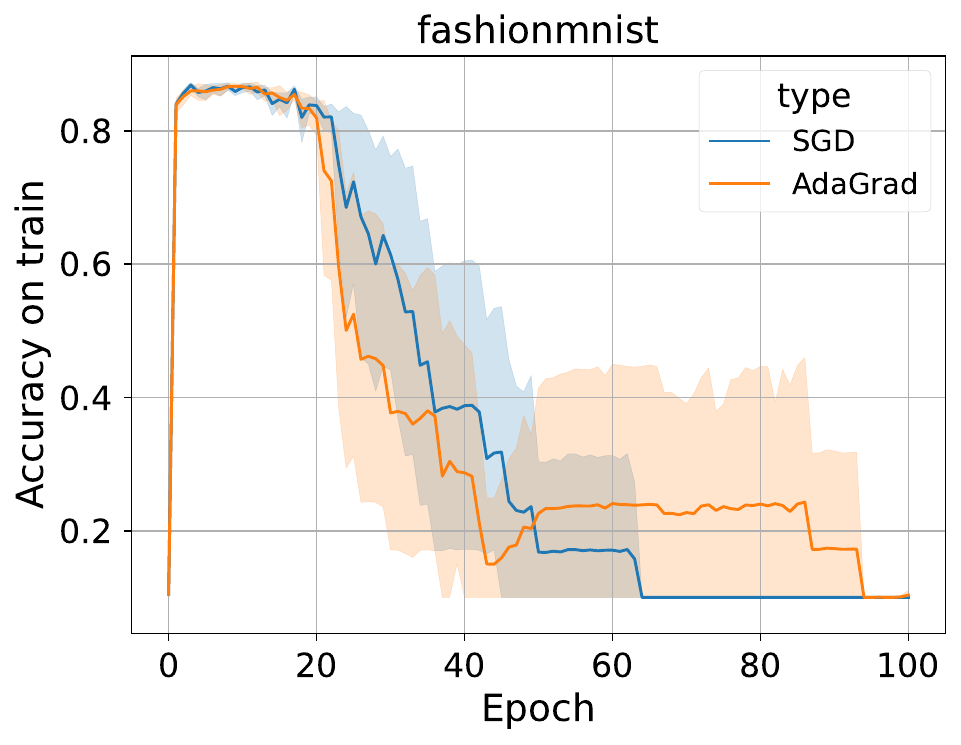}}{}
            \subf{\includegraphics[width=0.32\textwidth]{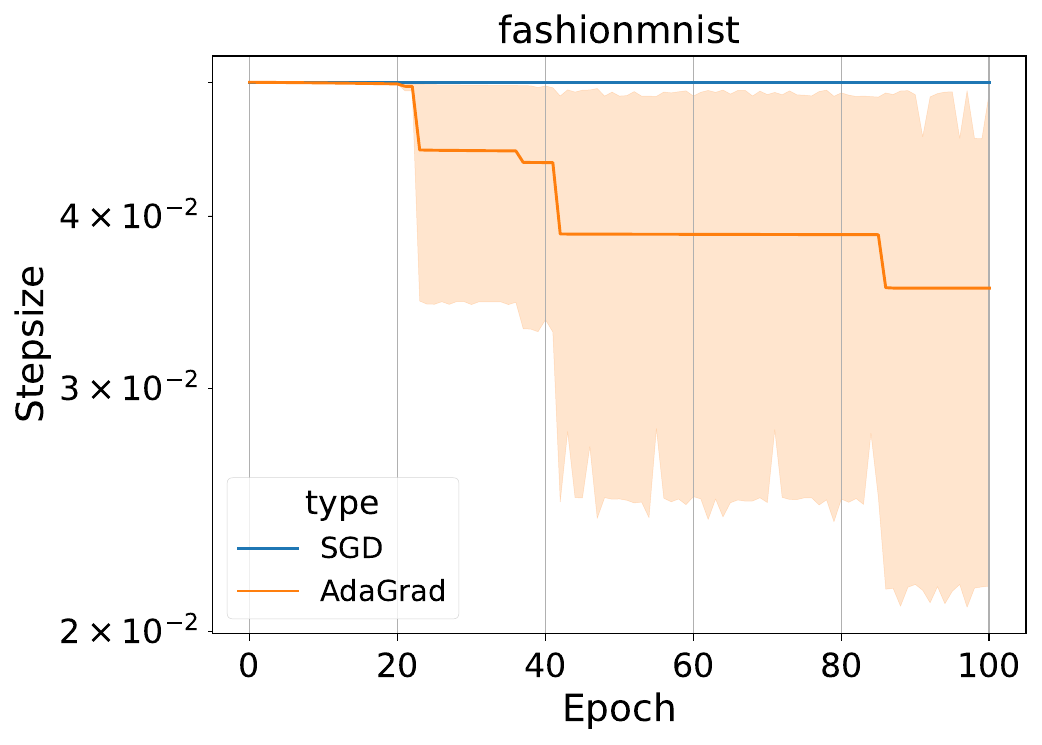}}{}
            \caption{\label{nn:task2}AdaGrad vs SGD with large $\beta$ for CNN on FashionMNIST. Here we choose $\beta$ so big that it majorizes the sum of gradient norms in the denominator of \eqref{eq:global step size}, and we have the same behavior as regular SGD. The step size is almost constant.}
            \vskip-30pt
        \end{figure}

    \paragraph{\textbf{AdaGrad vs SGD with normal $\beta$}\;}
    Then, we decrease the value of $\beta$ and get the results in Figures \ref{logreg:task3all-datasets},\ref{nllsq:task3all-datasets}.
        Now AdaGrad performs much better than SGD.
        The sum of the gradient norms has more influence on the step size, which results in deeper convergence and less stochasticity of values. For the experiment on neural network, as you can see in Figure \ref{nn:task3}, AdaGrad outperforms SGD.
        \begin{figure}[H]
        \vskip-20pt
            \centering
            \subf{\includegraphics[width=0.32\textwidth]{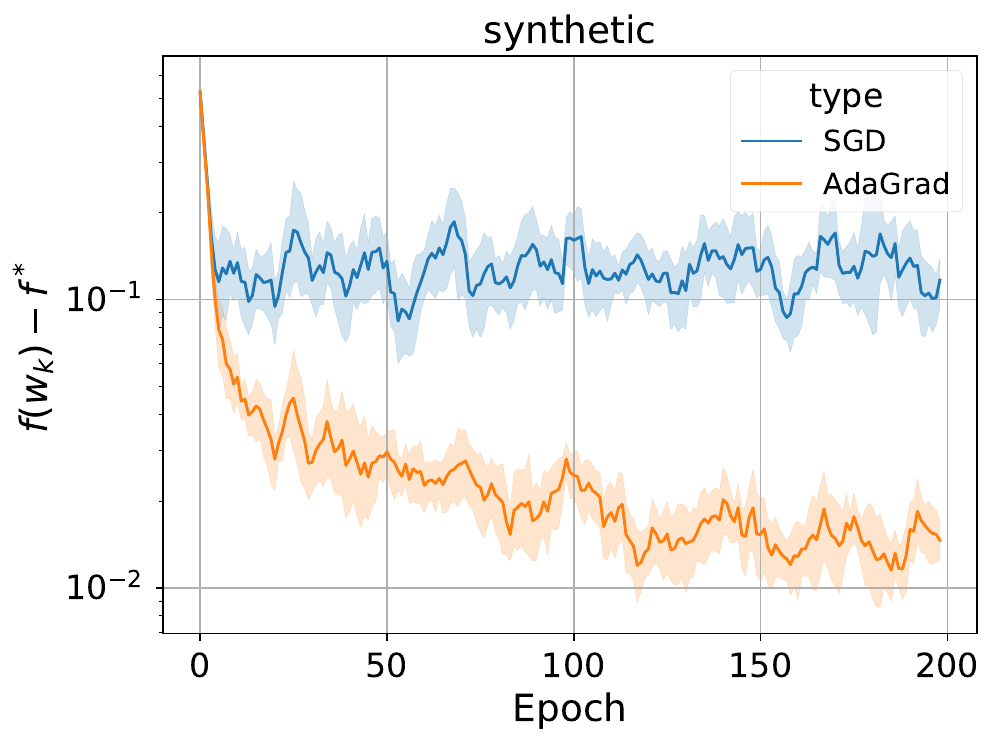}}{}  
            \subf{\includegraphics[width=0.32\textwidth]{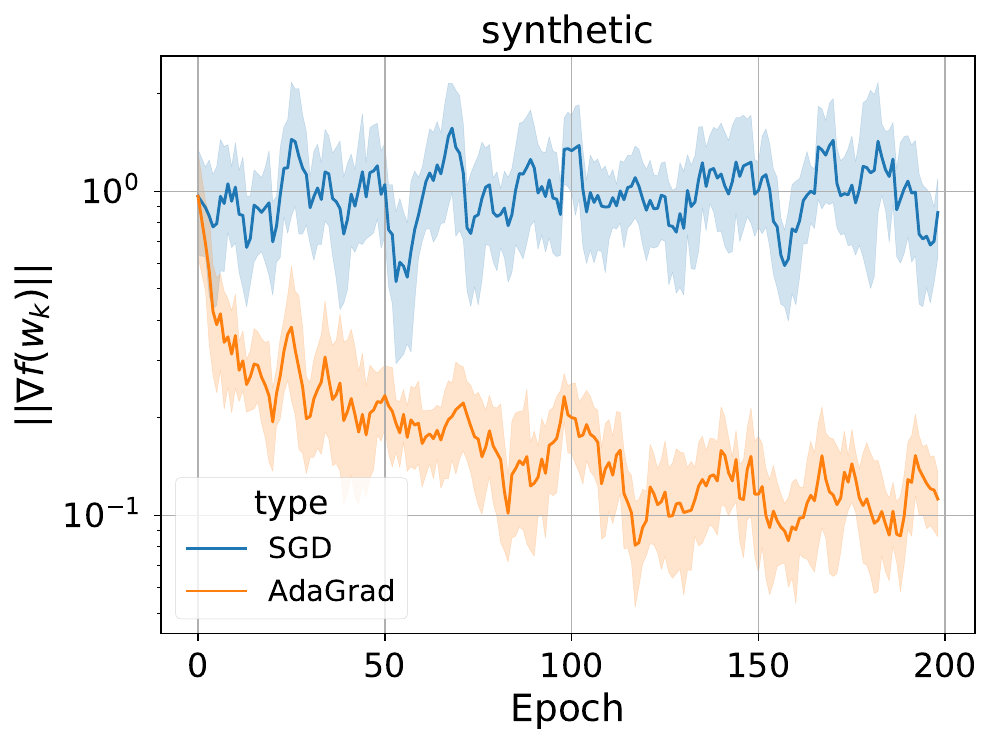}}{}  
            \subf{\includegraphics[width=0.32\textwidth]{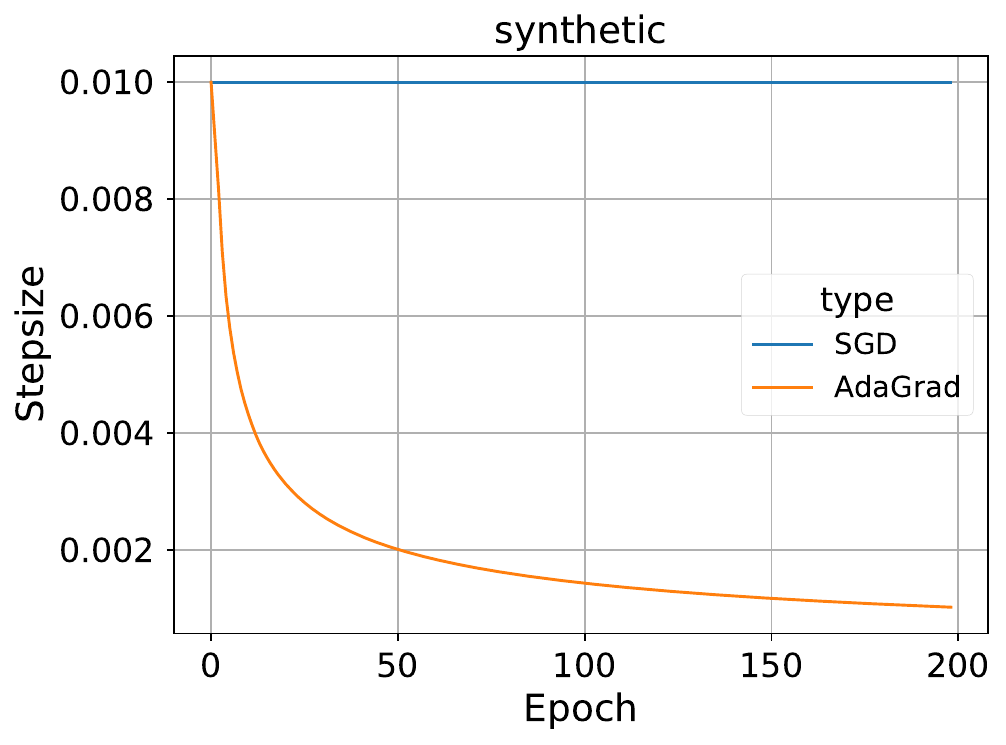}}{}
            \subf{\includegraphics[width=0.32\textwidth]{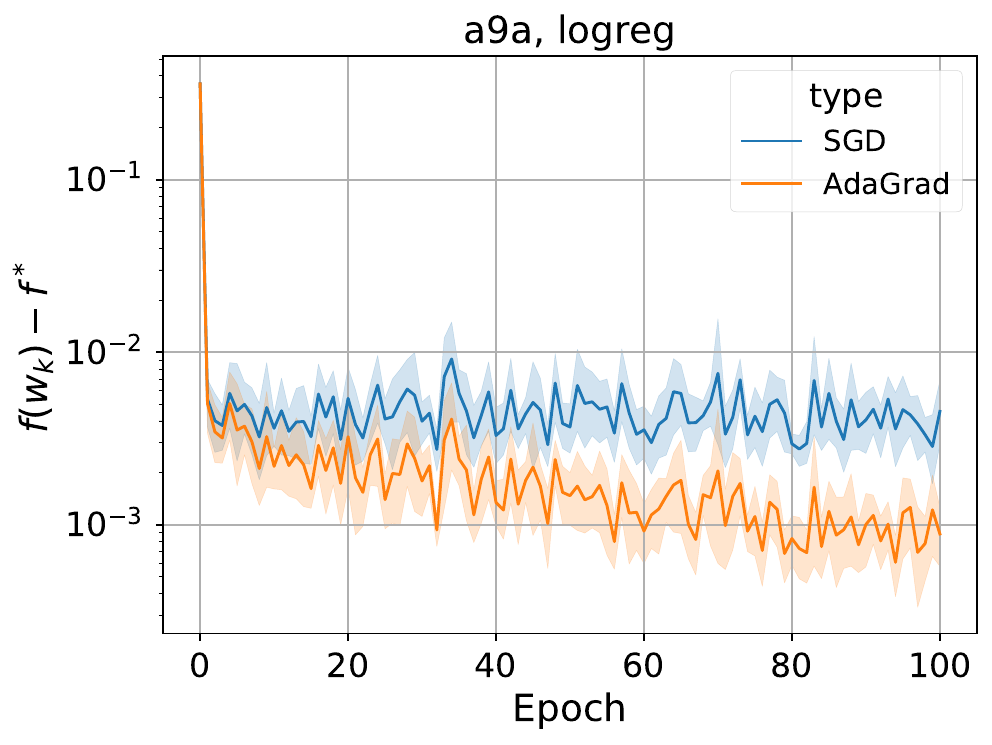}}{}  
            \subf{\includegraphics[width=0.32\textwidth]{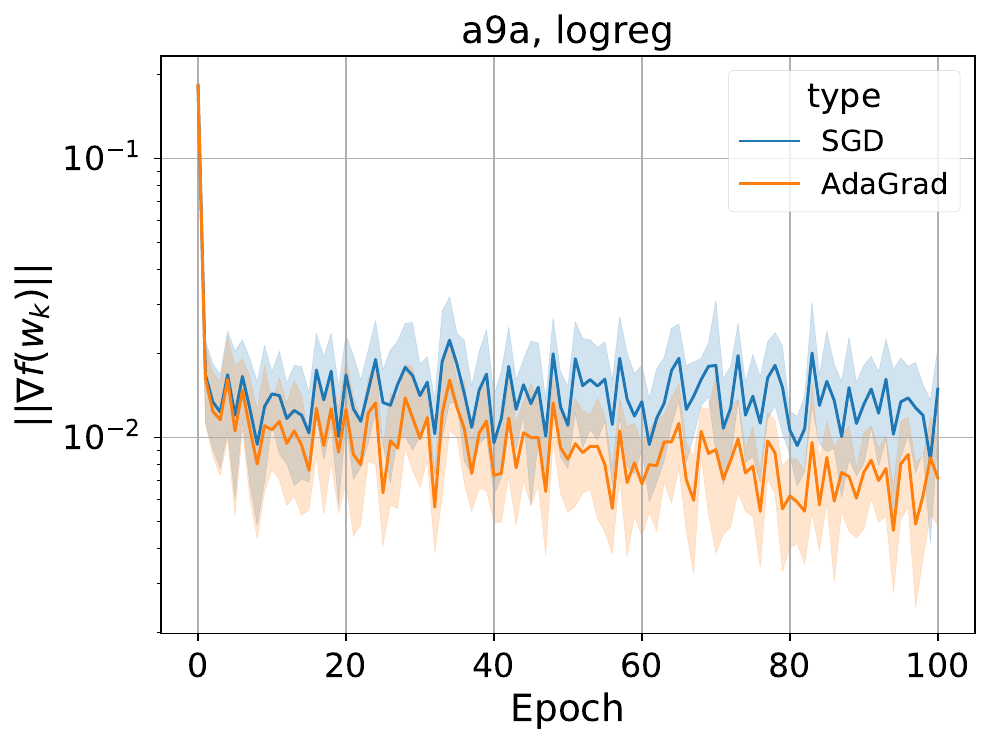}}{}  
            \subf{\includegraphics[width=0.32\textwidth]{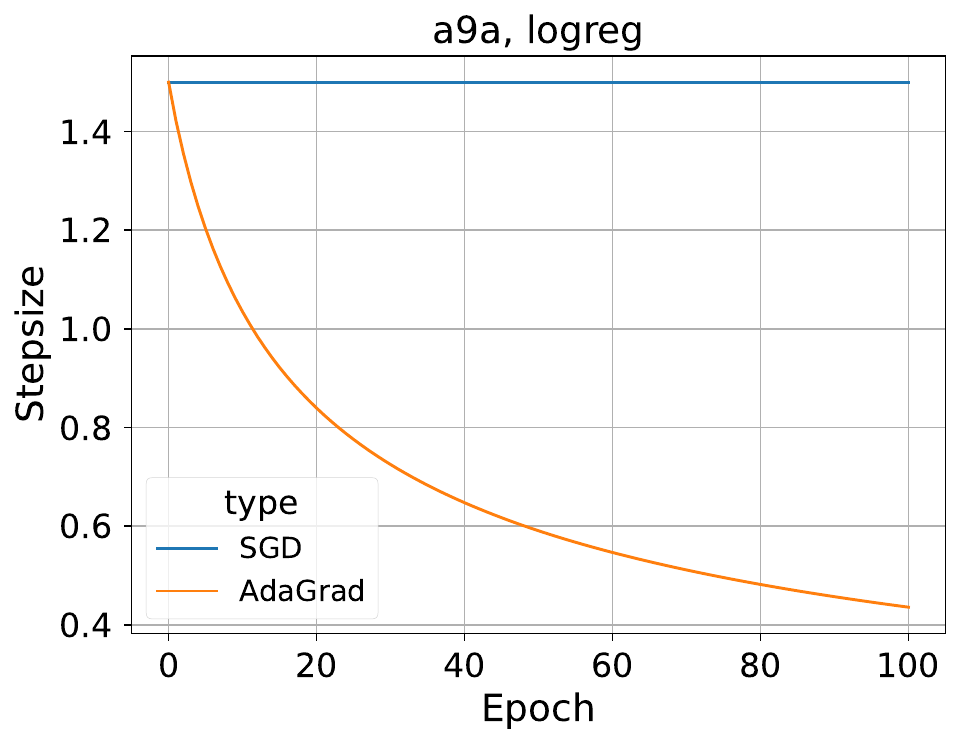}}{}  
            \subf{\includegraphics[width=0.32\textwidth]{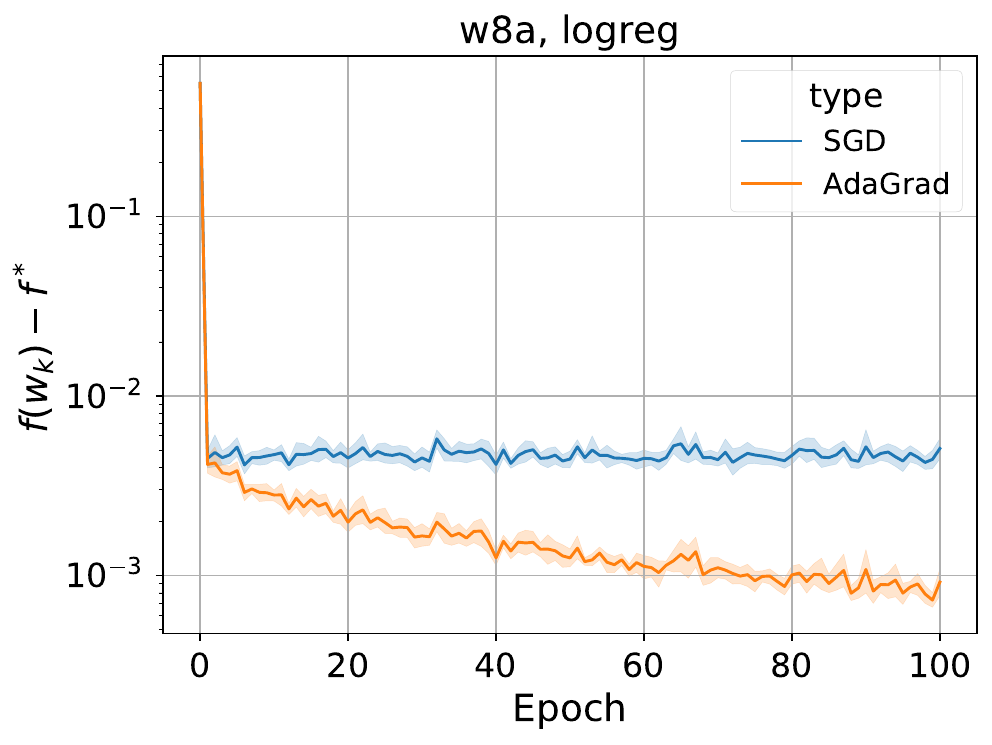}}{}  
            \subf{\includegraphics[width=0.32\textwidth]{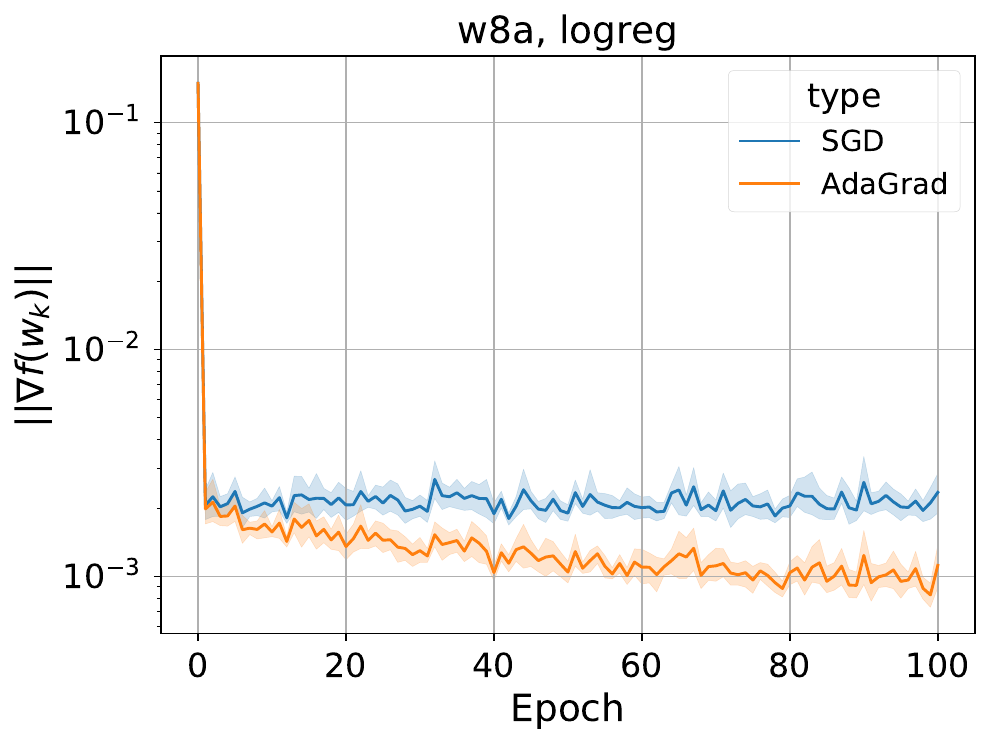}}{}  
            \subf{\includegraphics[width=0.32\textwidth]{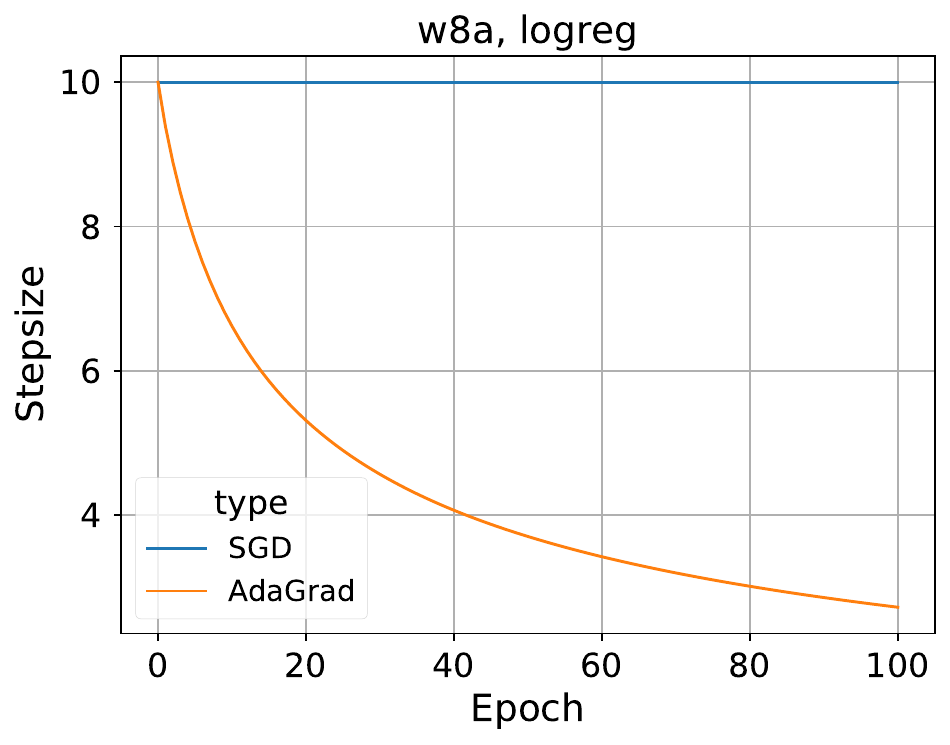}}{}
            \vskip-2pt
            \caption{\label{logreg:task3all-datasets}AdaGrad vs SGD with small $\beta$ for synthetic (upper plots), logistic regression on a9a (middle plots) and on w8a (lower plots). Here we choose little $\beta$ so that it allows the sum of gradient norms in the denominator of \eqref{eq:global step size} to change step size with time. As you can see, now step size gradually decreases, while approaching the area of solution.}
            \vskip-20pt
        \end{figure}
    \begin{figure}[H]
    \vskip-10pt
        \centering
        \subf{\includegraphics[width=0.32\textwidth]{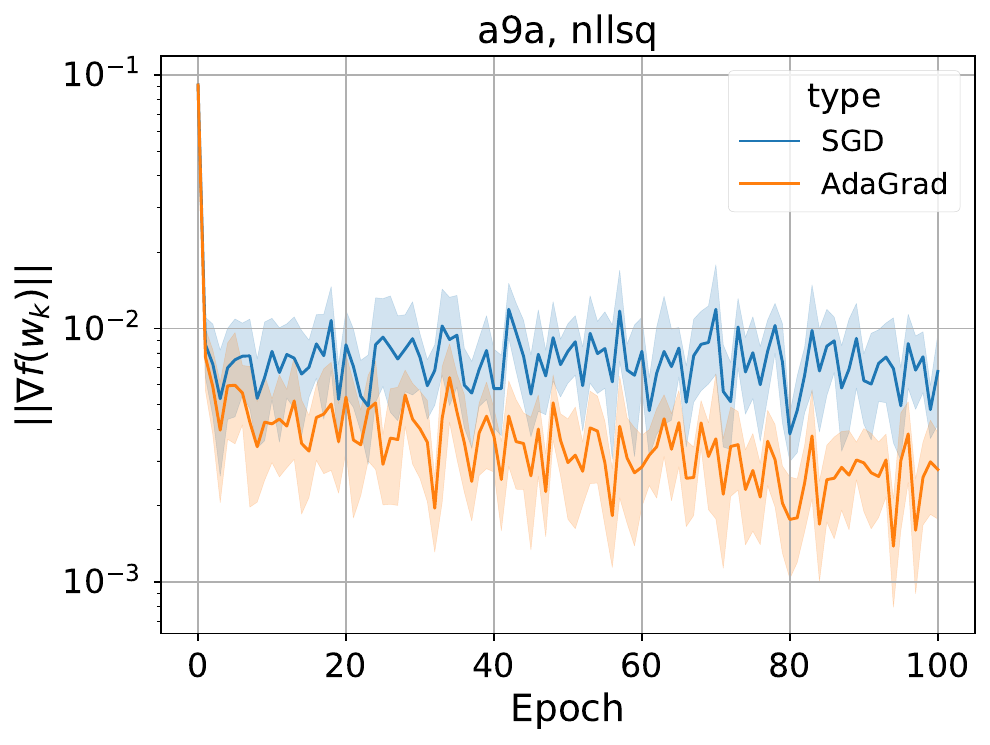}}{}  
        \subf{\includegraphics[width=0.32\textwidth]{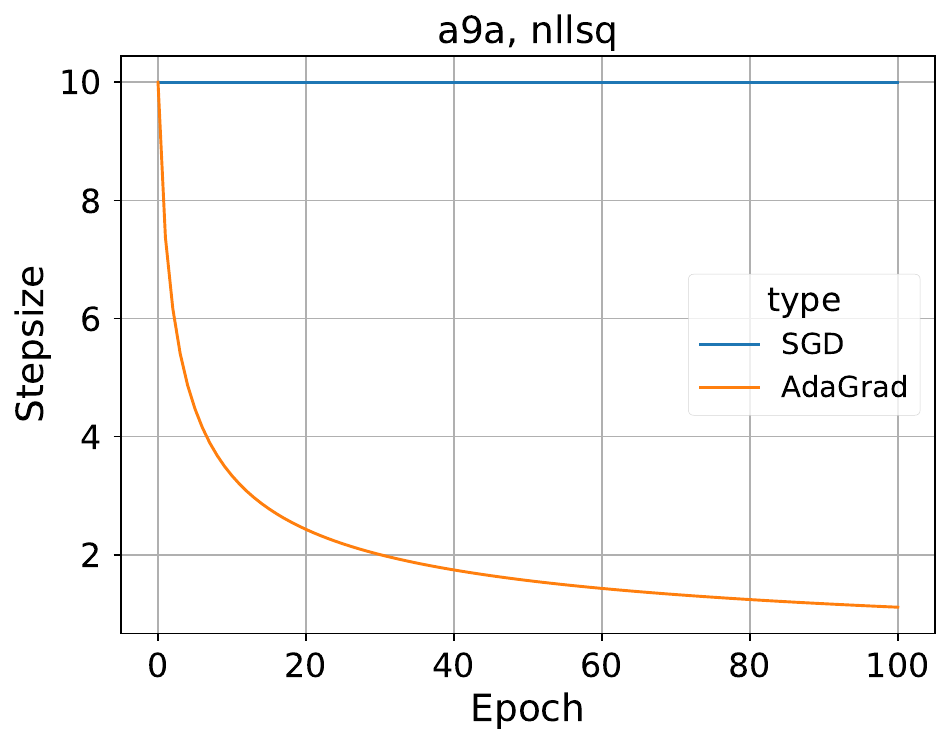}}{}\\  
        \vskip-4pt
        \subf{\includegraphics[width=0.32\textwidth]{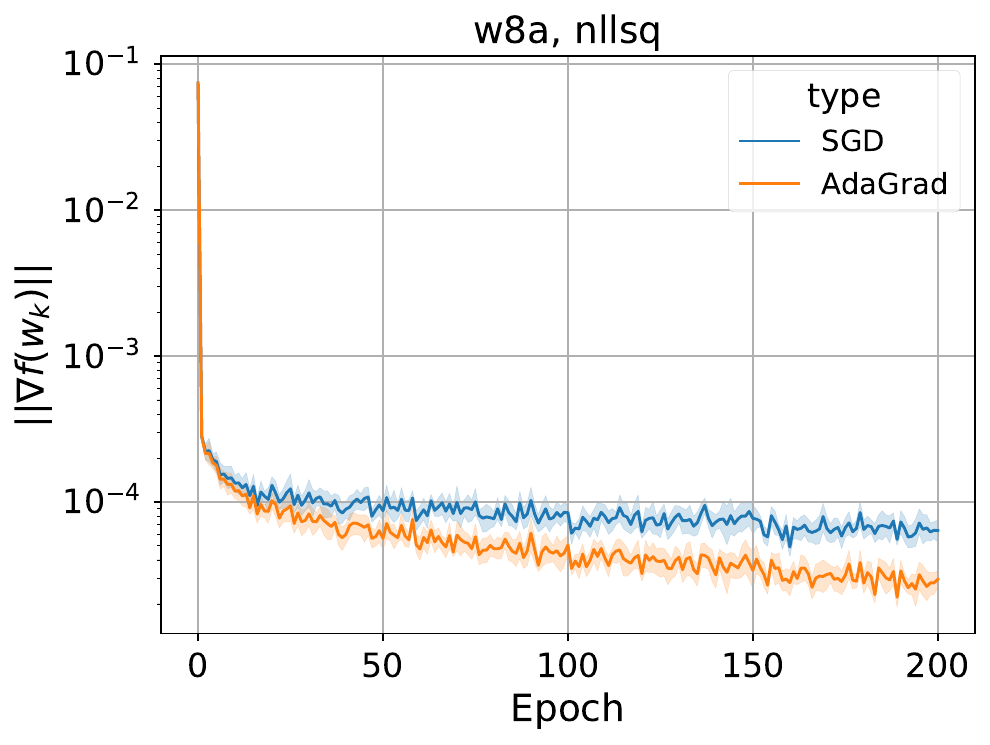}}{}  
        \subf{\includegraphics[width=0.32\textwidth]{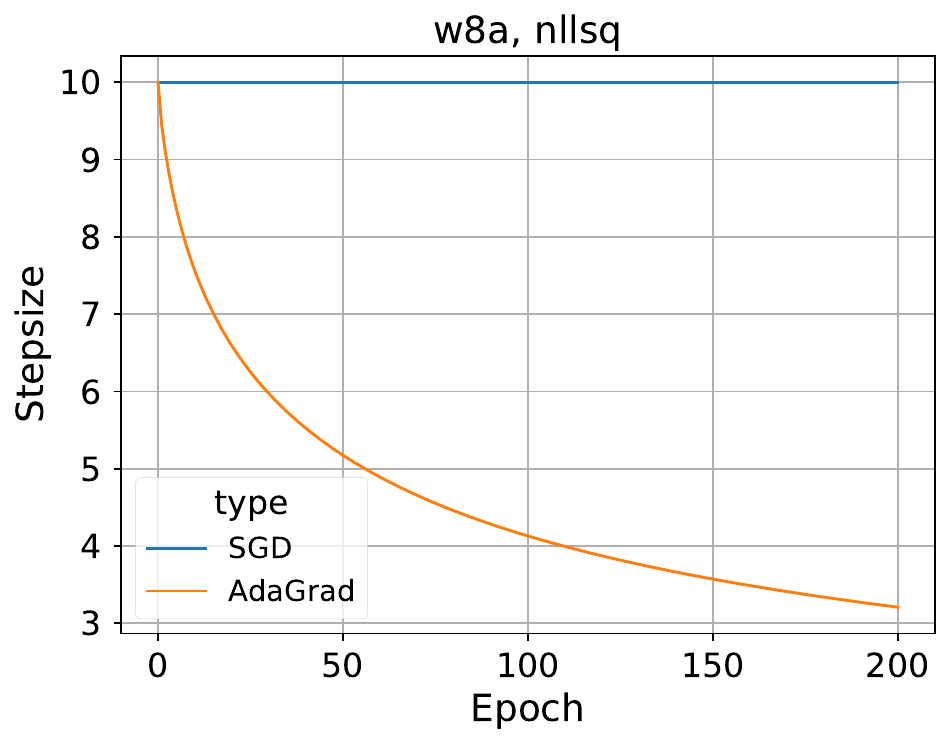}}{}       \vskip-4pt      
        \caption{\label{nllsq:task3all-datasets}AdaGrad vs SGD with small $\beta$ for NLLSQ on a9a (left two plots) and on w8a (right two plots). Here we choose little $\beta$ so that it allows the sum of gradient norms in the denominator of \eqref{eq:global step size} to change step size with time. As you can see, now step size gradually decreases, while approaching the area of solution.
        }
        \vskip-20pt
            \end{figure}
            
            \begin{figure}[H]
            \vskip-20pt
                \centering
        \subf{\includegraphics[width=0.34\textwidth]{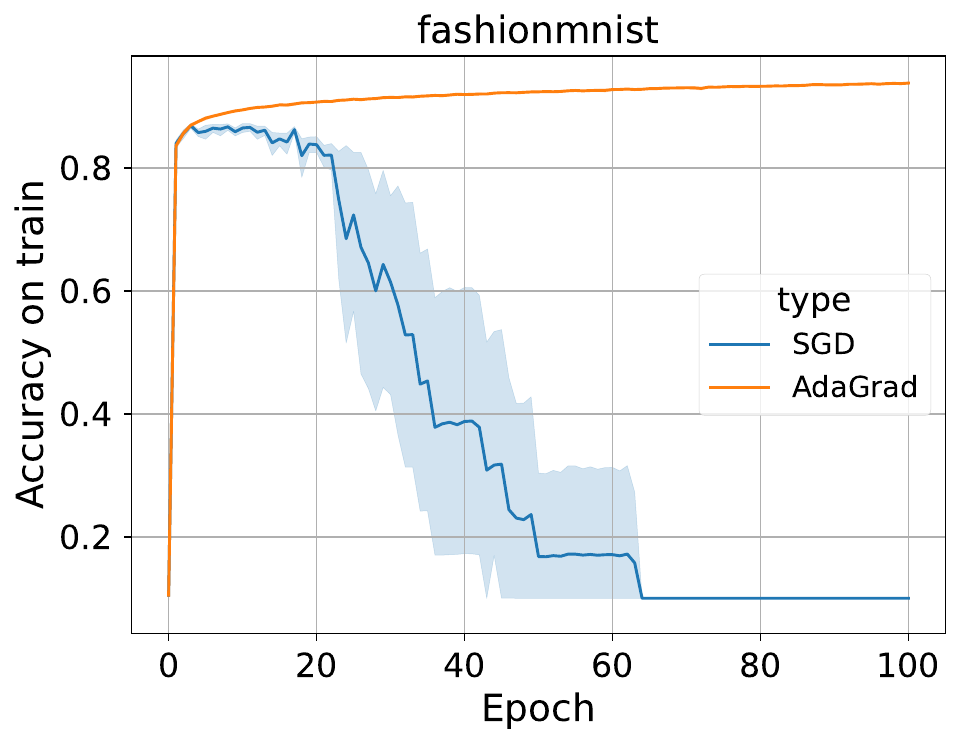}}{}
        \subf{\includegraphics[width=0.34\textwidth]{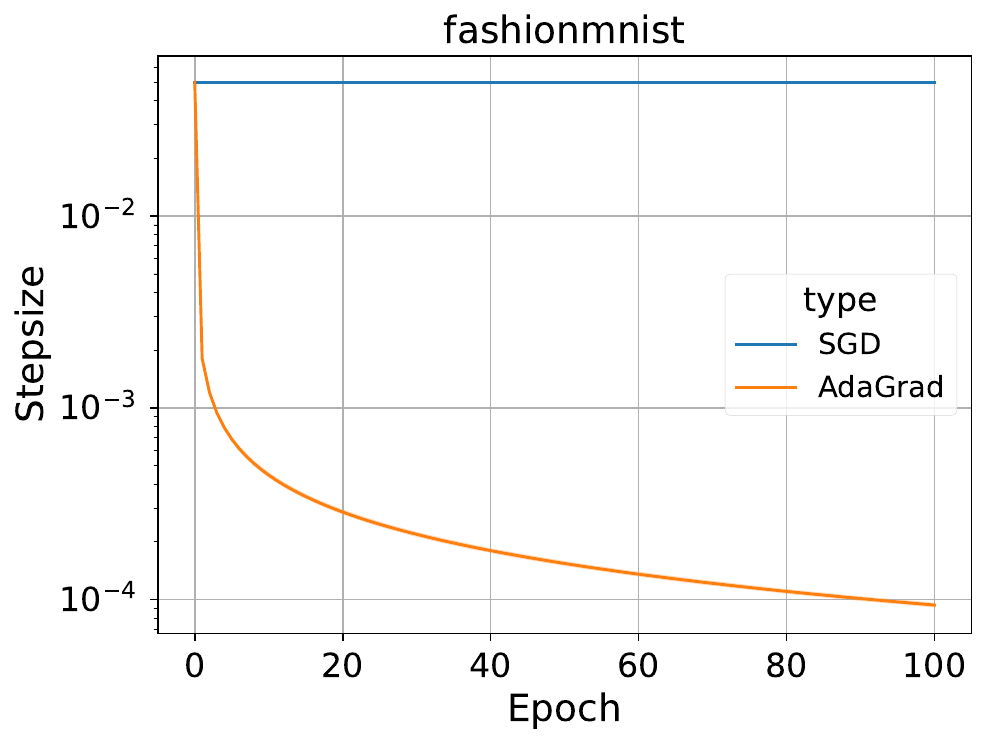}}{}
        \vskip-2pt
        \caption{\label{nn:task3}AdaGrad vs SGD with small $\beta$ for CNN on FashionMNIST. Here we choose little $\beta$ so that it allows the sum of gradient norms in the denominator of \eqref{eq:global step size} to change step size with time. As you can see, now step size gradually decreases, while approaching better accuracy.}
        \vskip-20pt
    \end{figure}
    
    \paragraph{\textbf{SGD with different batch sizes}\;}
        The second way to improve SGD is by increasing batch size
        This procedure reduces the variance of the stochastic gradient.
        However, it also makes computation of one stochastic gradient more expensive, because we have to operate with bigger matrices.
        In Figure \ref{logreg:task4all-datasets},\ref{nllsq:task4all-datasets} we show, how batch size influences the performance of the regular SGD.
        As you can see, the higher the batch size, the lower the variance of the gradient approximation, which leads to a better convergence rate. For the experiment on neural network, Figure \ref{nn:task4} shows the ideal batch size might lie between 64 and 1024. 
        \begin{figure}[H]
            \centering
            \subf{\includegraphics[width=0.3\textwidth]{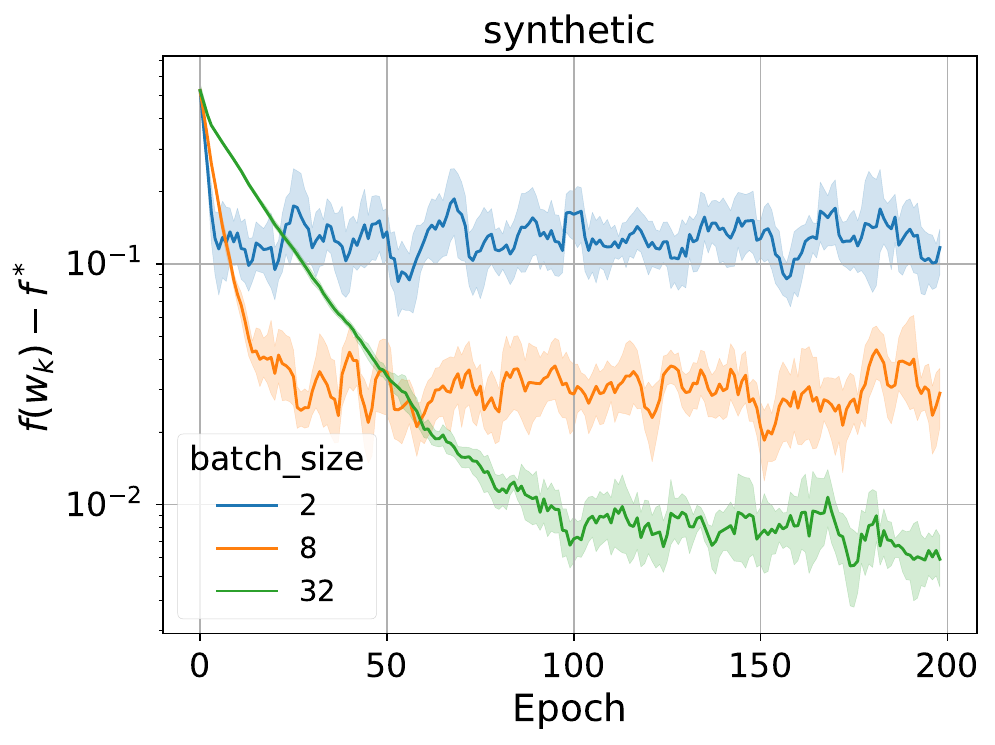}}{}
            \subf{\includegraphics[width=0.3\textwidth]{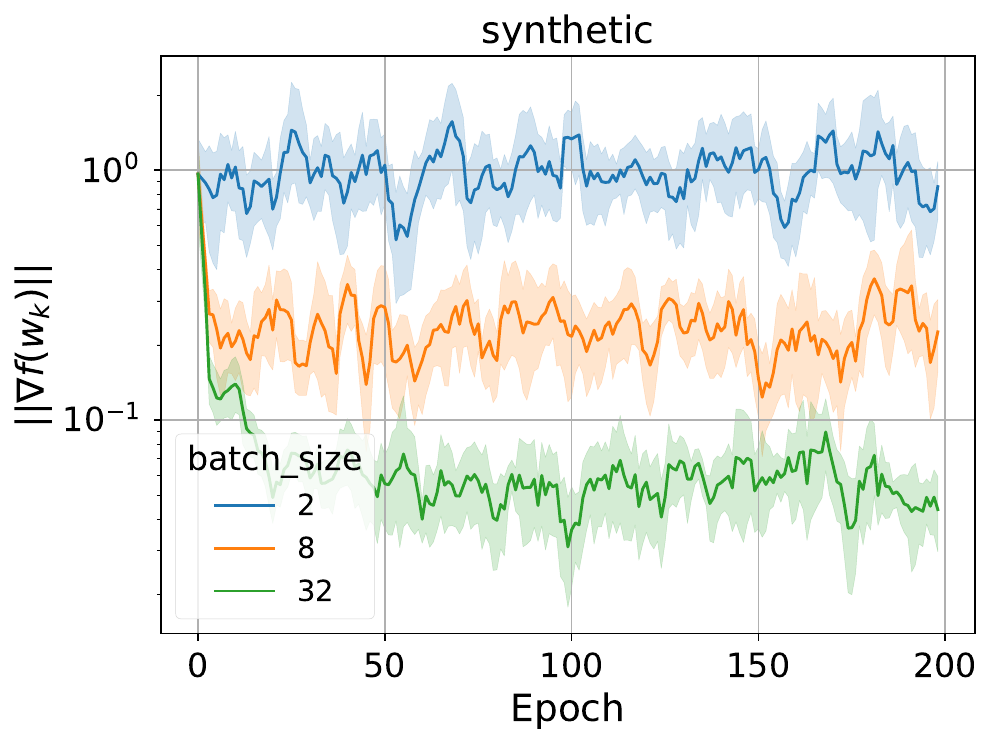}}{}\\
            \subf{\includegraphics[width=0.3\textwidth]{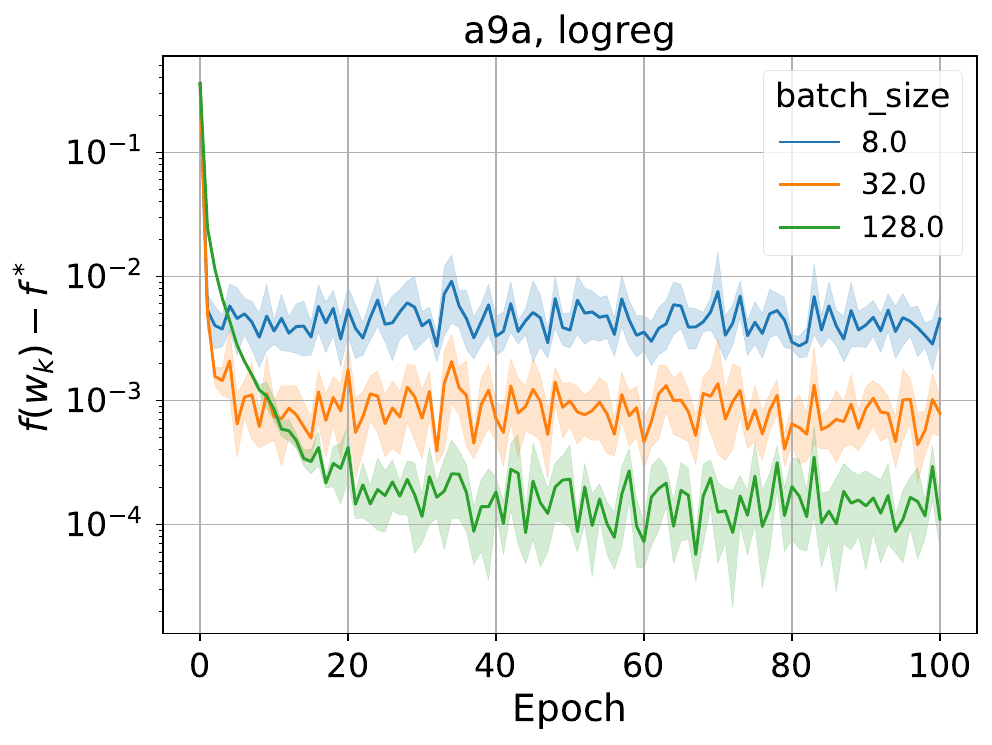}}{}
            \subf{\includegraphics[width=0.3\textwidth]{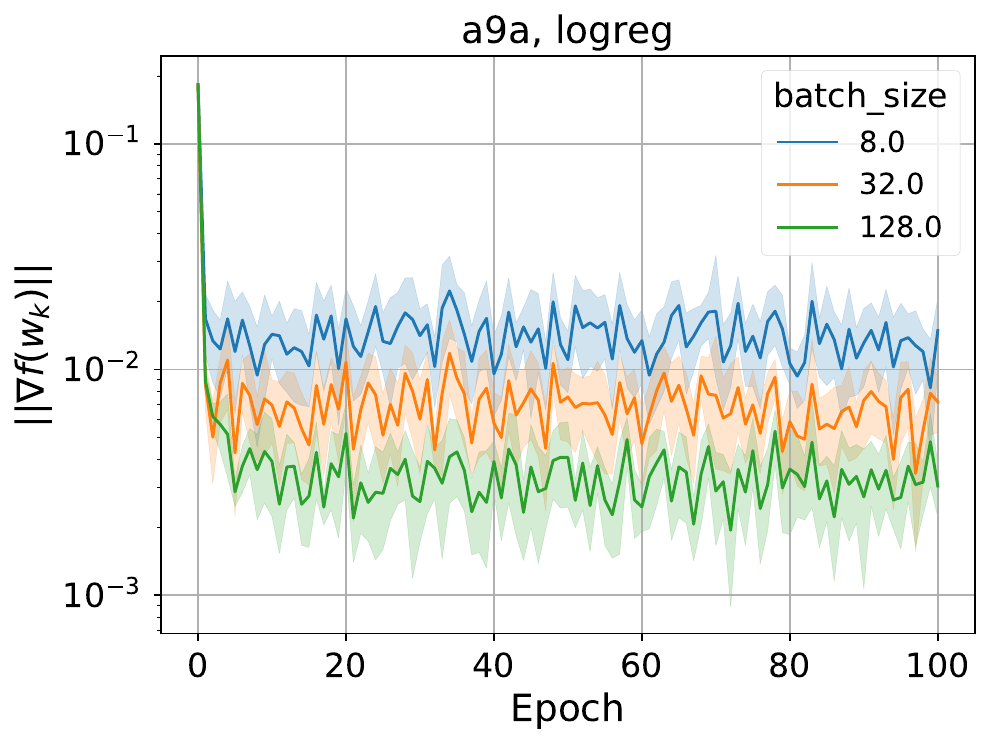}}{} \\
            \subf{\includegraphics[width=0.3\textwidth]{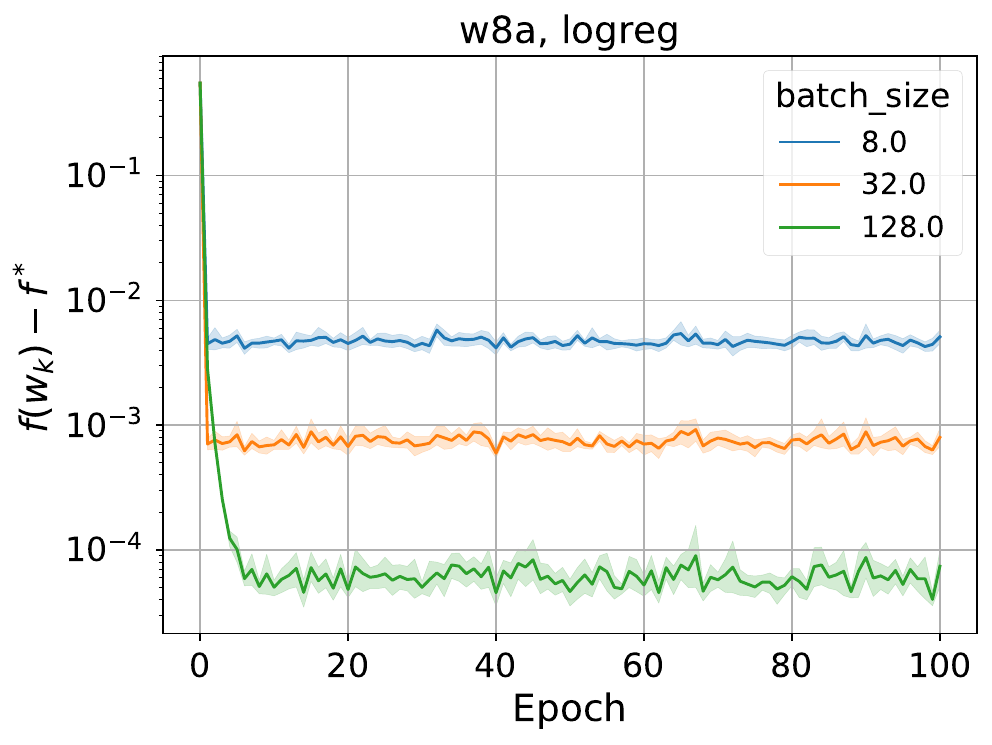}}{}            
            \subf{\includegraphics[width=0.3\textwidth]{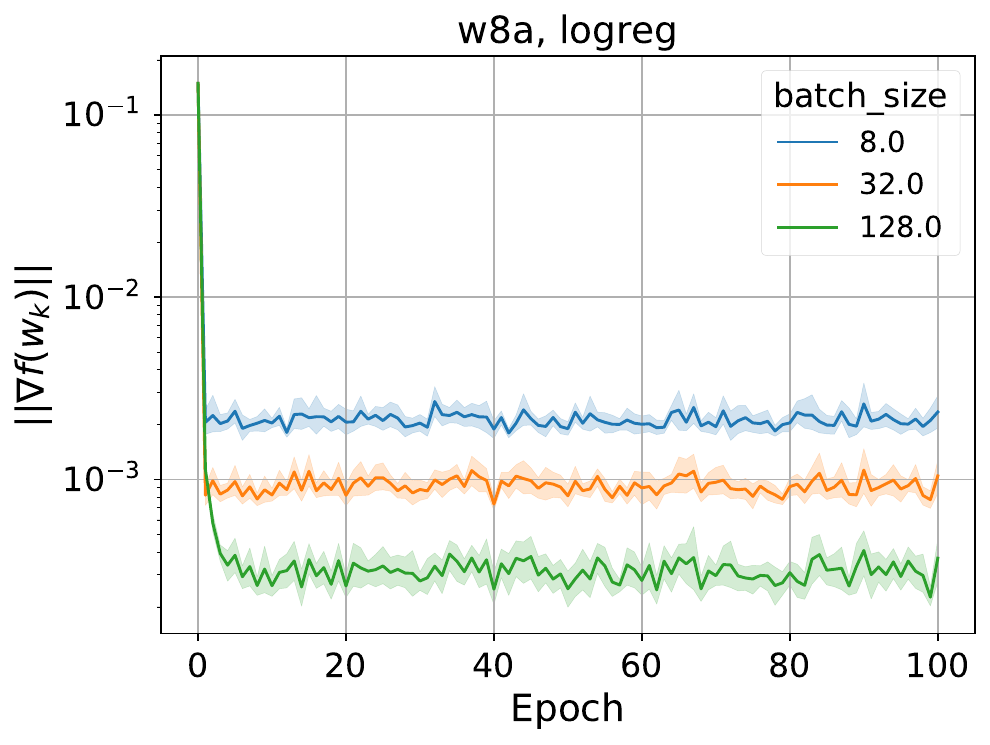}}{}
            \caption{\label{logreg:task4all-datasets}Results for SGD with different batch sizes for synthetic (upper two plots), logistic regression on a9a (middle two plots) and on w8a (lower two plots) datasets.
            Larger batch size allows the algorithm to converge to better confusion region by reducing stochastic gradient approximation variance.}
            \subf{\includegraphics[width=0.3\textwidth]{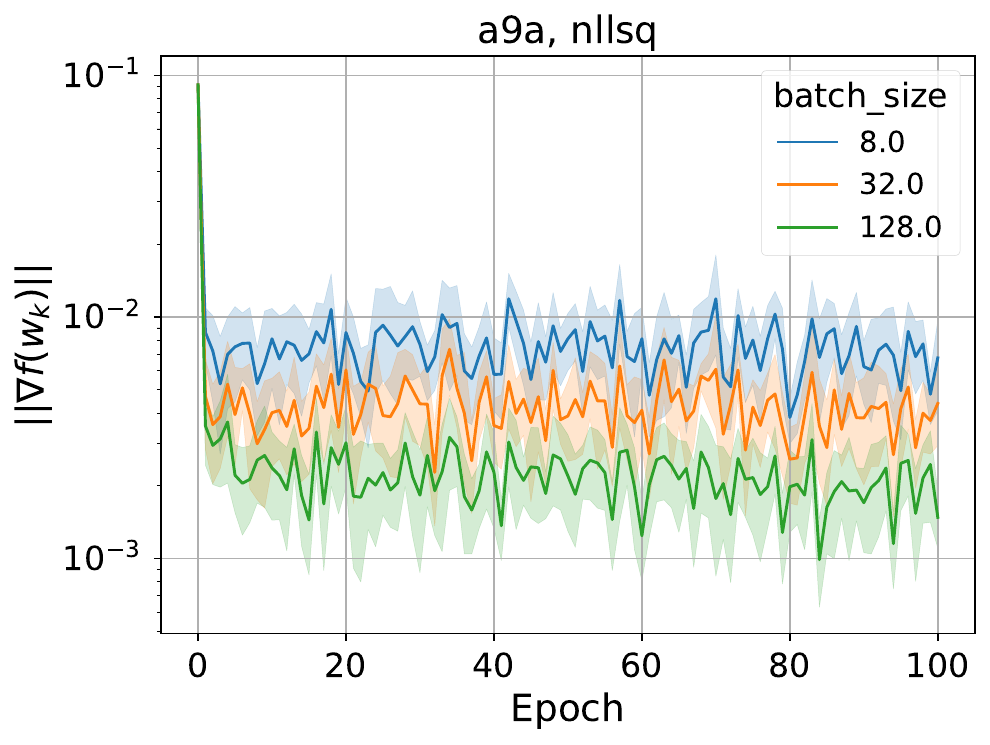}}{}
            \subf{\includegraphics[width=0.3\textwidth]{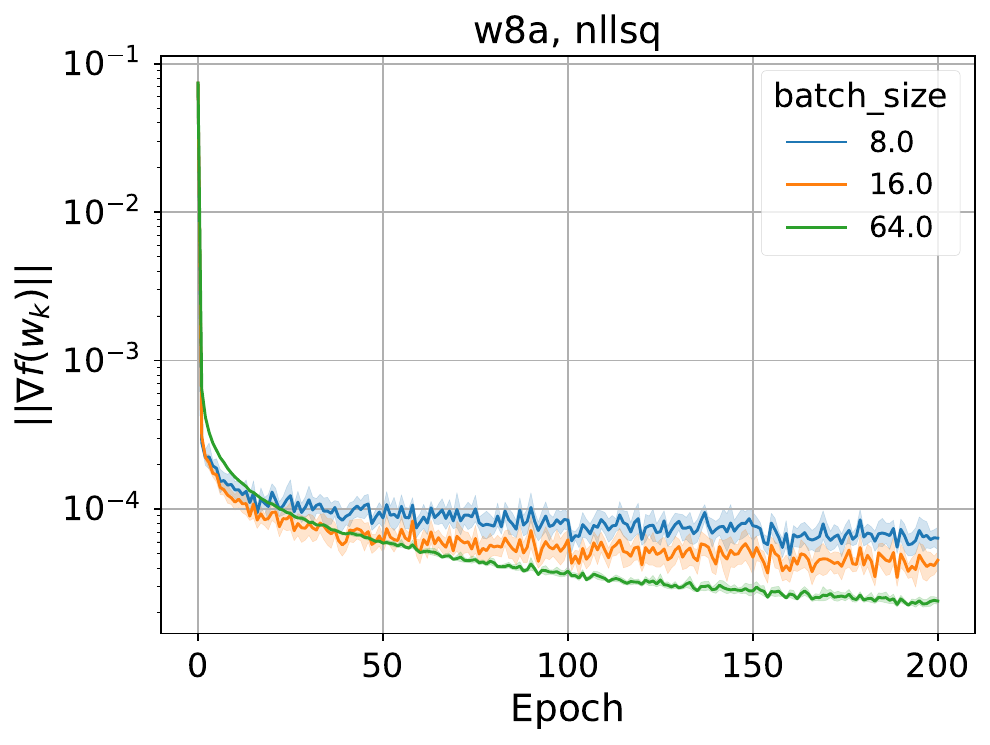}}{}
            \caption{\label{nllsq:task4all-datasets}Results for SGD with different batch size for NLLSQ on a9a (left plot) and w8a (right plot).
            A larger batch size allows the algorithm to converge to a better confusion region by reducing stochastic gradient approximation variance.}
            \subf{\includegraphics[width=0.3\textwidth]{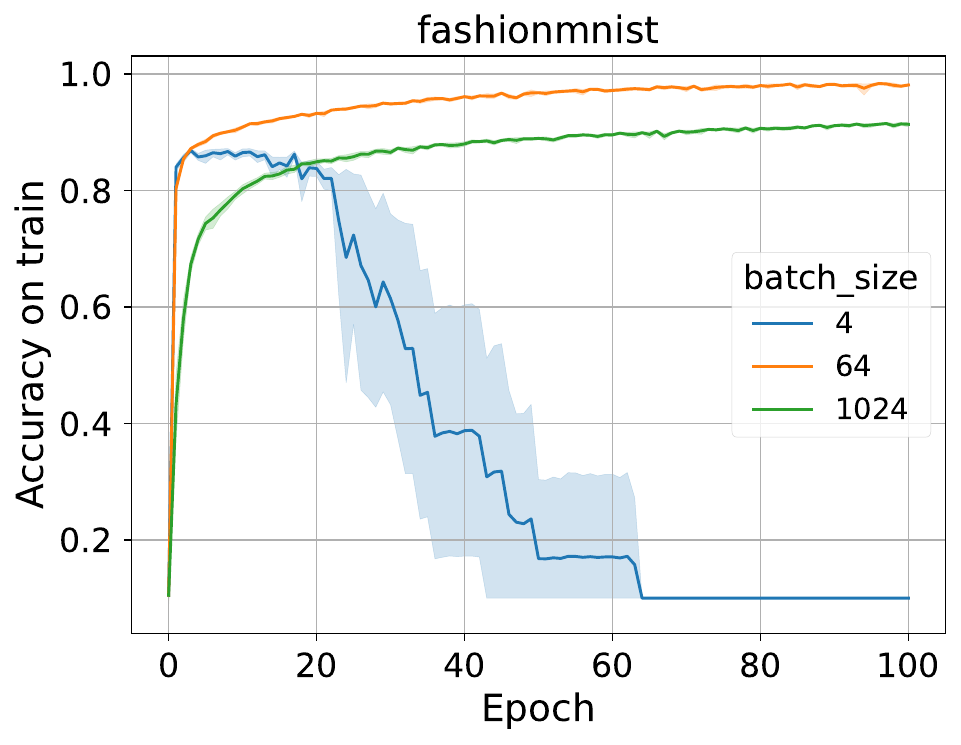}}{}
            \caption{\label{nn:task4}Accuracy results for SGD with different batch sizes for CNN on FashionMNIST.
            A larger batch size allows the algorithm to converge to a better confusion region by reducing stochastic gradient approximation variance.}
        \end{figure}
    
    \paragraph{\textbf{SGD vs SGD + tests}\;}
        Now we can make the procedure of batch size adjustment adaptive.
        For this, we can use approximated batch size tests \eqref{eq:approximated inner product test}, \eqref{eq:approximated orthogonality test}, that we have mentioned earlier.
        Despite the fact that in our theoretical results, we used norm test \eqref{eq:norm test}, in our experiments we adjusted batch size with approximated inner product test \eqref{eq:approximated inner product test} and \eqref{eq:orthogonality test}. 
        We fixed hyperparameters $\theta = 1.5$ and $\nu = 7$.
        The results are presented in Figure \ref{logreg:task5all-datasets},\ref{nllsq:task5all-datasets}.
        As you can see, now SGD has a much lower variance and it converges faster. The same result holds for the neural network in Figure \ref{nn:task5}.
         \begin{figure}[H]
         \vskip-20pt
            \centering
            \subf{\includegraphics[width=0.32\textwidth]{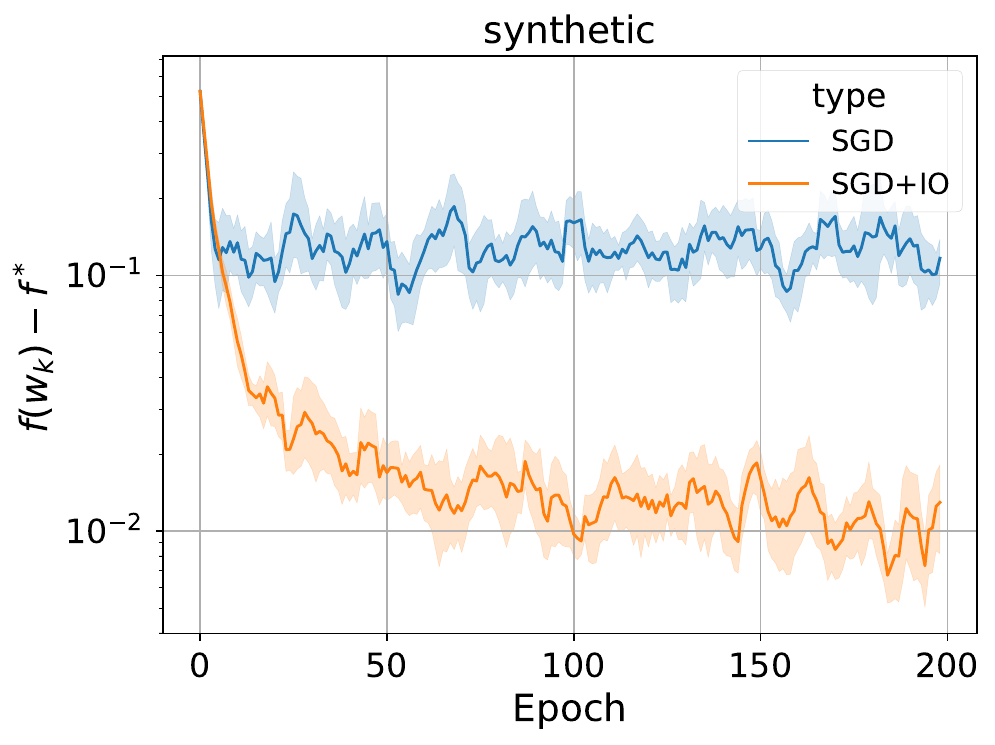}}{}  
            \subf{\includegraphics[width=0.32\textwidth]{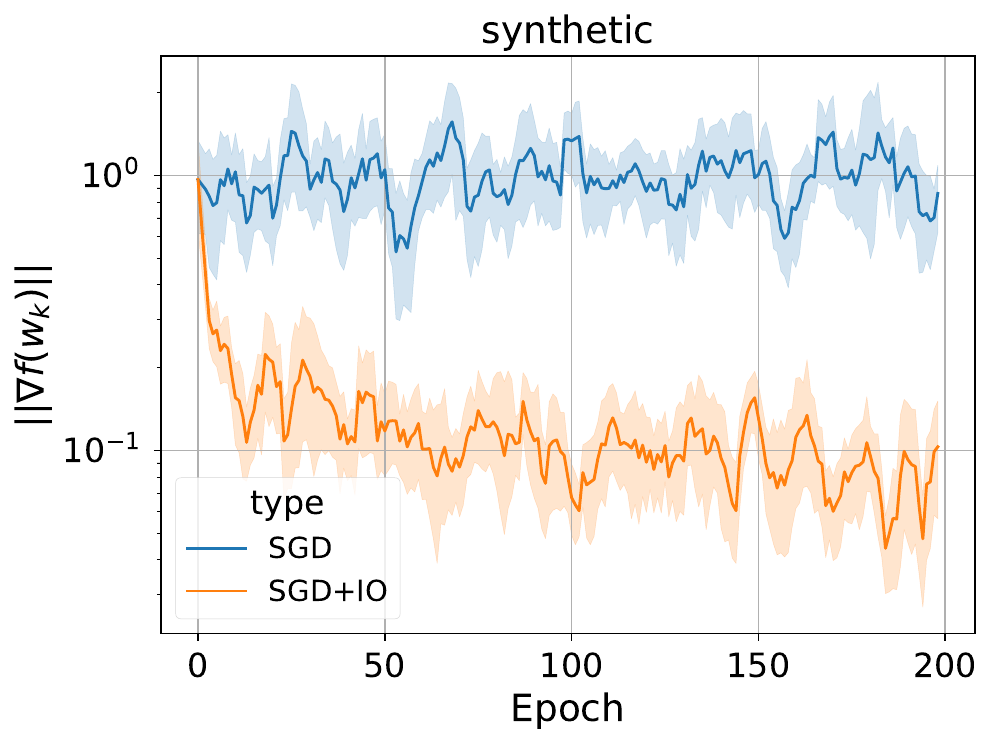}}{}  
            \subf{\includegraphics[width=0.32\textwidth]{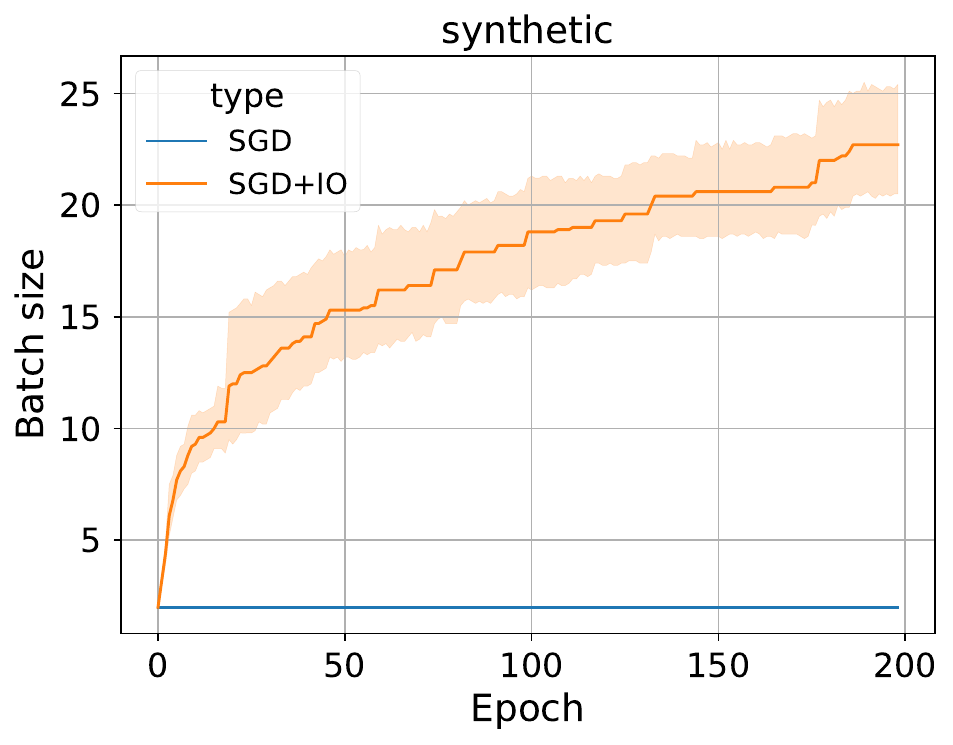}}{}
            \vskip-5pt
            \subf{\includegraphics[width=0.32\textwidth]{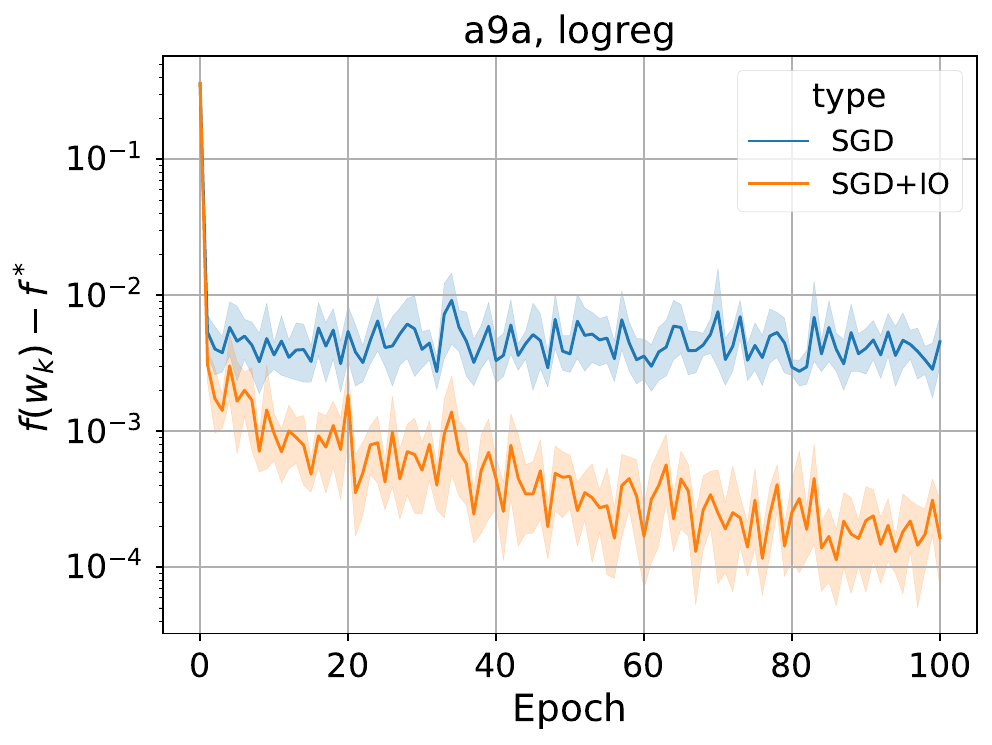}}{}  
            \subf{\includegraphics[width=0.32\textwidth]{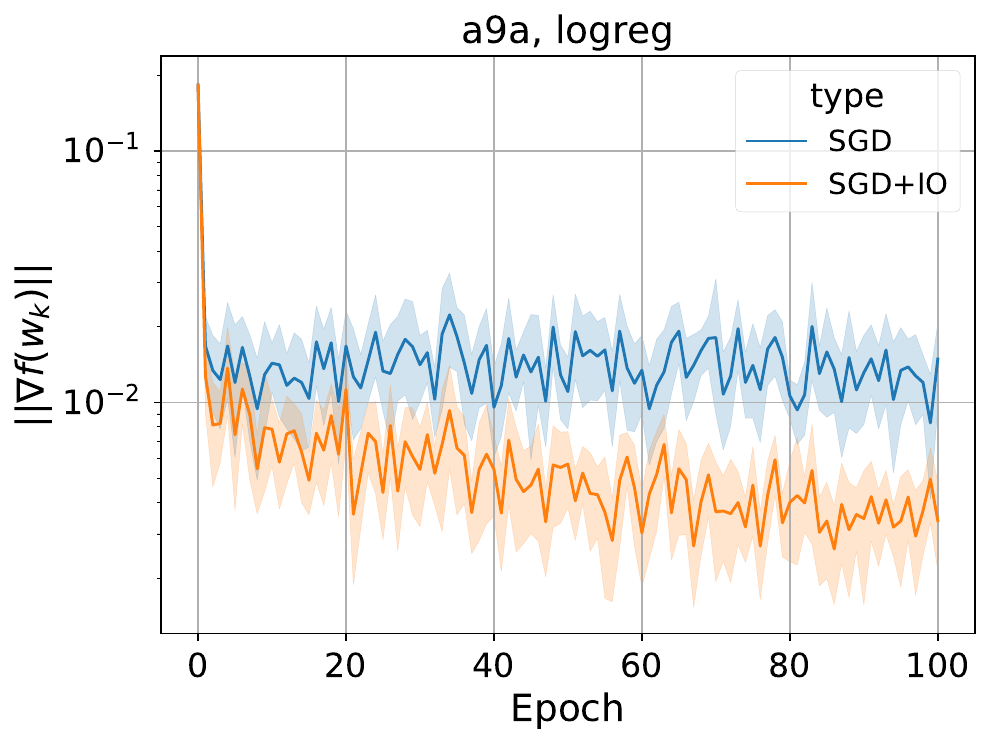}}{}  
            \subf{\includegraphics[width=0.32\textwidth]{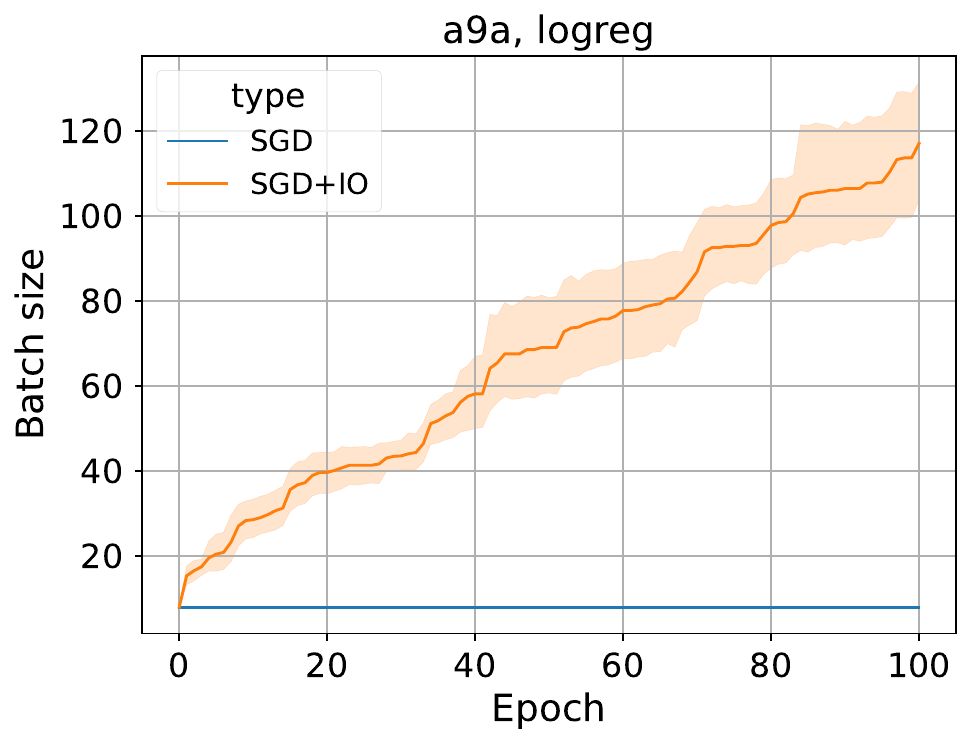}}{}  
            \vskip-5pt
            \subf{\includegraphics[width=0.32\textwidth]{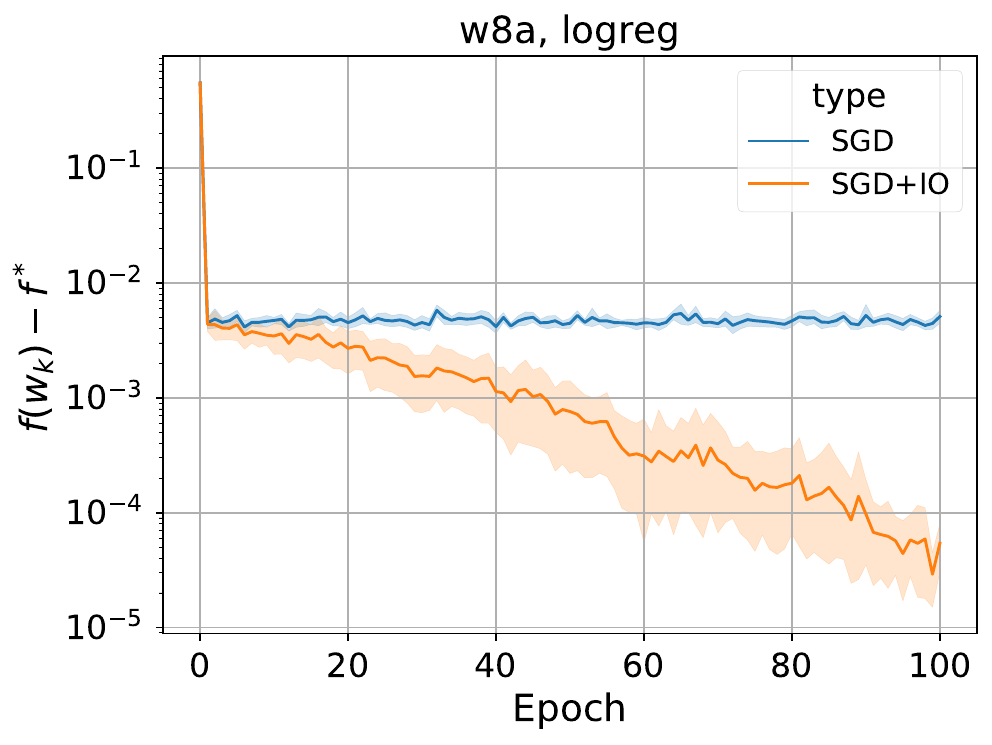}}{}  
            \subf{\includegraphics[width=0.32\textwidth]{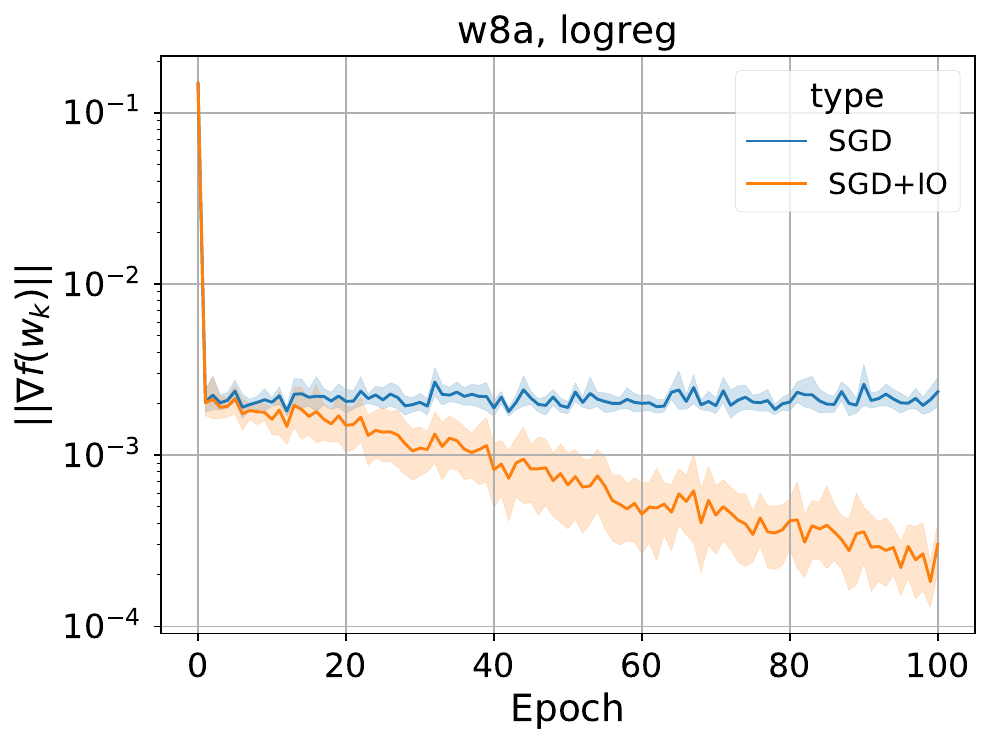}}{}  
            \subf{\includegraphics[width=0.32\textwidth]{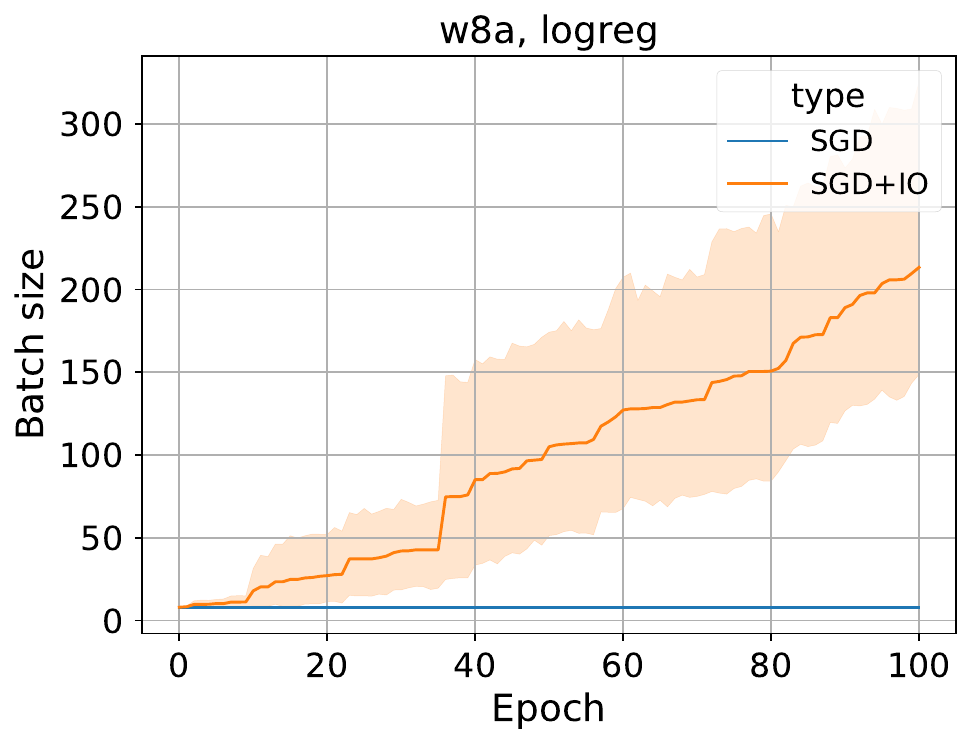}}{}
            \vskip-3pt
            \caption{\label{logreg:task5all-datasets}SGD vs SGD + Inner test + Ortho test (IO) for synthetic (upper plots), logistic regression on a9a (middle plots) and on w8a (lower plots). Adaptivity in batch size makes the algorithm adjust batch size automatically as long as it converges to a minimum. Despite the fact that batch size grows up to a very small part of the whole dataset, it allows us to significantly improve the results.}
            \vskip-20pt
        \end{figure}
        \begin{figure}[H]
        \vskip-10pt
            \centering
            \subf{\includegraphics[width=0.32\textwidth]{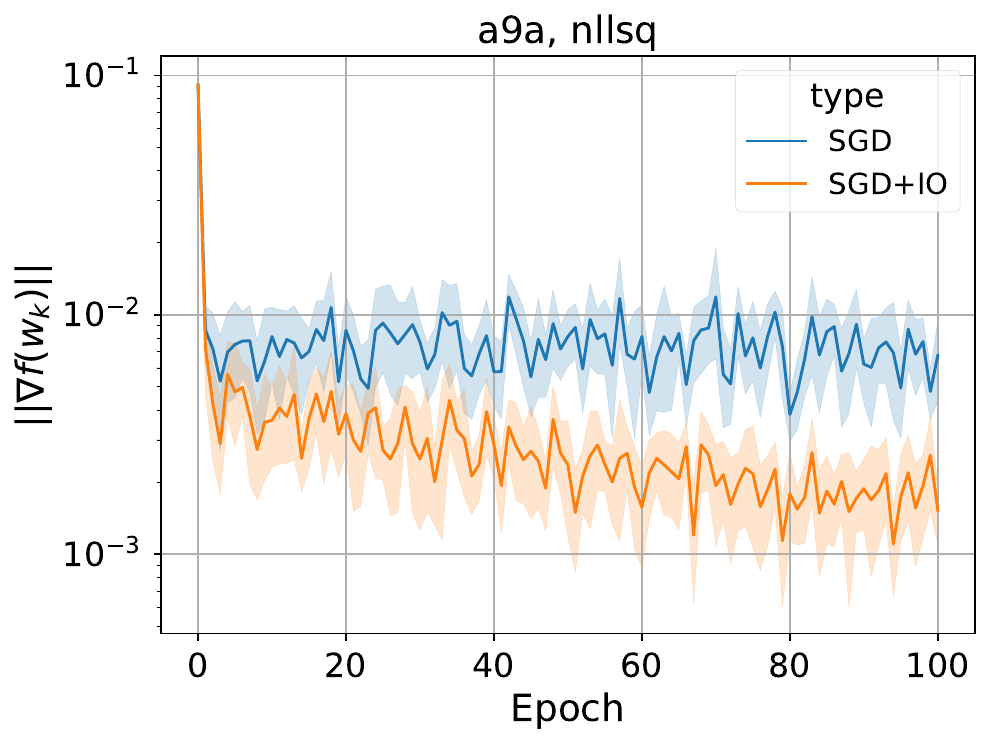}}{}  
            \subf{\includegraphics[width=0.32\textwidth]{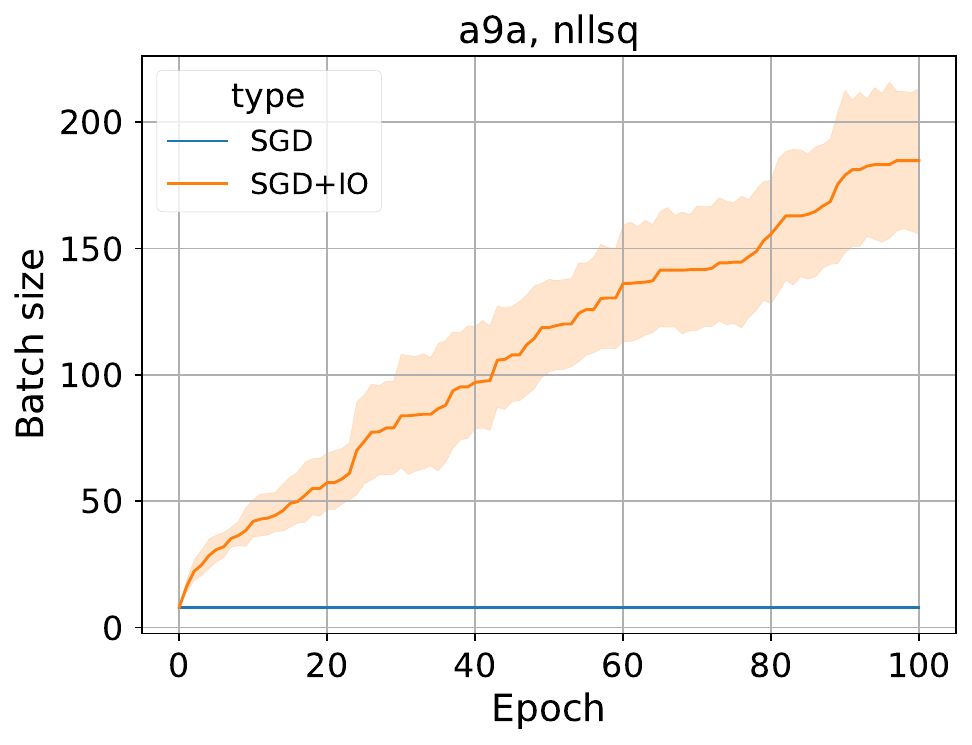}}{}   \\
            \vskip-5pt
            \subf{\includegraphics[width=0.32\textwidth]{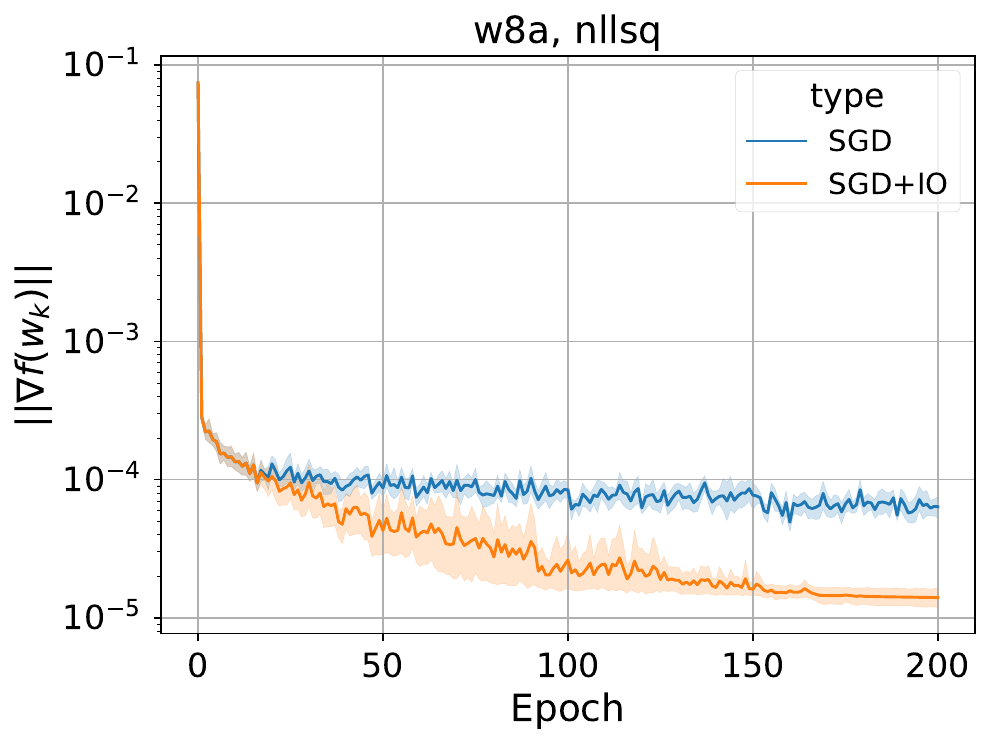}}{}  
            \subf{\includegraphics[width=0.32\textwidth]{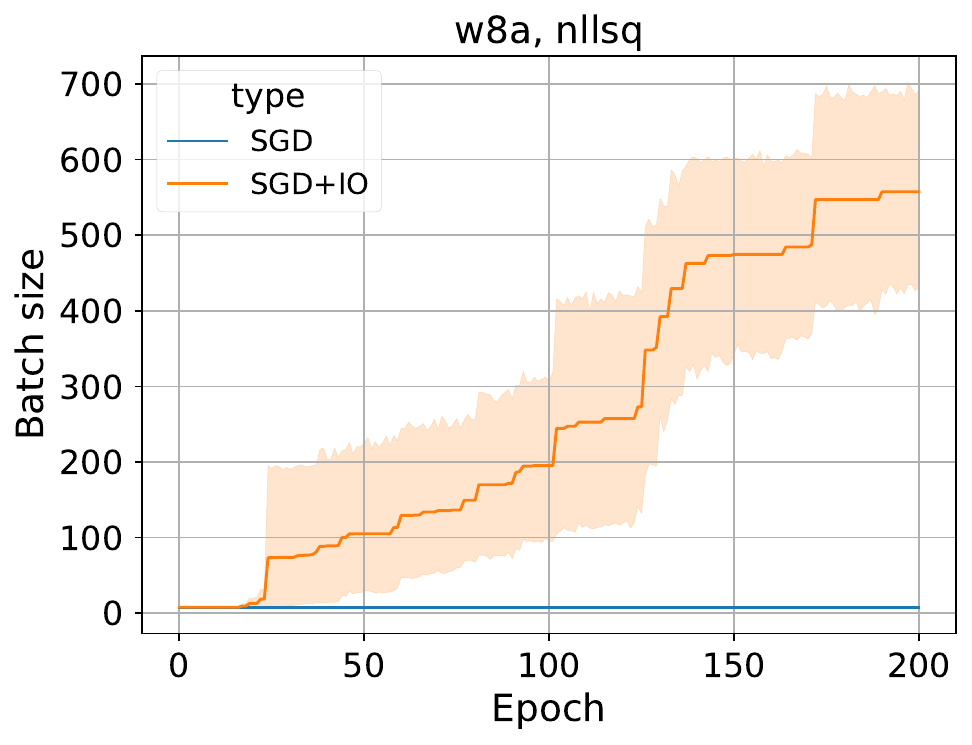}}{}
            \vskip-3pt
            \caption{\label{nllsq:task5all-datasets}SGD vs SGD + Inner test + Ortho test (IO) for NLLSQ on a9a (left two plots) and on w8a (right two plots). Adaptivity in batch size makes the algorithm adjust batch size automatically as long as it converges to a minimum. Despite the fact that batch size grows up to a very small part of the whole dataset, it allows us to significantly improve the results.}
            \vskip-20pt
            \end{figure}
        \begin{figure}[H]
            \vskip-10pt
            \centering
            \subf{\includegraphics[width=0.34\textwidth]{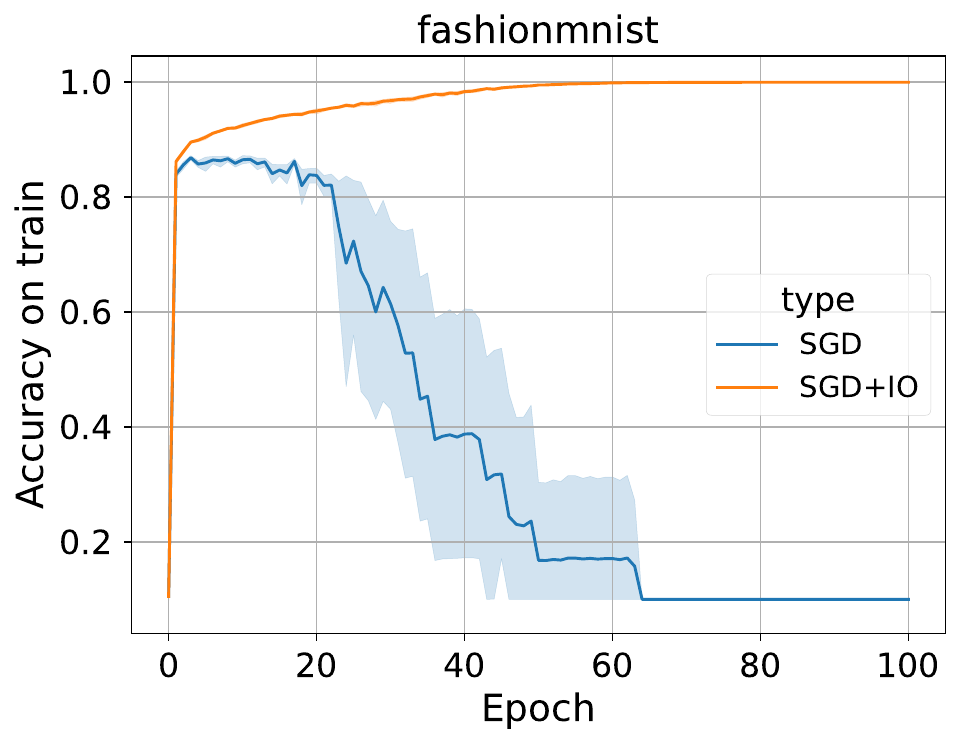}}{}
            \subf{\includegraphics[width=0.34\textwidth]{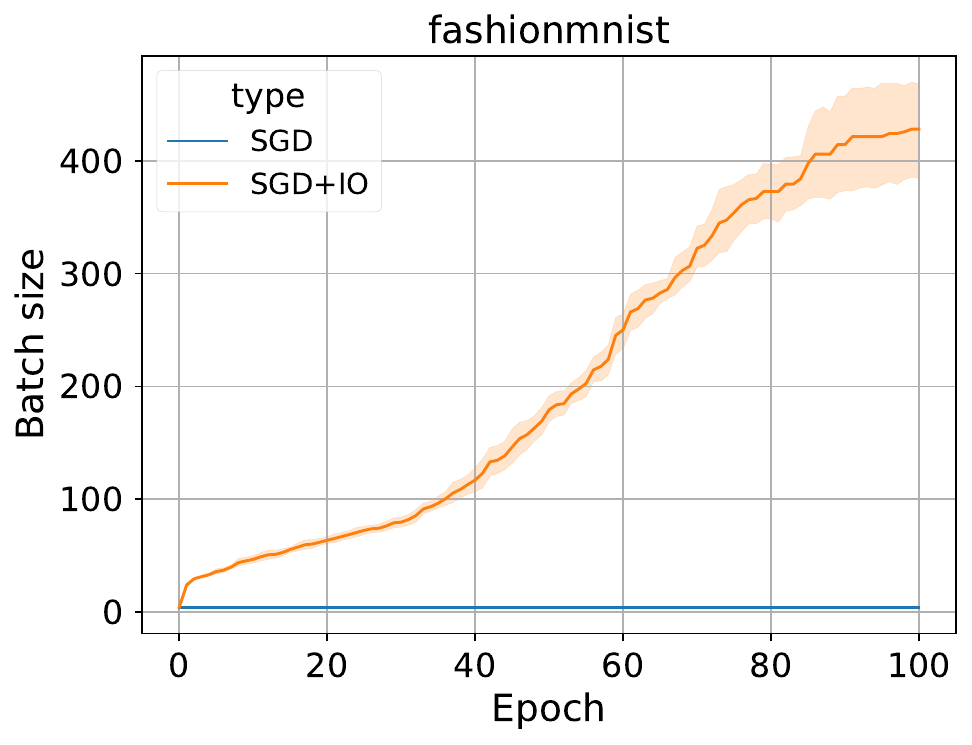  }}{}
            \vskip-3pt
            \caption{\label{nn:task5}SGD vs SGD + Inner test + Ortho test (IO) for CNN on FashionMNIST. Adaptivity in batch size makes the algorithm adjust batch size automatically as long as it converges to a minimum. Despite the fact that batch size grows up to a very small part of the whole dataset, it allows us to significantly improve the results.}
            \vskip-20pt
        \end{figure}
        
    \paragraph{\textbf{SGD vs SGD + tests vs AdaBatchGrad}\;}
        Finally, we can get the best of the two worlds: use both adaptive step size from AdaGrad \eqref{eq:global step size} and approximated tests for adaptive batch size \eqref{eq:approximated inner product test}, \eqref{eq:approximated orthogonality test}. 
        The results of Algorithm \ref{alg:adabatchgrad}, are presented in Figure \ref{logreg:task6all-datasets},\ref{nllsq:task6all-datasets}.
        Here for both AdaGrad and Algorithm \ref{alg:adabatchgrad} we used $\beta = 5 \cdot 10^4$ so that this value dominates the sum of the norms of the gradients only during the first epochs.
        For batch size tests we used the same hyperparameters as before.
        You can see, that during the first epochs Algorithm \ref{alg:adabatchgrad} performs almost the same way as AdaGrad.
        But as far as its batch size grows, the variance of the stochastic gradient reduces, and Algorithm \ref{alg:adabatchgrad} converges better.
        And we need to point out, that batch size grows only from 2 to 100 after 50 epochs, which is far from the size of the full dataset. On the neural network, we can see AdaGrad and AdaBatchGrad performed equally well in Figure \ref{nn:task6}. But AdaBatchGrad increases the batch size adaptly, which gives more freedom in choosing parameters.

    \begin{figure}[htbp]
            \centering
            \subf{\includegraphics[width=0.24\textwidth]{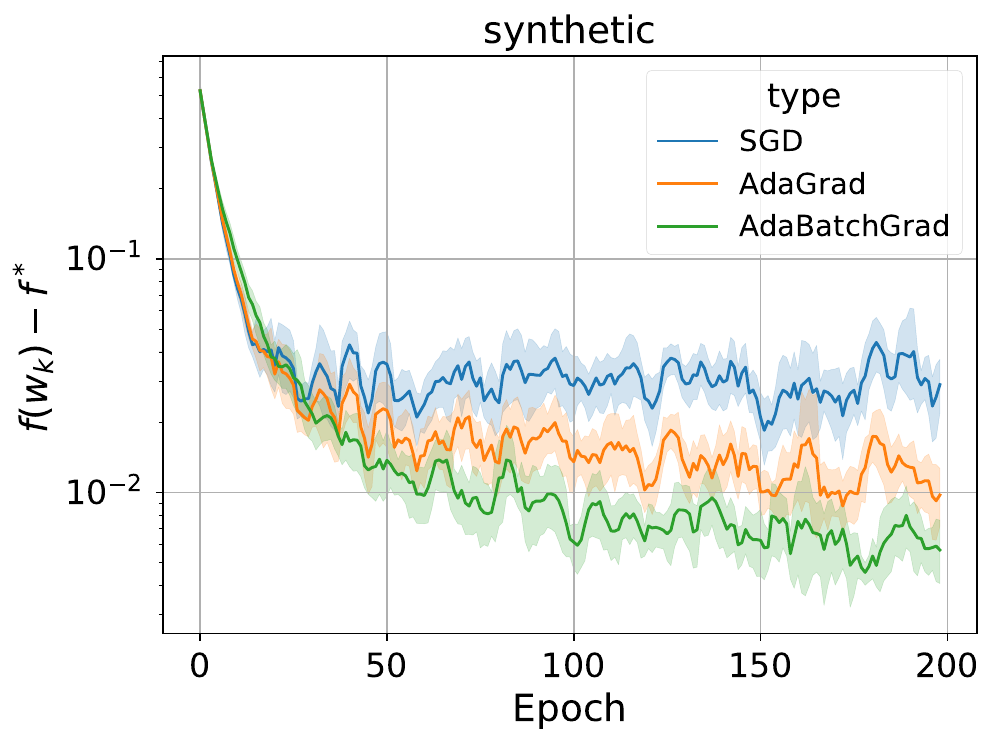}}{}  
            \subf{\includegraphics[width=0.24\textwidth]{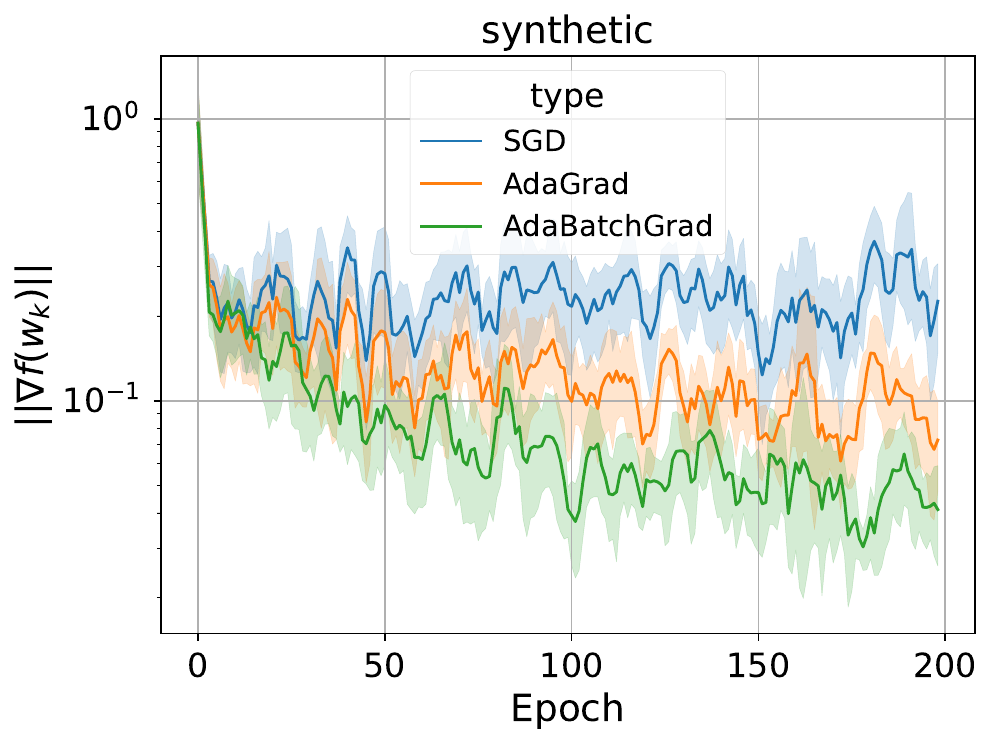}}{}  
            \subf{\includegraphics[width=0.24\textwidth]{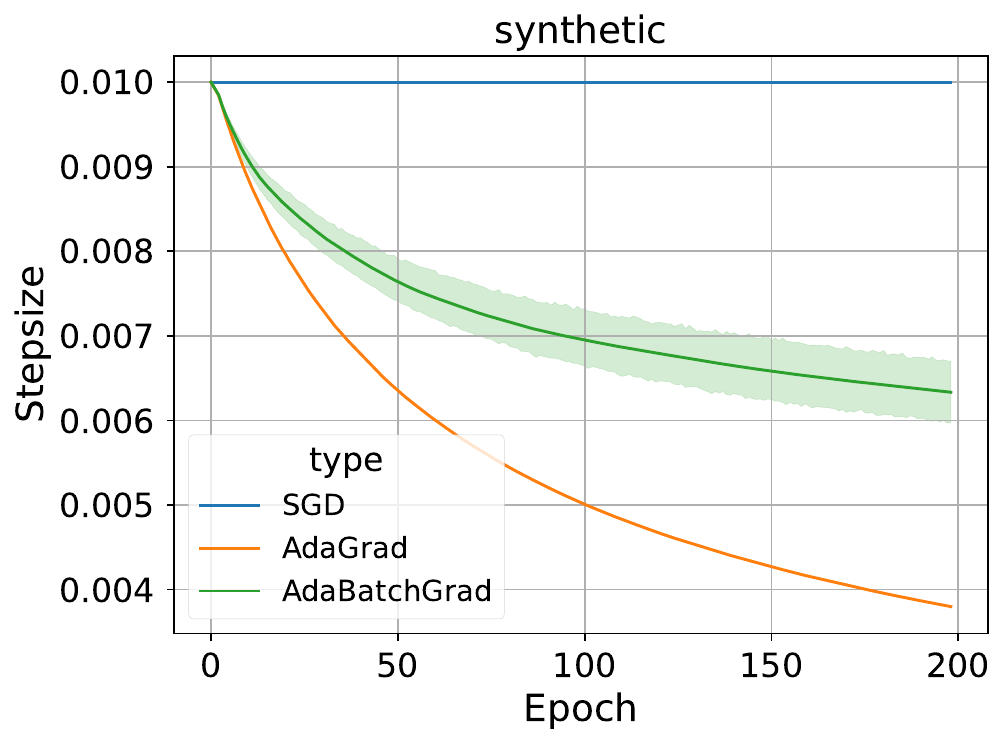}}{}  
            \subf{\includegraphics[width=0.24\textwidth]{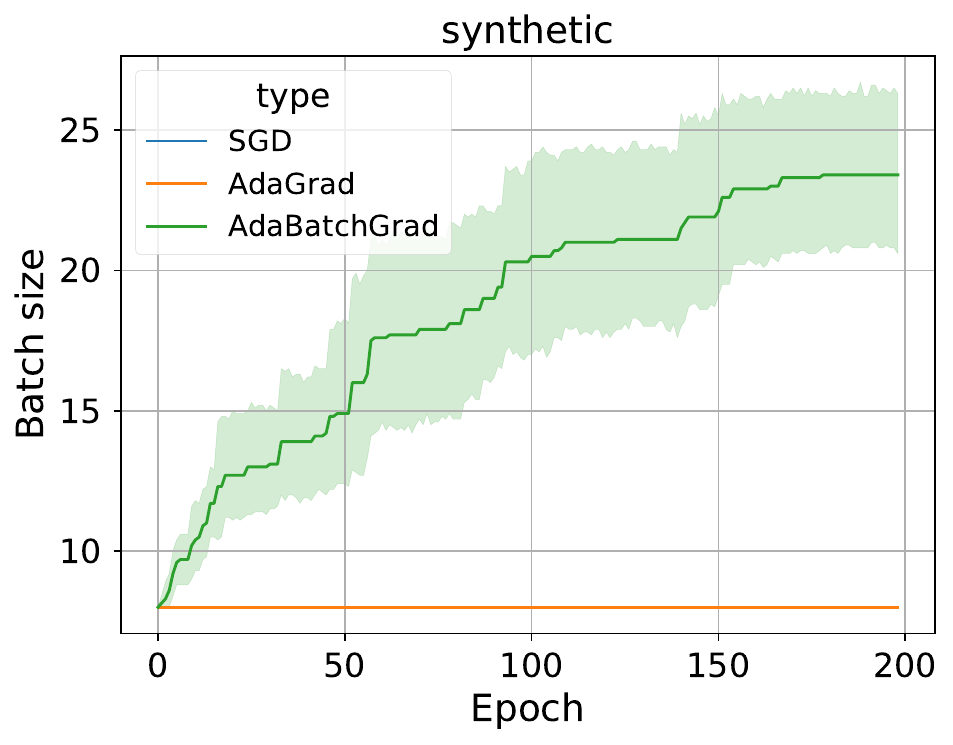}}{}
            \subf{\includegraphics[width=0.24\textwidth]{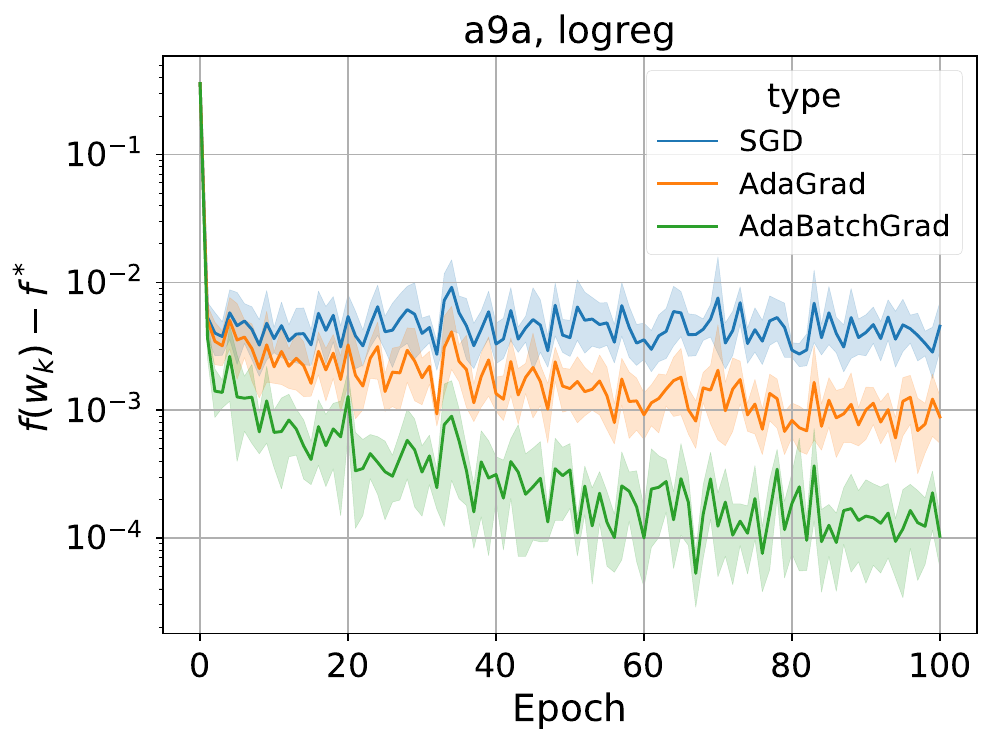}}{}  
            \subf{\includegraphics[width=0.24\textwidth]{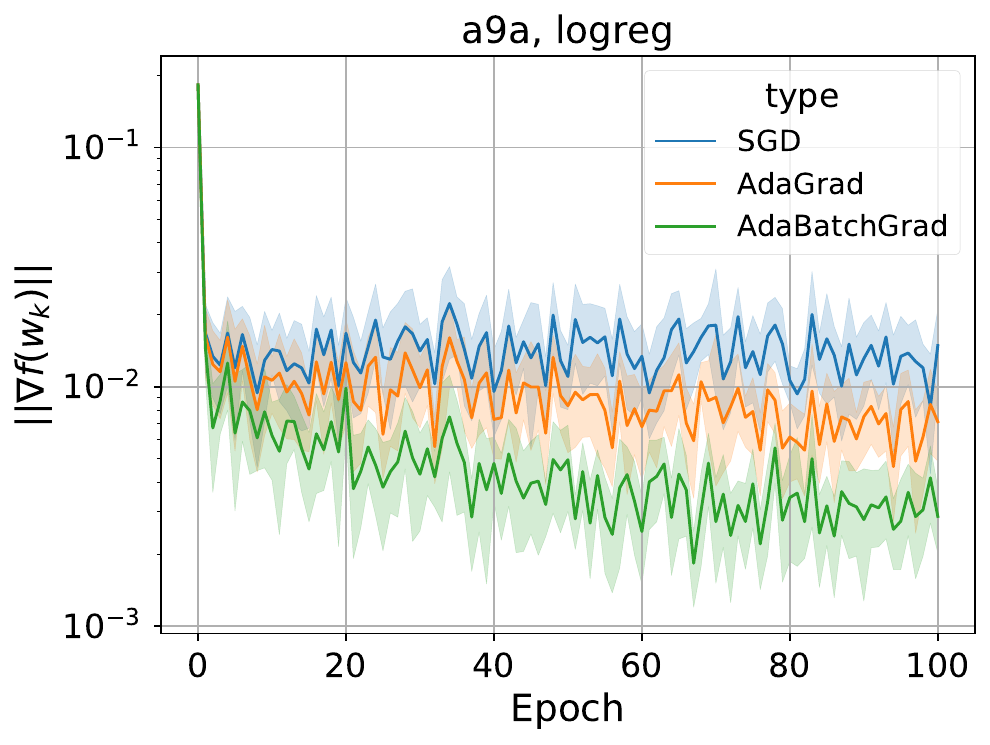}}{}  
            \subf{\includegraphics[width=0.24\textwidth]{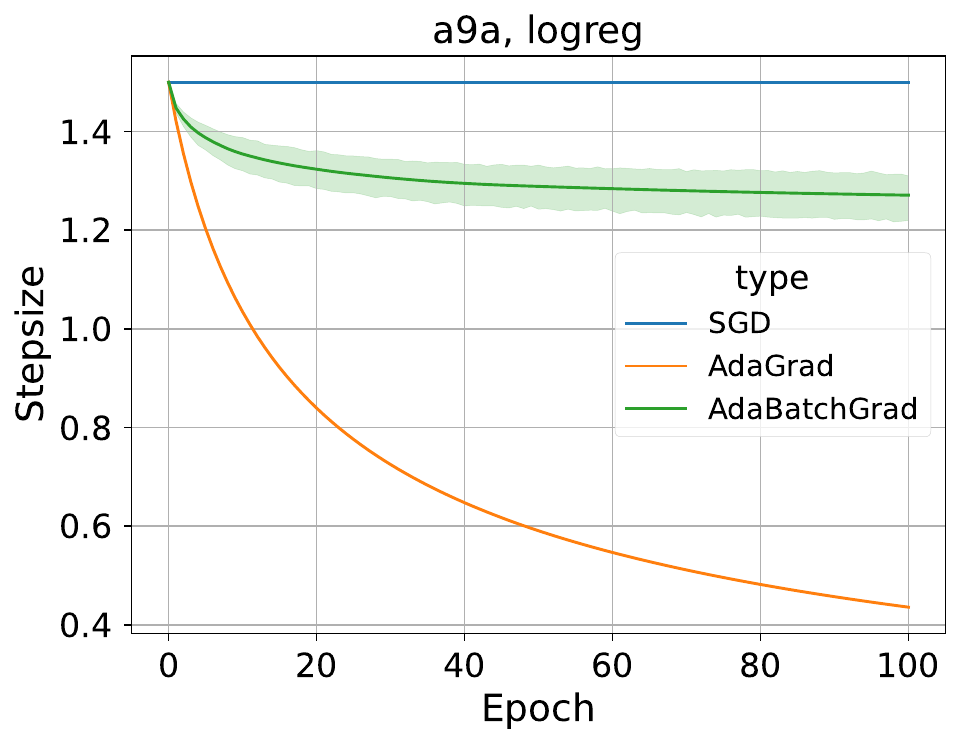}}{}  
            \subf{\includegraphics[width=0.24\textwidth]{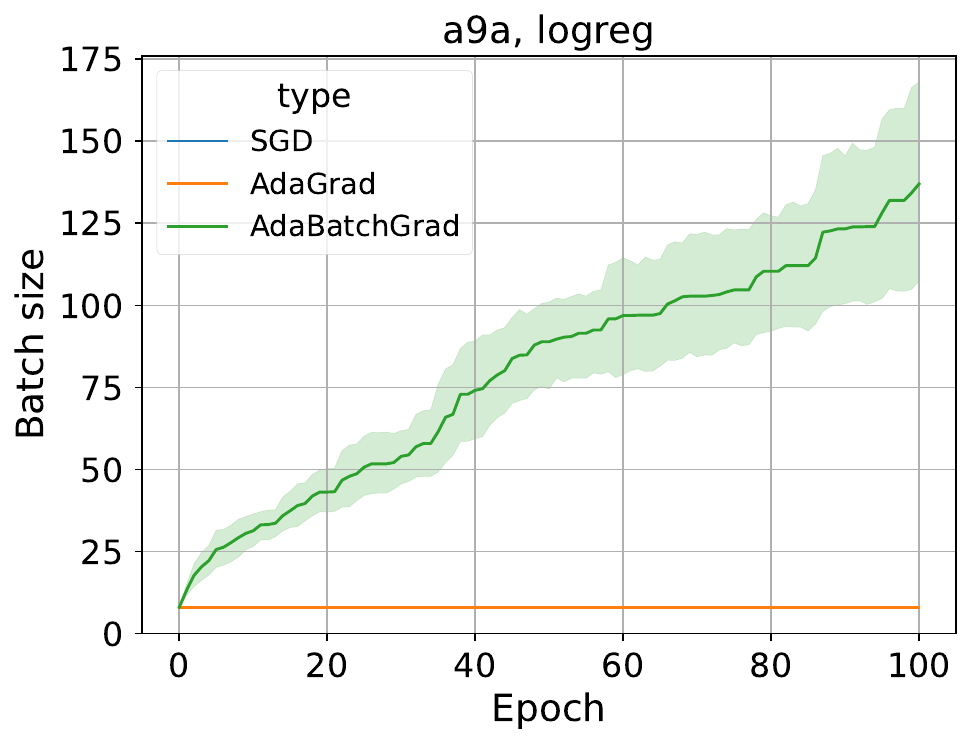}}{}  
            \subf{\includegraphics[width=0.24\textwidth]{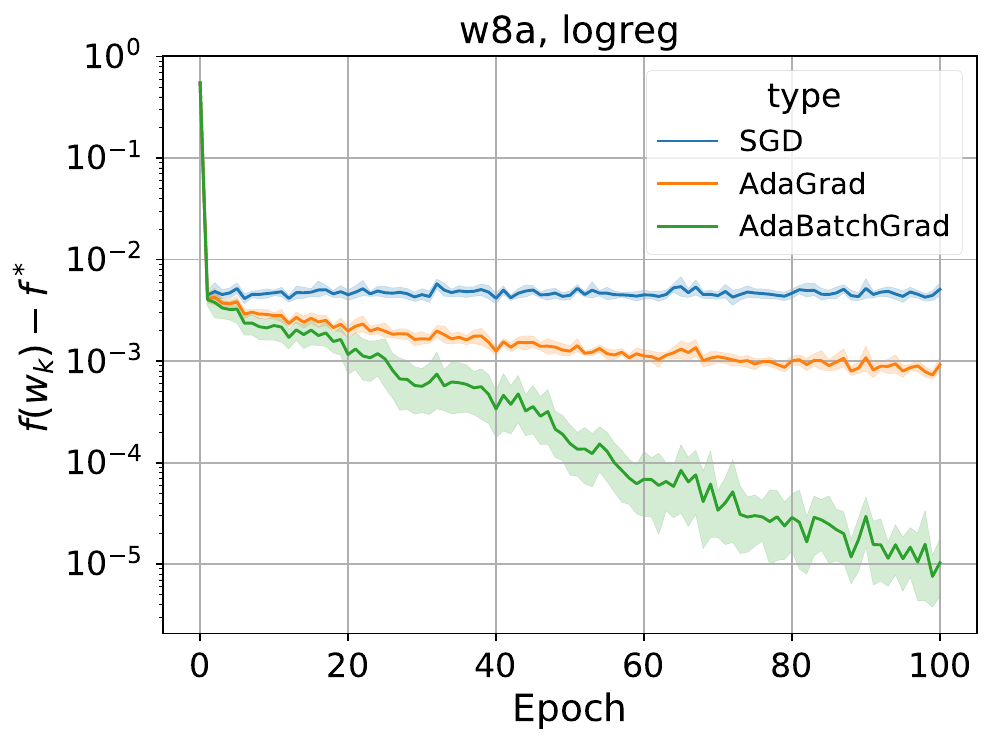}}{}  
            \subf{\includegraphics[width=0.24\textwidth]{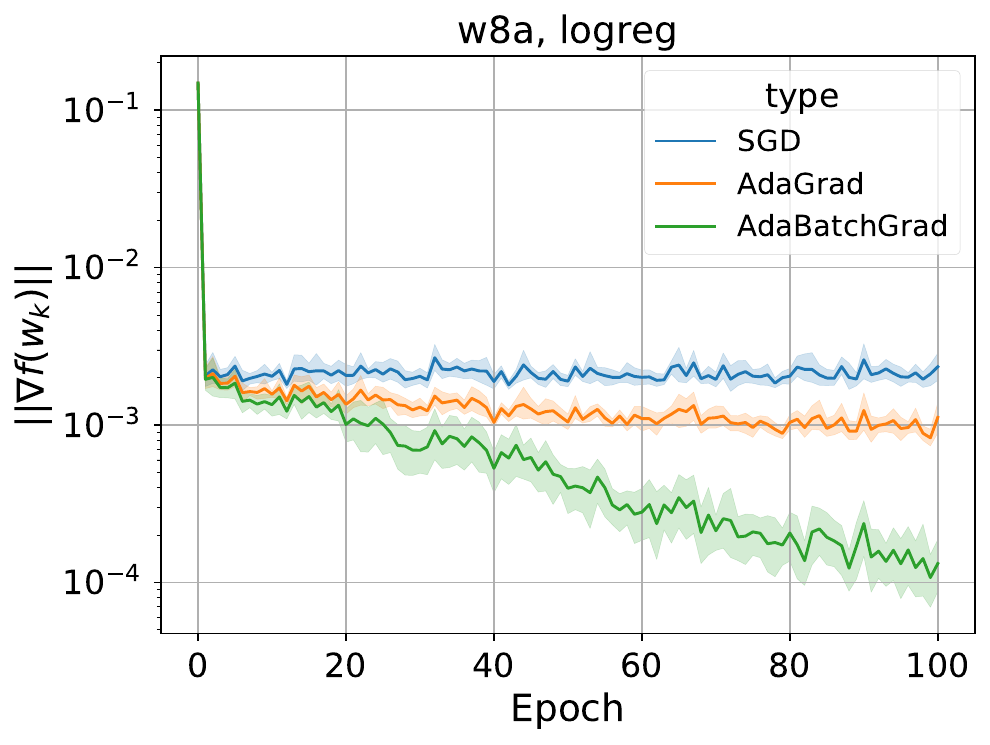}}{}  
            \subf{\includegraphics[width=0.24\textwidth]{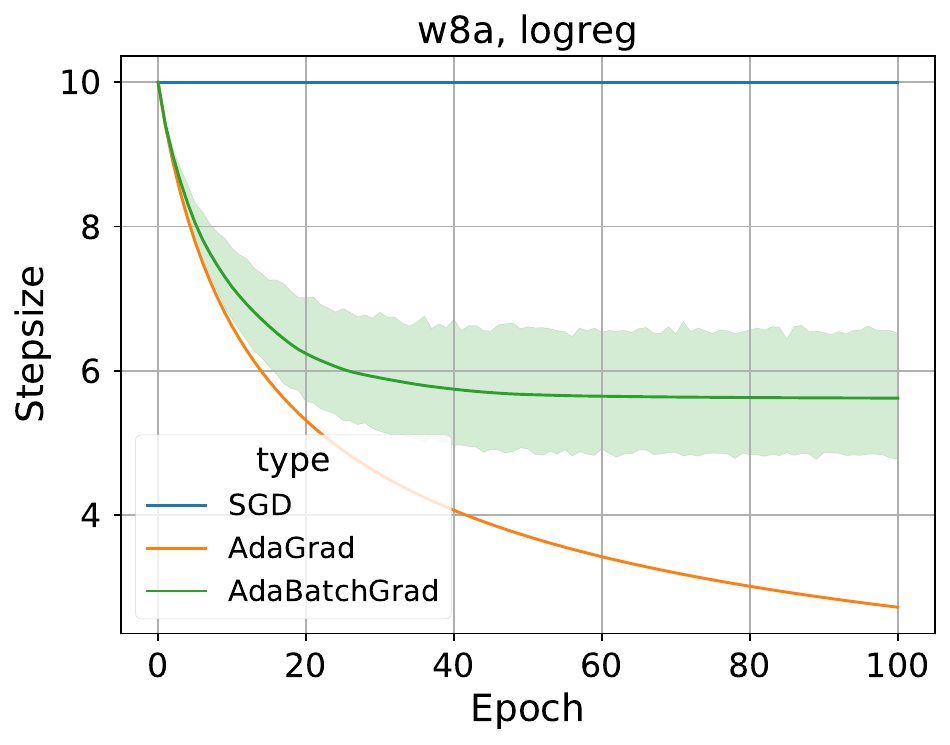}}{}
            \subf{\includegraphics[width=0.24\textwidth]{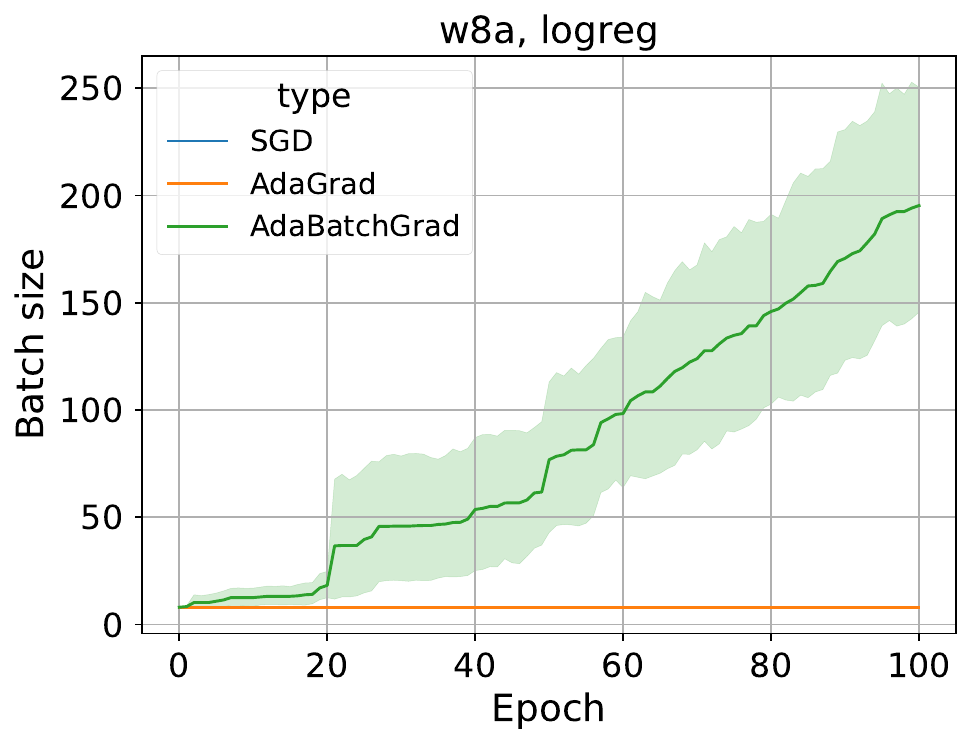}}{}
            \caption{\label{logreg:task6all-datasets}\AdaBatchGrad \  vs AdaGrad vs SGD for synthetic (upper plots), logistic regression on a9a (middle plots) and on w8a (lower plots). 
            We use both techniques and get an algorithm, that is adaptive both in step size and in batch size. This gives us more freedom in choosing the parameters of the algorithm and lets us find a tradeoff between step size and batch size.}
            \subf{\includegraphics[width=0.24\textwidth]{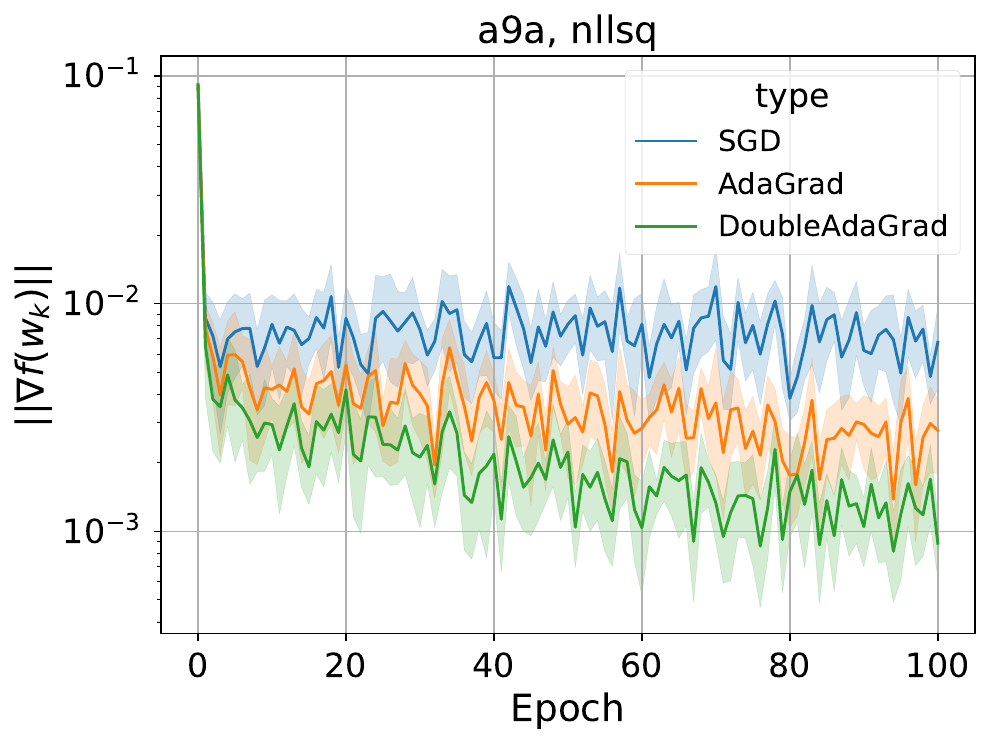}}{}  
            \subf{\includegraphics[width=0.24\textwidth]{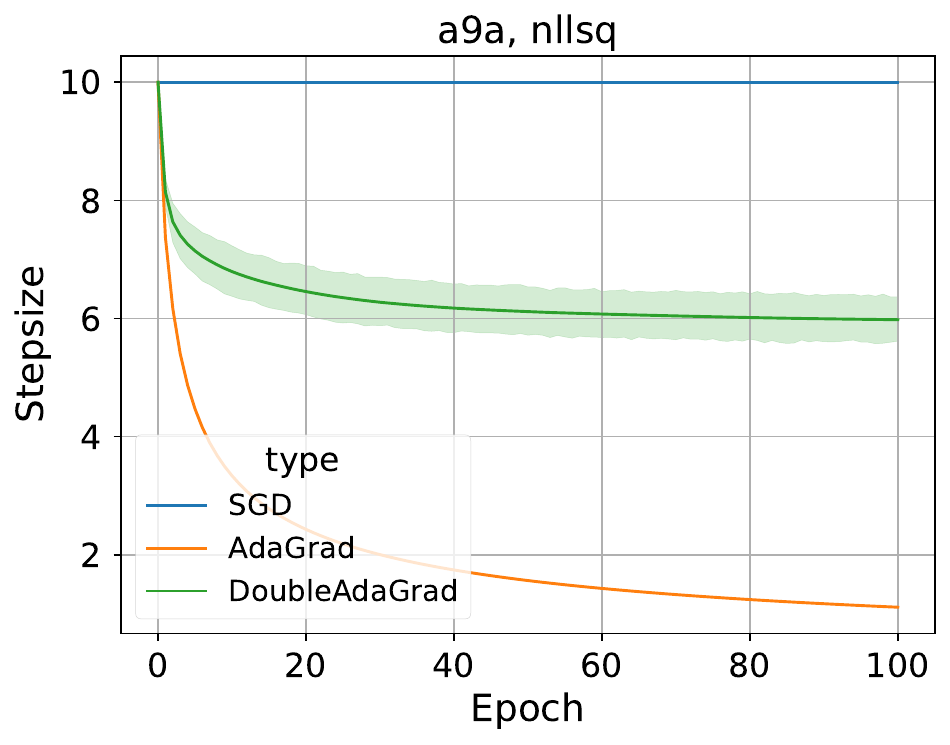}}{}   
            \subf{\includegraphics[width=0.24\textwidth]{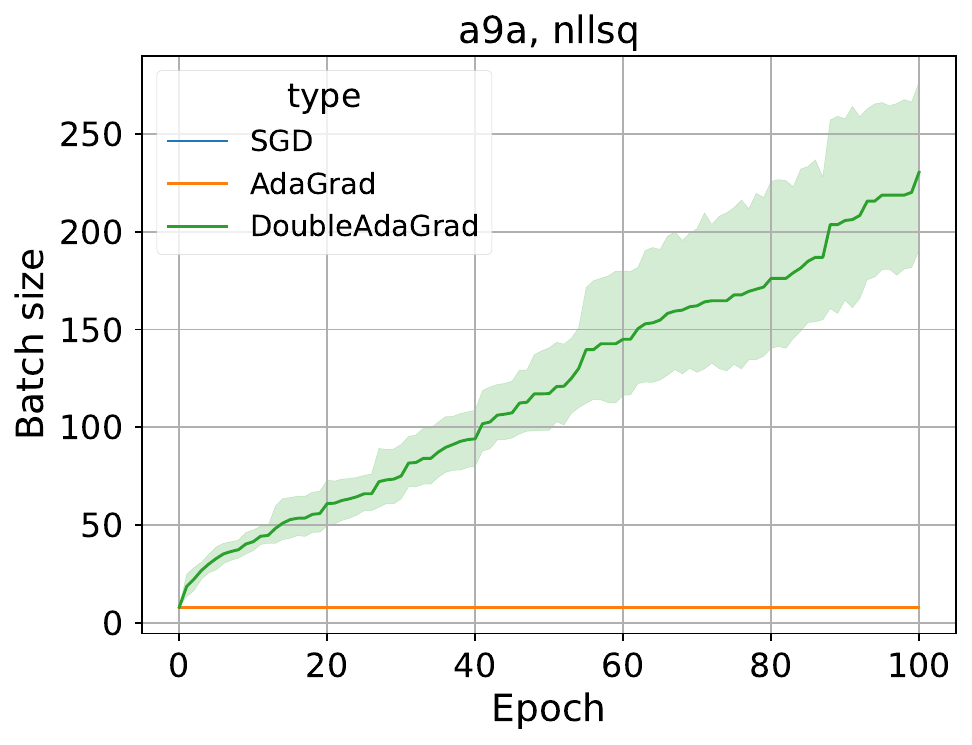}}{}   
            \\
            \subf{\includegraphics[width=0.24\textwidth]{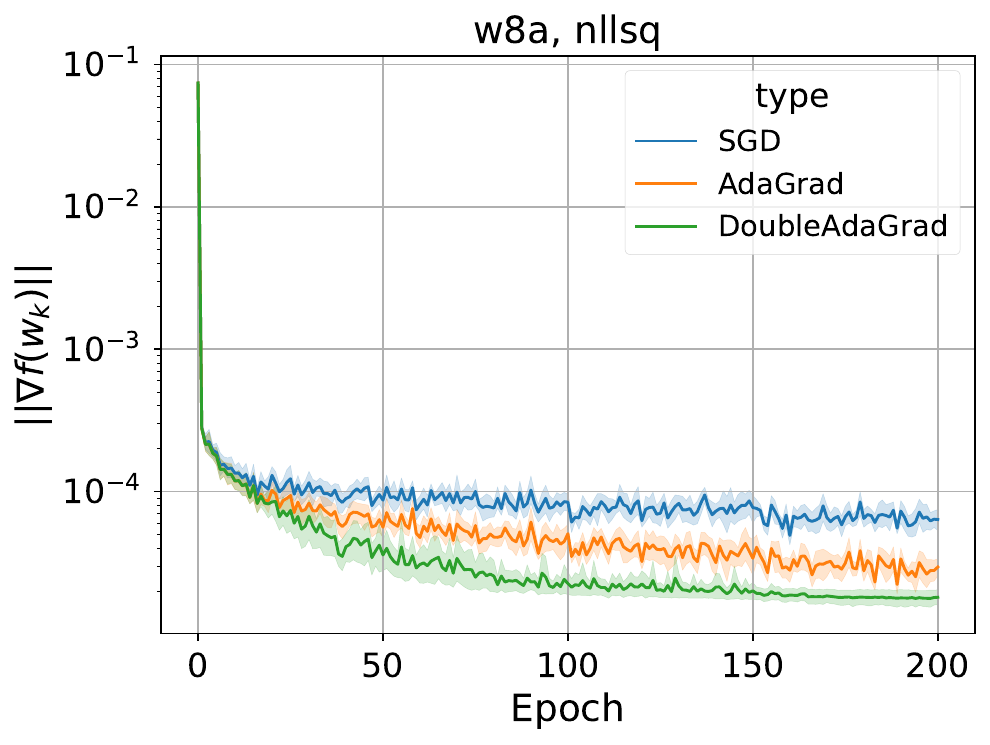}}{}  
            \subf{\includegraphics[width=0.24\textwidth]{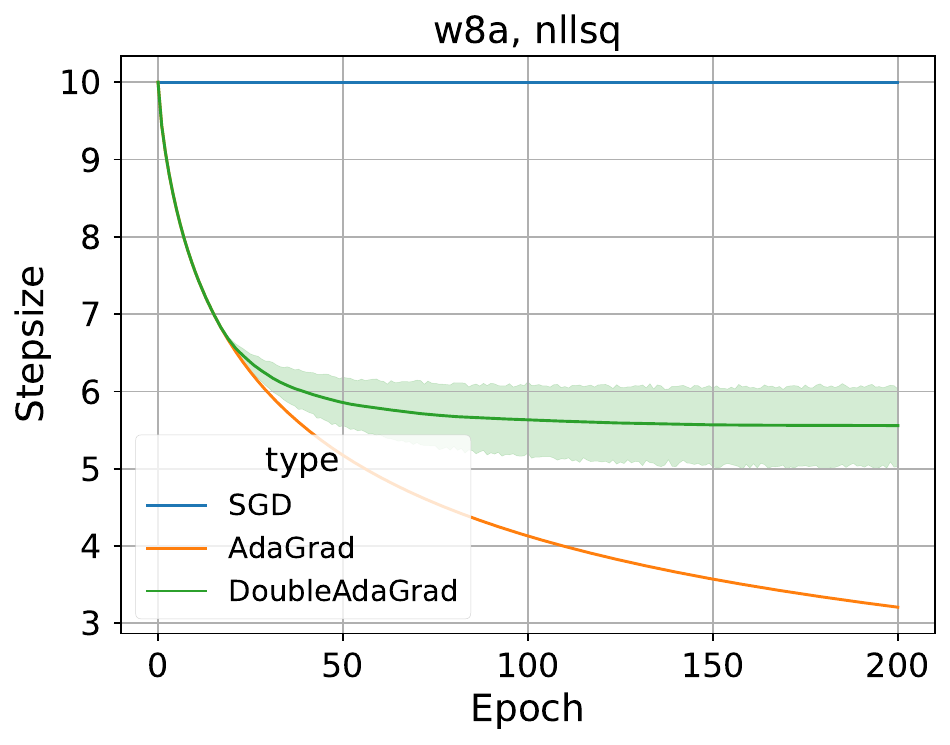}}{}               
            \subf{\includegraphics[width=0.24\textwidth]{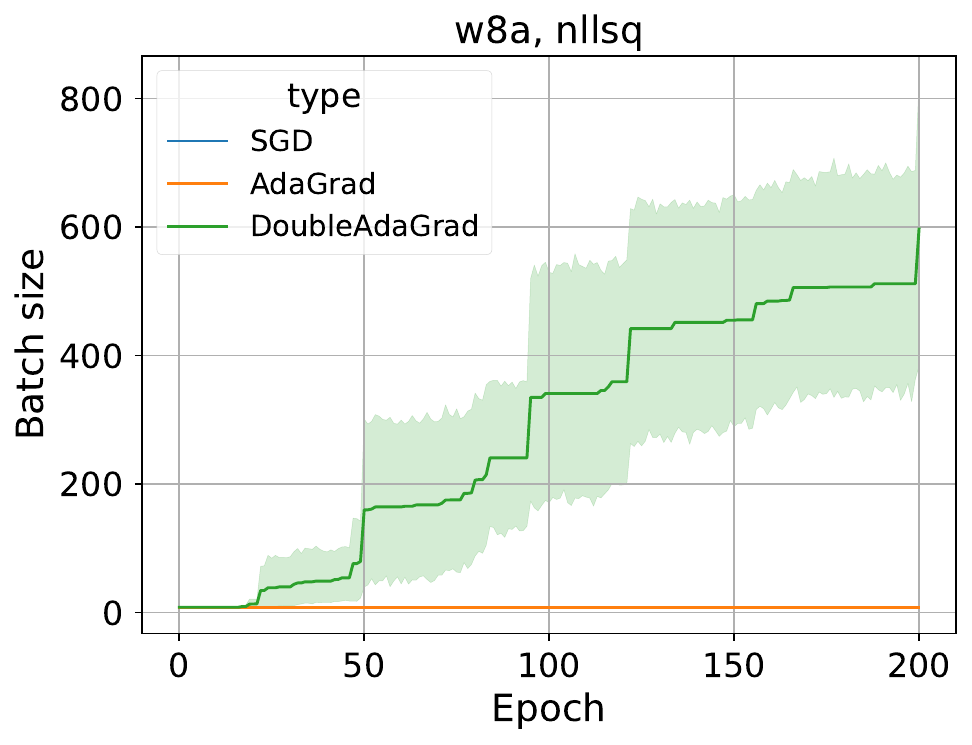}}{}                          
            \caption{\label{nllsq:task6all-datasets}\AdaBatchGrad \  vs AdaGrad vs SGD for NLLSQ on a9a (upper plots) and on w8a (lower plots). 
            We use both techniques and get an algorithm, that is adaptive both in step size and in batch size. This gives us more freedom in choosing the parameters of the algorithm and lets us find a tradeoff between step size and batch size.}
            \subf{\includegraphics[width=0.24\textwidth]{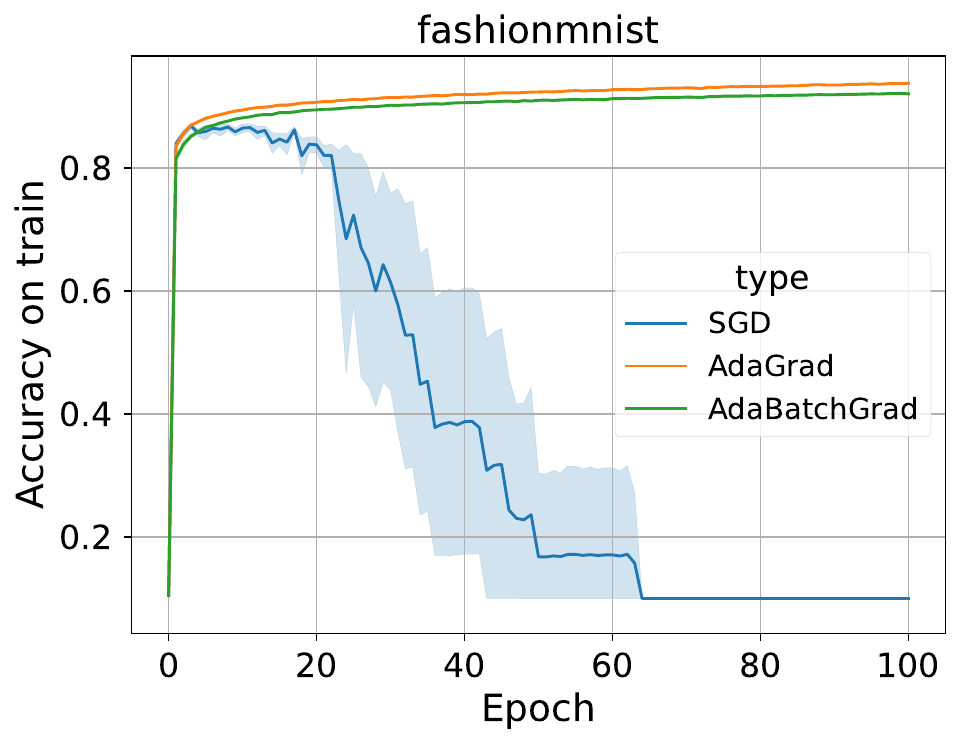}}{}
            \subf{\includegraphics[width=0.24\textwidth]{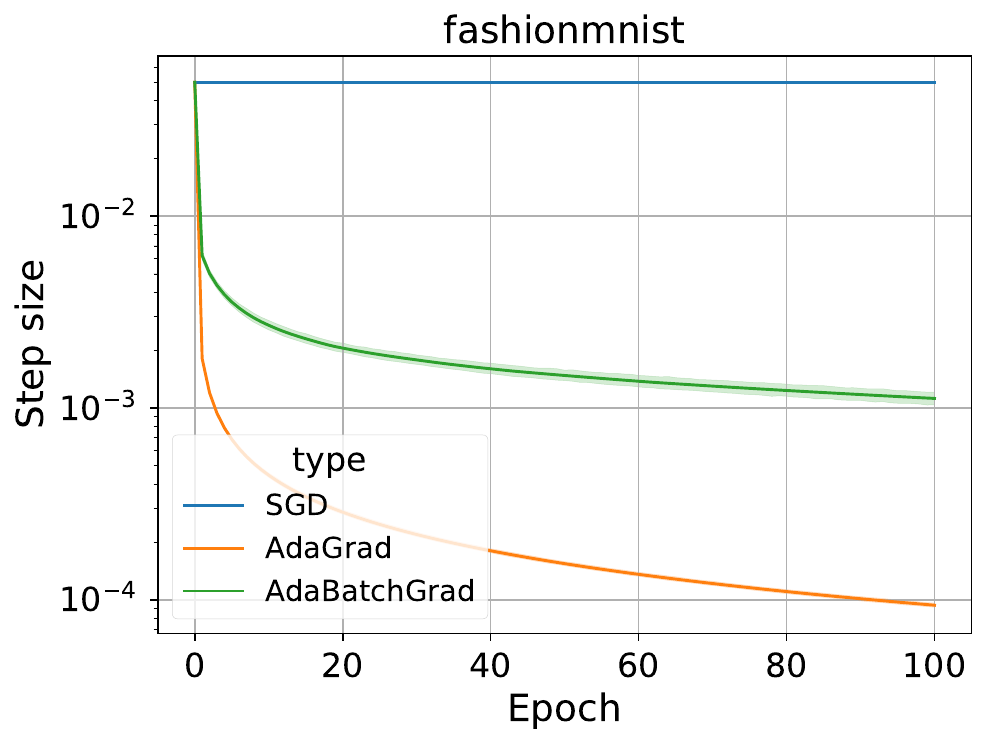}}{}
            \subf{\includegraphics[width=0.24\textwidth]{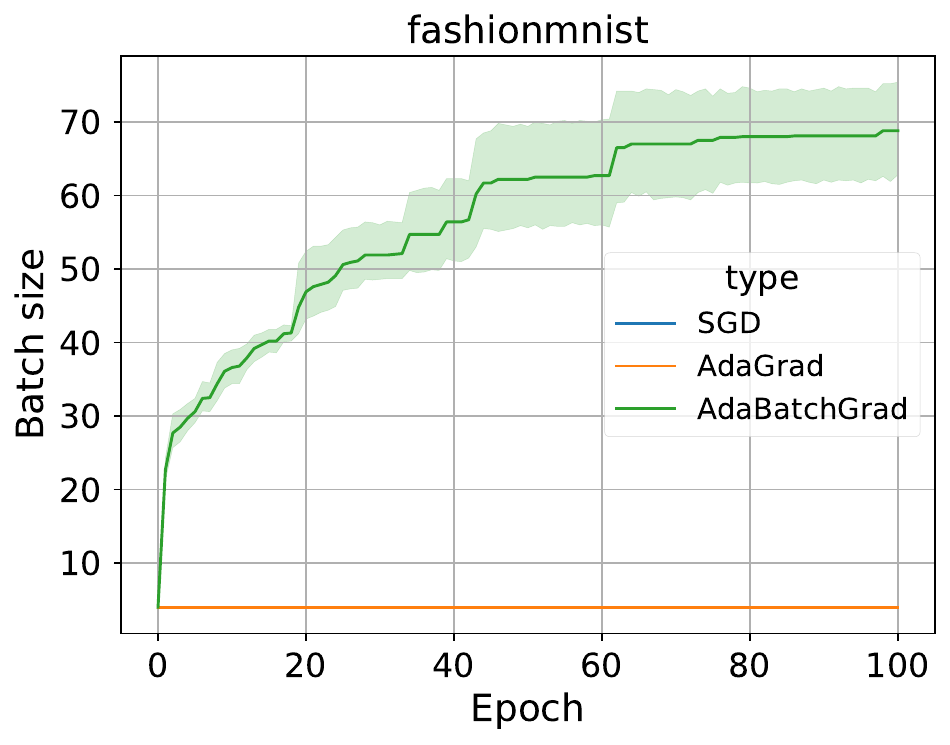}}{}
            \caption{\label{nn:task6}\AdaBatchGrad \  vs AdaGrad vs SGD for CNN on FashionMNIST. 
            We use both techniques and get an algorithm, that is adaptive both in step size and in batch size. This gives us more freedom in choosing the parameters of the algorithm and lets us find a tradeoff between step size and batch size.}
        \end{figure}

\section{Conclusion}
    In this paper, we present \AdaBatchGrad.
    It is an adaptive method, that uses AdaGrad step size adjustment strategy \eqref{eq:global step size} inside Adaptive Sampling Method (Algorithm \ref{alg:adaptive sampling}), that has adaptive batch size.
    The main issue of the Adaptive Sampling Method is that it uses approximated procedures to check, whether batch size needs to be increased.
    These tests can be inaccurate and have false results.
    Moreover, the Adaptive Sampling Method uses line search to choose step size.
    These two facts can cause divergence of the mentioned method.
    We show that the use of AdaGrad instead of line search prevents the method from diverging.
    Moreover, if on each iteration the chosen batch $S_k$ satisfies norm test \eqref{eq:norm test}, the method converges as $O(1/\e)$ regardless of noise intensity.

%\clearpage
\bibliographystyle{plain}
\bibliography{bibliography,kamzolov_full}

\clearpage
\appendix

\section{Auxiliary information}
    In this section, we provide some additional information and lemmas, which are crucial for our theoretical results.
    
    First of all, we need to mention Jensen's inequality
    \begin{equation}\label{eq:Jensen}
        \varphi(\E x) \le \E \varphi(x),\ \text{if $\varphi$ is convex},
    \end{equation}
    and H\"older inequality
    \begin{equation}\label{eq:Holder}
        \E[|AB|] \le \left(\E [|A|^p]\right)^\frac{1}{p}\left(\E[|B|^q]\right)^\frac{1}{q},\ \frac{1}{p} + \frac{1}{q} = 1.
    \end{equation}

    Then, there are some technical lemmas from \cite{li2019convergence} and its modifications.
    \begin{lemma}[Lemma 3 in \cite{li2019convergence}]\label{lem:lemma 3}
        Let $f: \R^d \to \R$ be smooth \eqref{eq:smoothness} and its gradient approximations be unbiased \eqref{eq:unbiasedness}.
        Then, the iterates of SGD with step sizes $\eta_t$ satisfy the following inequality
        \begin{equation}\label{eq:lemma 3 result}
            \E \left[ \sum_{t= 1}^T \la \nabla f(\x_t), \mathbf{\eta}_t \nabla f(\x_t) \ra \right] \le f(\x_1) - f^{\ast} + \frac{L}{2} \E \left[ \sum_{t=1}^T \|\mathbf{\eta}_t  \nabla f_{S_t}(\x_t) \|^2 \right].
        \end{equation}
    \end{lemma}

    \begin{lemma}[Lemma 4 in \cite{li2019convergence}]\label{lem:upper bound on gradient norm squared}
        Let $f: \R^d \to \R$ be smooth \eqref{eq:smoothness},
        then
        \begin{equation}\label{eq:lemma 4 result}
            \|\nabla f(\w)\|^2 \le 2L (f(\w) - \min_y f(y)), \forall \w \in \R^d.
        \end{equation}
    \end{lemma}

    \begin{lemma}[Modified Lemma 5 in \cite{li2019convergence}]\label{lem:technical lemma 5}
        If $x \ge 0$ and $x \le C \left(A + Bx^{1/2 + \tau}\right)$, then
        \begin{equation*}
            x \le \max \left\{ 2 AC; \left(2 C B\right)^\frac{1}{1/2 - \tau} \right\}
        \end{equation*}
    \end{lemma}
        \begin{proof}
            If $A \ge B x^{1/2 + \tau}$, then we have 
            \begin{equation*}
                x \le C \left(A + Bx^{1/2 + \tau}\right) \le 2AC.
            \end{equation*}
            If $A < B x^{1/2 + \tau}$, then 
            \begin{equation*}
                x \le C \left(A + Bx^{1/2 + \tau}\right) \le 2CBx^{1/2 + \tau} \Leftrightarrow x \le \left( 2CB \right)^\frac{1}{1/2 - \tau}.
            \end{equation*}
            Thus, we get the desired result.
        \end{proof}

    \begin{lemma}[Lemma 7 in \cite{li2019convergence}]\label{lem:(x+y)^p le x^p y^p}
        If $x, y \ge 0$ and $0 \le p \le 1$, then 
        \begin{equation}\label{eq:(x+y)^p <= x^p + y^p}
           (x + y)^p \le x^p + y^p.
        \end{equation}
    \end{lemma}

    \begin{lemma}[Modified Lemma 8 from \cite{li2019convergence}] \label{lem:modification of lemma 8}
        Let $f: \R^d \to \R$ be smooth \eqref{eq:smoothness}, $\eta_t$ satisfy \eqref{eq:global step size}, where $\alpha, \beta, \tau > 0$. 
        Then
        \begin{equation}\label{eq:lemma 8 result}
            \E \left[ \sum_{t = 1}^{T} \eta_t^2 \|\nabla f_{S_t}(\x_t)\|^2 \right] \le 2(1 + \omega^2) \E \left[ \sum_{t = 1}^T \eta_t^2 \| \nabla f(\x_t) \|^2 \right]
        \end{equation}
    \end{lemma}
    \begin{proof}
        Denote $\E_t[\cdot] \equiv \E[\cdot|S_1, ..., S_{t-1}]$.
        \begin{align*}
            \E \left[ \sum_{t = 1}^T \eta_t^2 \|\nabla f_{S_t}(\x_t)\|^2 \right] 
            &\le \E \left[ \sum_{t = 1}^T \eta_t^2 \left( \|\nabla f_{S_t}(\x_t) - \nabla f(\x_t) \| + \|\nabla f(\x_t)\|  \right)^2 \right] \\
            &\le 2 \sum_{t=1}^T \E \left[ \eta_t^2 ( \|\nabla f_{S_t}(\x_t) - \nabla f(\x_t)  \|^2 + \| \nabla f(\x_t) \|^2) \right] \\
            &= 2 \sum_{t=1}^T \E \left[ \eta_t^2 \left( \E_t \left[\|\nabla f_{S_t}(\x_t) - \nabla f(\x_t)  \|^2 \right] + \E_t \left[\| \nabla f(\x_t) \|^2 \right]\right) \right].
        \end{align*}
        Then, from \eqref{eq:norm test} and the fact that $\E_t\left[ \|\nabla f(\x_t) \|^2 \right] = \|\nabla f(\x_t) \|^2$ we get
        \begin{align*}
            \E \left[ \sum_{t = 1}^T \eta_t^2 \|\nabla f_{S_t}(\x_t)\|^2 \right]
            &\overset{\eqref{eq:norm test}}{\le} 2(1 + \omega^2) \E \left[ \sum_{i=1}^T \eta_t^2 \|\nabla f(\w_t) \|^2 \right].
        \end{align*}
    \end{proof}

\section{Additional plots for FashionMNIST}\label{sec:additional plots for fashionmnist}
    In this section, we provide additional plots for FashionMNIST.
    We present accuracy, function loss, and gradient norm.
    
        \begin{figure}[H]
            \centering
            \subf{\includegraphics[width=0.49\textwidth]{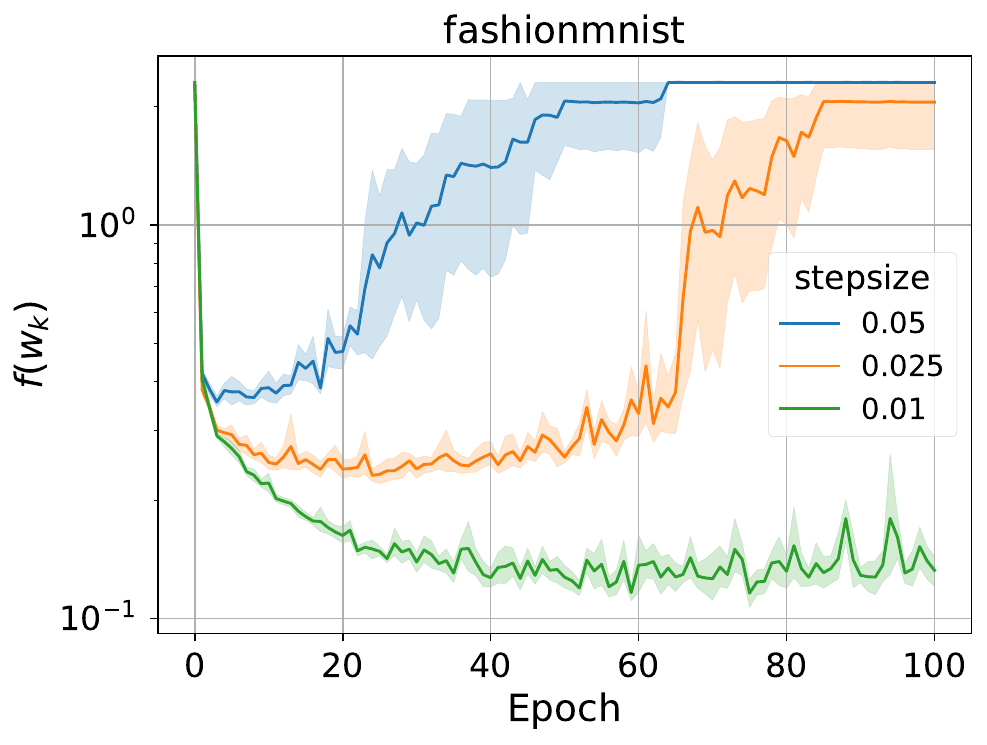}}{}
            \subf{\includegraphics[width=0.49\textwidth]{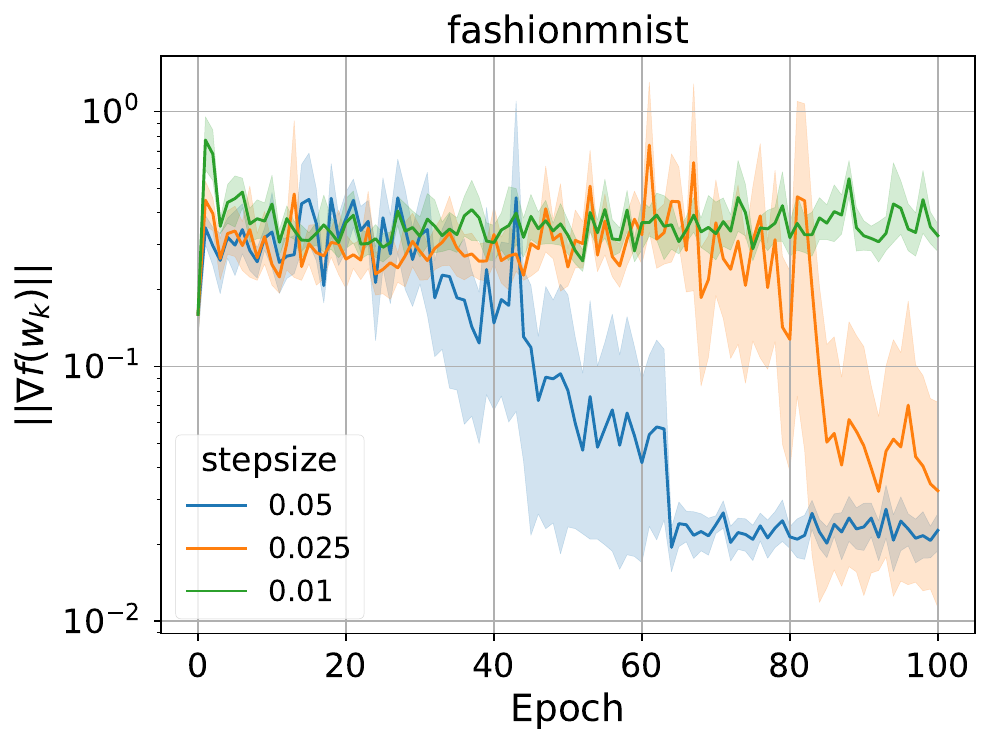}}{}
            % \\
            \subf{\includegraphics[width=0.49\textwidth]{pics/fashionmnist/task1_train_accuracy.pdf}}{}
            \caption{SGD with different step sizes for CNN on FashionMNIST. Smaller step size allows algorithm converge to better solution.}
        \end{figure}

        \begin{figure}[H]
            \centering
            \subf{\includegraphics[width=0.5\textwidth]{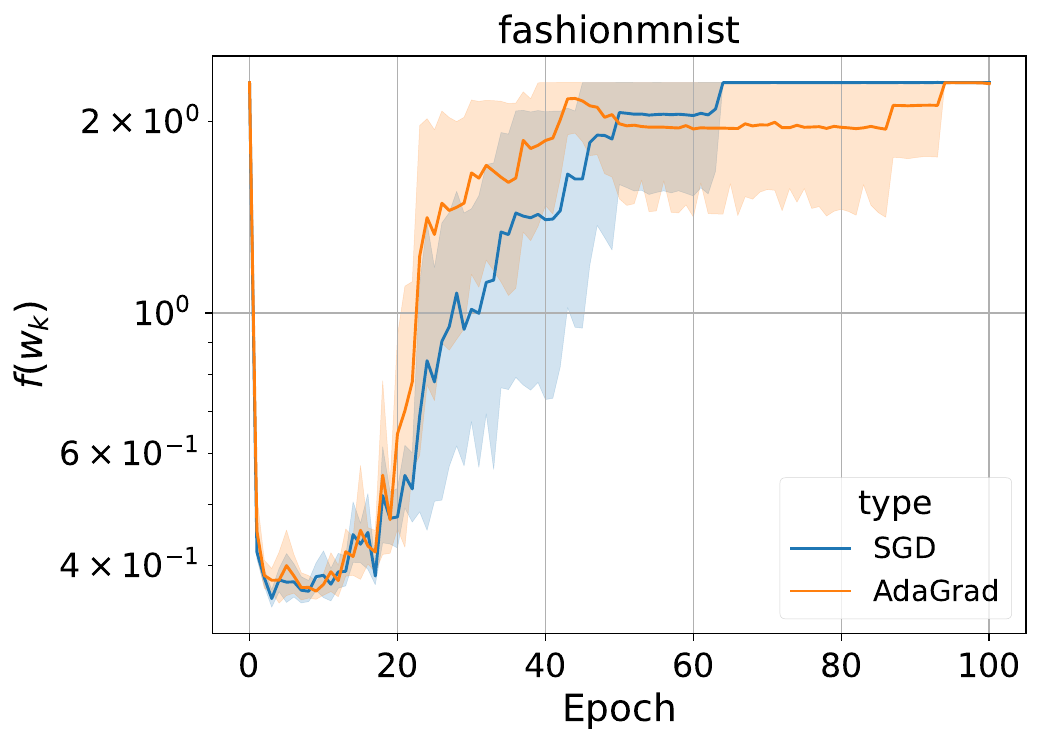}}{}
            \subf{\includegraphics[width=0.47\textwidth]{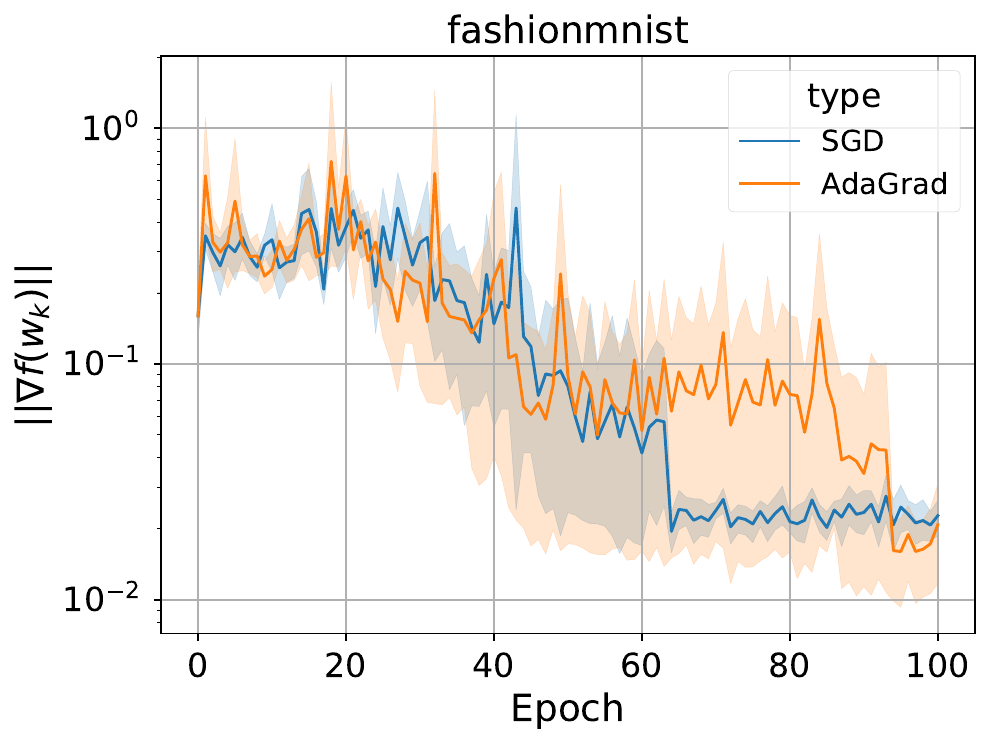}}{}
            % \\
            \subf{\includegraphics[width=0.47\textwidth]{pics/fashionmnist/task2_train_accuracy.pdf}}{}
            \subf{\includegraphics[width=0.5\textwidth]{pics/fashionmnist/task2_step_size.pdf}}{}
            \caption{AdaGrad vs SGD with large $\beta$ for CNN on FashionMNIST. Here we choose $\beta$ so big that it majorizes the sum of gradient norms in the denominator of \eqref{eq:global step size}, and we have the same behavior as regular SGD. The step size is almost constant.}
        \end{figure}
        
        \begin{figure}[H]
            \centering
            \subf{\includegraphics[width=0.49\textwidth]{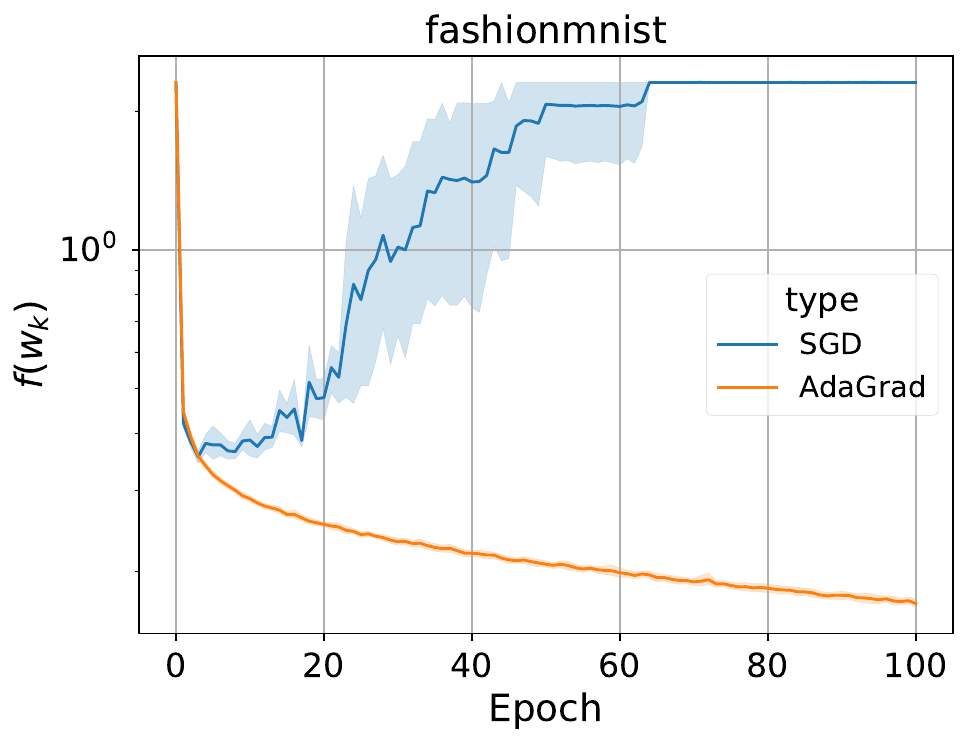}}{}
            \subf{\includegraphics[width=0.49\textwidth]{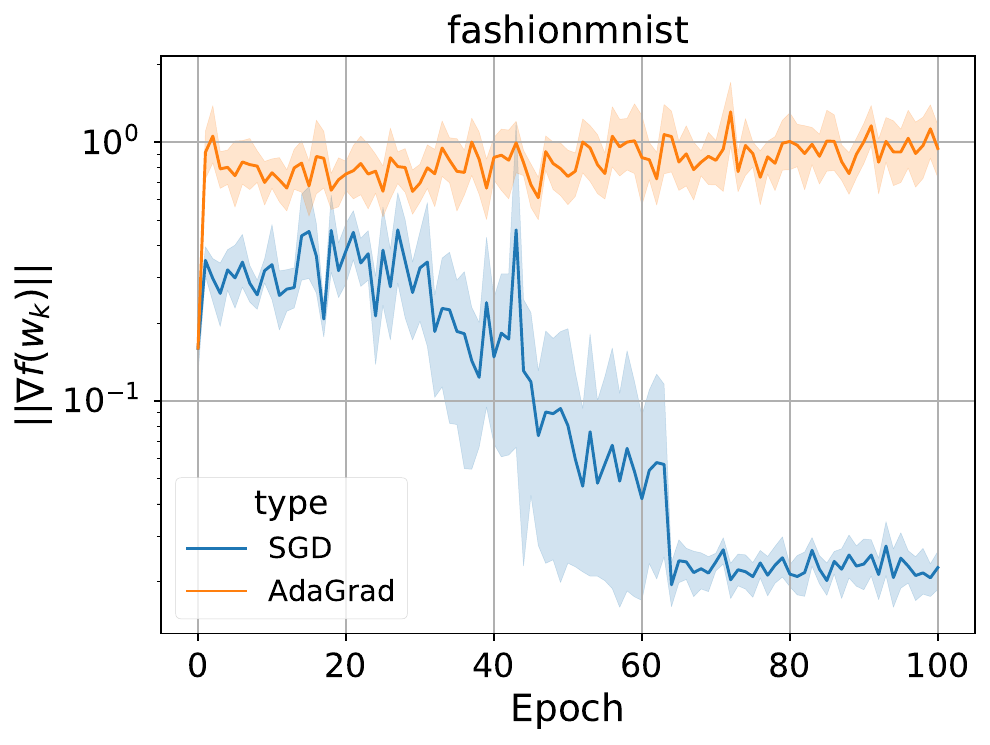}}{}
            % \\
            \subf{\includegraphics[width=0.49\textwidth]{pics/fashionmnist/task3_train_accuracy.pdf}}{}
            \subf{\includegraphics[width=0.49\textwidth]{pics/fashionmnist/task3_step_size.pdf}}{}
            \caption{AdaGrad vs SGD with small $\beta$ for CNN on FashionMNIST. Here we choose little $\beta$ so that it allows the sum of gradient norms in the denominator of \eqref{eq:global step size} to change step size with time. As you can see, now step size gradually decreases, while approaching better accuracy.}
        \end{figure}

        \begin{figure}[H]
            \centering
            \subf{\includegraphics[width=0.49\textwidth]{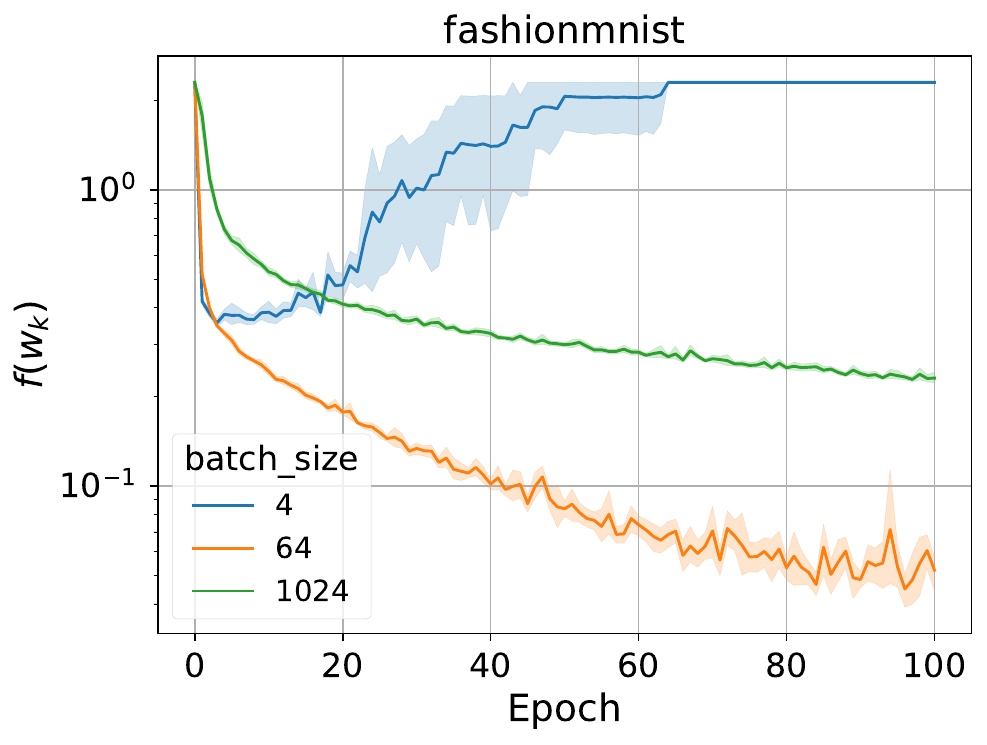}}{}
            \subf{\includegraphics[width=0.49\textwidth]{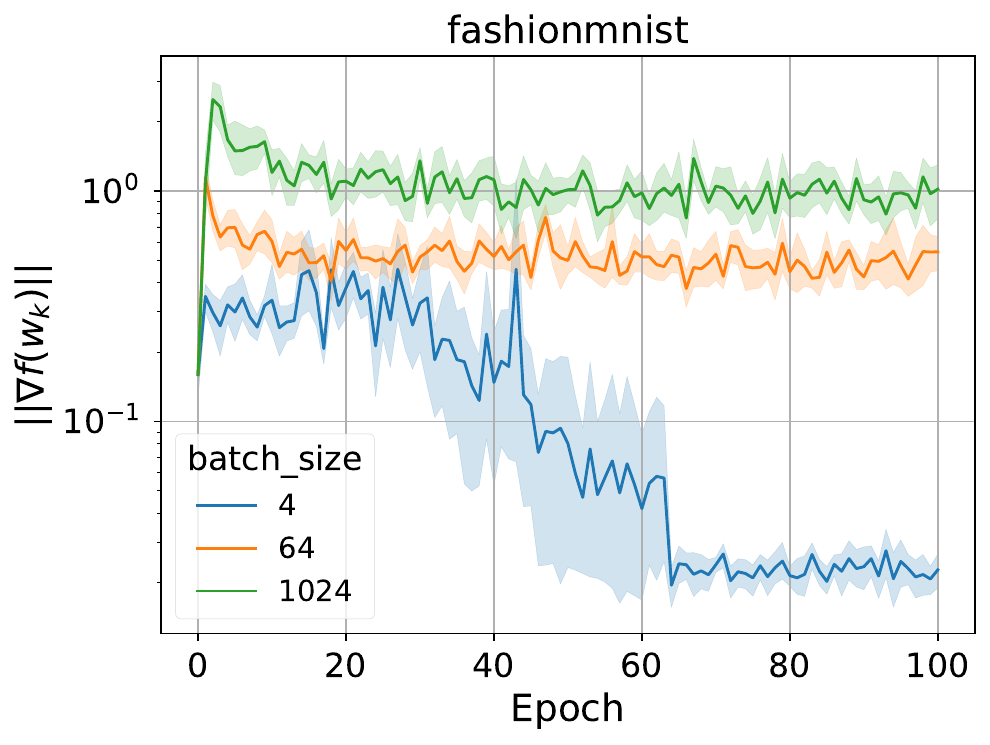}}{}
            % \\
            \subf{\includegraphics[width=0.49\textwidth]{pics/fashionmnist/task4_train_accuracy.pdf}}{}
            \caption{SGD with different batch sizes for CNN on FashionMNIST.
            A larger batch size allows the algorithm to converge to a better confusion region.}
        \end{figure}

        \begin{figure}[H]
            \centering
            \subf{\includegraphics[width=0.49\textwidth]{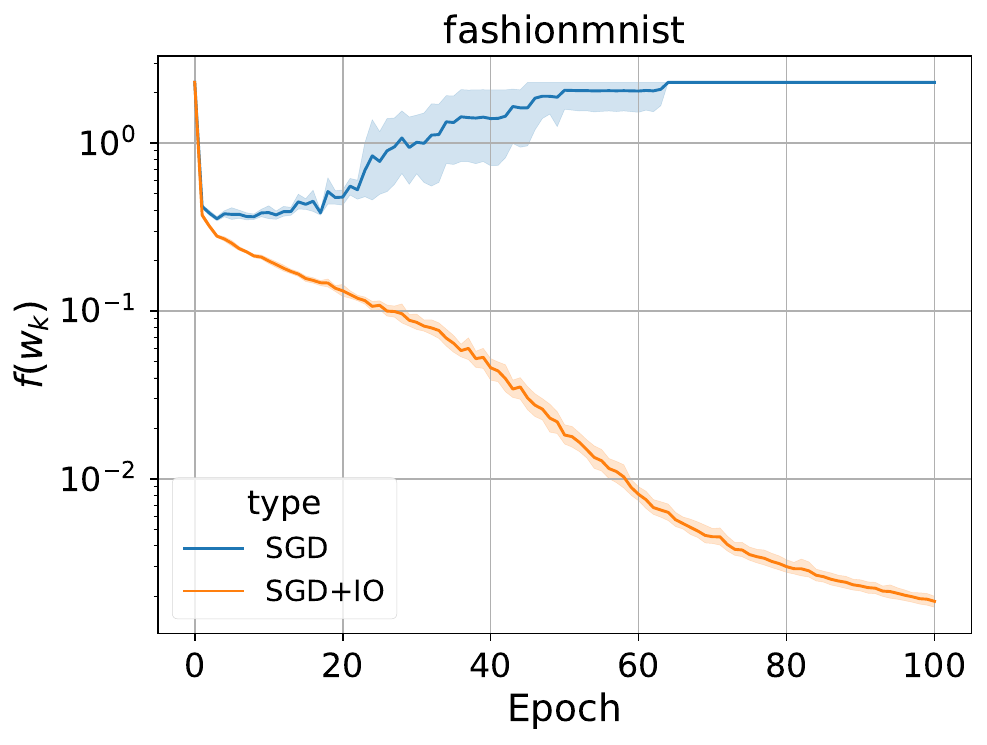}}{}
            \subf{\includegraphics[width=0.49\textwidth]{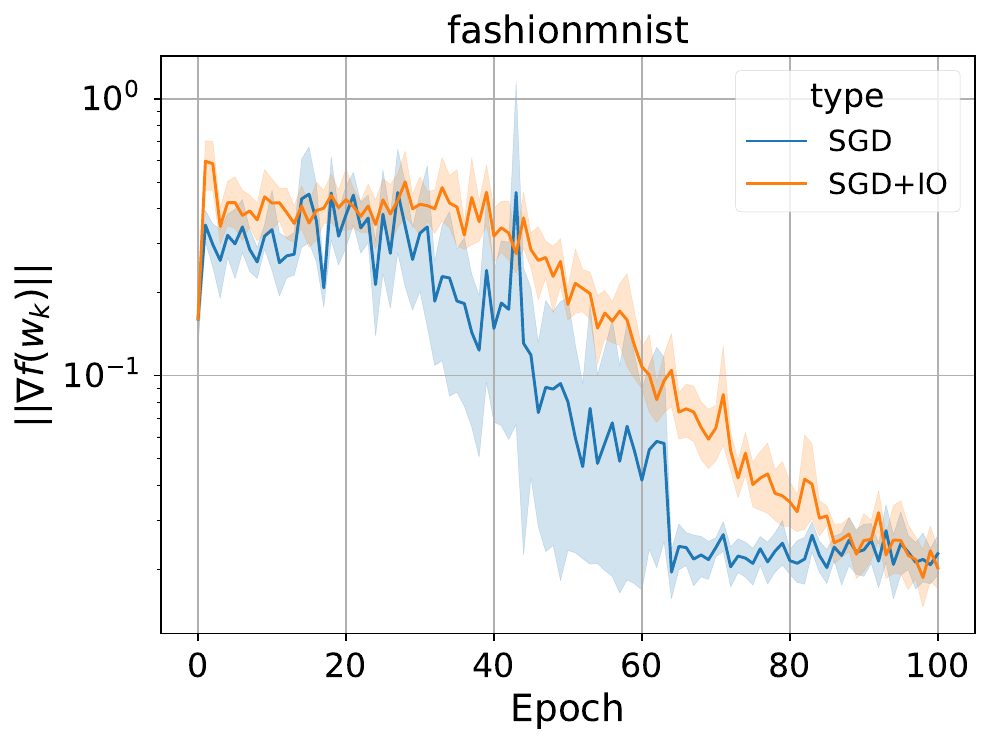}}{}
            % \\
            \subf{\includegraphics[width=0.49\textwidth]{pics/fashionmnist/task5_train_accuracy.pdf}}{}
            \subf{\includegraphics[width=0.49\textwidth]{pics/fashionmnist/task5_batch_size.pdf}}{}
            \caption{SGD vs SGD + Inner test + Ortho test (IO) for CNN on FashionMNIST. Adaptivity in batch size makes the algorithm adjust batch size automatically as long as it converges to a minimum. Despite the fact that batch size grows up to a very small part of the whole dataset, it allows us to significantly improve the results.}
        \end{figure}

        \begin{figure}[H]
            \centering
            \subf{\includegraphics[width=0.49\textwidth]{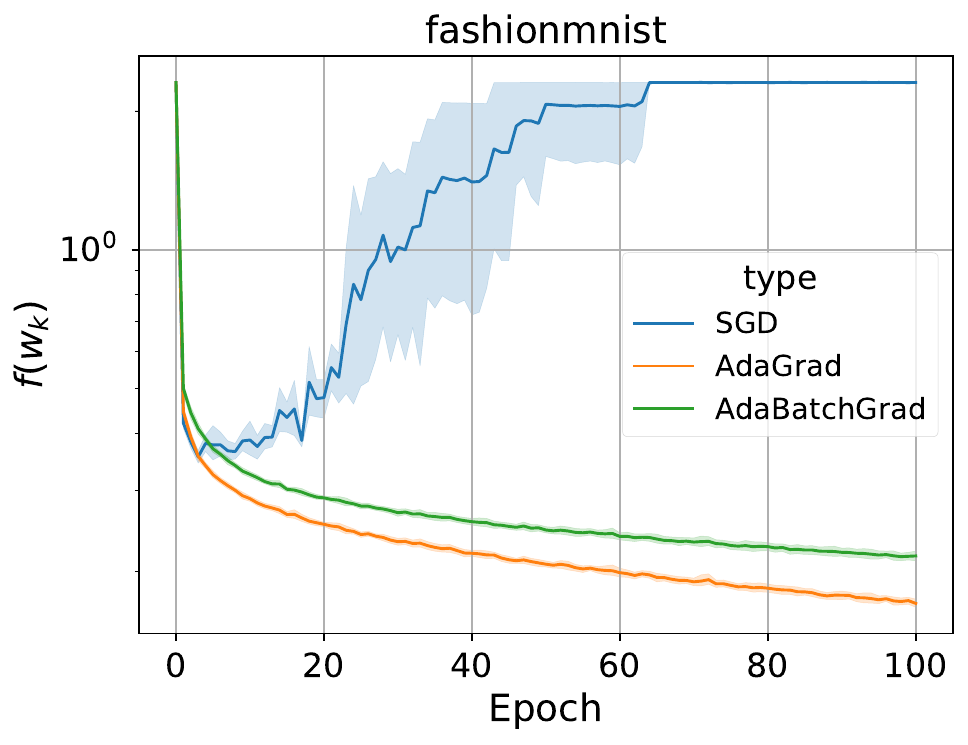}}{}
            \subf{\includegraphics[width=0.49\textwidth]{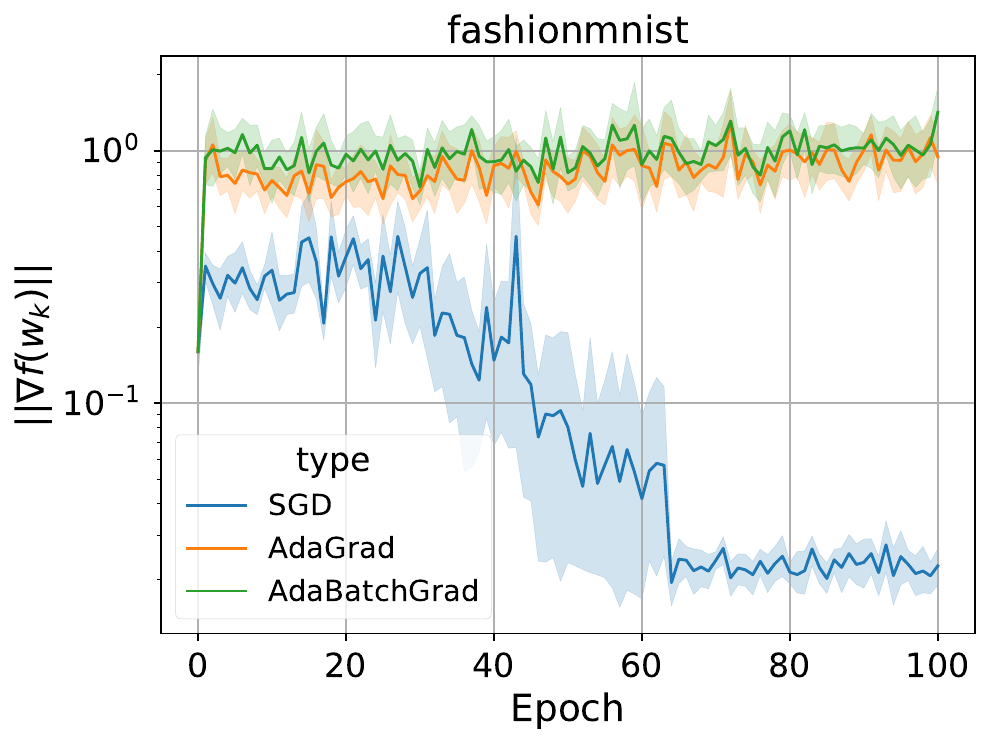}}{}
            % \\
            \subf{\includegraphics[width=0.49\textwidth]{pics/fashionmnist/task6_train_accuracy.pdf}}{}
            \subf{\includegraphics[width=0.49\textwidth]{pics/fashionmnist/task6_batch_size.pdf}}{}
            \subf{\includegraphics[width=0.49\textwidth]{pics/fashionmnist/task6_step_size.pdf}}{}
            \caption{\AdaBatchGrad \  vs AdaGrad vs SGD for CNN on FashionMNIST. 
            We use both techniques and get an algorithm, that is adaptive both in step size and in batch size. This gives us more freedom in choosing the parameters of the algorithm and lets us find a tradeoff between step size and batch size.}
        \end{figure}

\end{document}